\documentclass{article}
\usepackage[utf8]{inputenc}
\usepackage{amsfonts}
\usepackage{amsmath}
\usepackage{amsthm}
\usepackage{algorithm}
\usepackage{algpseudocode}
\usepackage{mathtools}
\usepackage{color}
\usepackage{tabularx}

\newtheorem{thm}{Theorem}
\newtheorem{lem}[thm]{Lemma}

\usepackage{subcaption}

\def \Rdz {\mathbb{R}^{{d}_{\mathbf{y}}}}

\def \dz {{d}_{\mathbf{y}}}
\def \bz {\mathbf{y}}

\def \dtheta {{d}_{\mathbf{\boldsymbol{\theta}}}}
\def \btheta {\boldsymbol{\theta}}
\def \bxi {\boldsymbol{\xi}}
\def \bm {\mathbf{m}}
\newcommand{\norm}[1]{\left\lVert#1\right\rVert}
\DeclareMathOperator{\Tr}{Tr}

\title{Gaussian processes for Bayesian inverse problems associated with linear partial differential equations}
\author{Tianming Bai$^{1,2}$, Aretha L. Teckentrup$^{1,2}$, Konstantinos C. Zygalakis$^{1,2}$}
\date{\today}

\begin{document}

\maketitle

\begin{abstract}

This work is concerned with the use of Gaussian surrogate models for Bayesian inverse problems associated with linear partial differential equations. A particular focus is on the regime where only a small amount of training data is available. In this regime the type of Gaussian prior used is of critical importance with respect to how well the surrogate model will perform in terms of Bayesian inversion. We  extend the framework of Raissi et. al. (2017) to construct PDE-informed Gaussian priors that we then use to construct different approximate posteriors.  A number of different numerical experiments illustrate the superiority of the PDE-informed Gaussian priors over more traditional priors. 


\end{abstract}

\footnotetext[1]{School of Mathematics, University of Edinburgh, Edinburgh, Scotland, UK}
\footnotetext[2]{Maxwell Institute for Mathematical Sciences, Bayes Centre, 47 Potterrow, Edinburgh, Scotland, UK.}

\section{Introduction}

Combining complex mathematical models with observational data is an extremely challenging yet ubiquitous problem in the field of modern applied mathematics and data science. Inverse problems, where one is interested in learning inputs to a mathematical model such as physical parameters and initial conditions given partial and noisy observation of model outputs, are hence of frequent interest. Adopting a Bayesian approach \cite{ks05,S10}, we incorporate our prior knowledge on the inputs into a probability distribution, \emph{the prior distribution}, and obtain a more accurate representation of the model inputs in the \emph{posterior distribution}, which results from conditioning the prior distribution on the observed data.  

The posterior distribution contains all the necessary information about the characteristics of our inputs. However, in most cases the posterior is unfortunately intractable and one needs to resort to sampling methods 
such as Markov chain Monte Carlo (MCMC) \cite{RC04,BGJ11} to explore it.  A major challenge in the application of MCMC methods to problems of practical interest is the large computational cost associated with numerically solving the mathematical model for a given set of the input parameters. Since the generation of each sample by an MCMC method
requires a solve of the governing equations, and often millions of 
samples are required in practical applications, this process
can quickly become very costly.

One way to deal with the challenge of full Bayesian inference for 
complex models is the use of surrogate models, also known as emulators, meta-models or reduced order models. In particular, instead 
of using the complex (and computationally expensive) model, one uses a 
simpler and computationally more efficient model to approximate the 
solution of the governing equations, which in turn is used to approximate 
the data likelihood. Within the statistics literature, the most 
commonly used type of surrogate model is a Gaussian process emulator
\cite{RW06,S99,SWM89,KH00,H06,HKC04}, but other types of surrogate models can also be used including projection-based methods \cite{BWG08},
generalised Polynomial Chaos \cite{XK03,MNR07},  sparse grid 
collocation \cite{BNT07,MX09}   and  adaptive subspace methods 
\cite{C15,CDW14}.

In this paper, we focus on the use of Gaussian process surrogate models for approximating the posterior distribution in inverse problems, where the forward model relates to the solution of a linear partial differential equation (PDE). In particular,  we consider two different ways of using the surrogate model, by emulating either the parameter-to-observation map or the
negative log-likelihood. Convergence properties of 
the corresponding posterior approximations, as the number of design 
points $N$ used to construct the surrogate model goes to infinity, have recently been studied in 
\cite{ST16,t20,helin2023introduction}. These results put the methodology on a firm theoretical footing, and show that the error in the approximate posterior distribution can be bounded by the corresponding error in the surrogate model. Furthermore, the error in the approximate posteriors tends to zero as $N$ tends to infinity.
However, when the forward model of interest is given by a complex model such as a PDE,  one normally operates in a regime where only a very limited number of design points $N$ can be used due to constraints on computational cost. 
This setting is less understood and is the setting of main interest in this paper.

With a small number of design points, different modelling choices made in the derivation of the approximate posterior can have a large effect on its accuracy. In particular, the choice of Gaussian prior distribution in the emulator is crucial, as it heavily influences its accuracy. Intuitively, we want to make the prior distribution as informative as possible, by incorporating known information about the underlying forward model. For example, an informed prior specially tailored  to  solving the forward problem in  linear PDEs can be found in \cite{RRPK17}. For incorporating more general constraints, we refer the reader to the recent review \cite{sgfsj21}. Other modelling choices that require careful consideration are whether we build a surrogate model for the parameter-to-observation map or the log-likelihood directly, and whether we use the full distribution of the emulator or only the mean (see e.g \cite{ST16,lst18}).

The focus of this paper is on computational aspects of the use of Gaussian process surrogate models in PDE inverse problems, with particular emphasis on the setting where the number of design points is limited by computational constraints. The main contributions of this paper are the following:
\begin{enumerate}
    \item We extend the PDE-informed Gaussian process priors from \cite{RRPK17} to enable their use in inverse problems, which requires a Gaussian process prior as a function of both the spatial variable of the PDE and the unknown parameter(s).
    \item By showing that the required gradients can be computed explicitly, we establish that gradient-based MCMC samplers such as the Metropolis-adjusted Langevin algorithm (MALA) can be used to efficiently sample from the approximate posterior distributions. 
    \item Using a range of numerical examples, we demonstrate the isolated effects of various modelling choices made, and thus offer valuable insights and guidance for practitioners. This includes choices on posterior approximation in the inverse problem (e.g. emulating the parameter-to-observation map or the log-likelihood) and on prior distributions for the Gaussian process emulator (e.g. black-box or PDE-constrained).
\end{enumerate}
The rest of the paper is organised as follows. In Section \ref{sec:set_up} we set up notation with respect to the inverse problems of interest, as well as discuss the different kinds of posterior approximations that result from using Gaussian surrogate models for the data-likelihood. We then proceed in Section \ref{sec:methodology} to present our main methodology, discussing how can one blend better-informed Gaussian surrogate models with inverse problems as well as presenting the  MCMC algorithm that we use.  A number of different numerical experiments that illustrate the computational benefits of our approach are then presented in Section \ref{sec:numerics}, while finally  Section \ref{sec:conclusions} provides a summary and discussion of the main results.

\section{Preliminaries}\label{sec:set_up}
We now give more details about the type of inverse problems that we consider in this paper as well as discuss different aspects of Gaussian emulators and the corresponding type of approximate posteriors that we consider in this work. At the end of this section, we summarise in Table \ref{tab:symbols} all the different notations introduced in this section.

\subsection{PDE Inverse problems}

Consider the linear PDE
\begin{subequations}
\label{eq:PDE_lin}
\begin{align}
   \mathcal{L}^{\btheta} u(\mathbf{x}) &= f(\mathbf{x}), \qquad \mathbf x \in D, \label{eq:pde1} \\
   \mathcal B u(\mathbf{x}) &= g(\mathbf{x}), \qquad \mathbf x \in \partial D, \label{eq:pde2}
\end{align}
\end{subequations}
posed on a computational domain $D \subseteq \mathbb{R}^{{d}_{\mathbf{x}}}$,  where $\mathcal{L}^{\btheta}$ denotes a linear differential operator depending on parameters $\boldsymbol{\theta} \in \mathcal{T} \subseteq \mathbb{R}^{{d}_{\boldsymbol{\theta}}}$ and the linear operator $\mathcal B$ incorporates boundary conditions. 
The inverse problem of interest in this paper is to infer the parameters {$\boldsymbol{\theta}$ from the noisy data $\mathbf{y} \in \mathbb{R}^{{d}_{\mathbf{y}}}$} given by 
\begin{equation}\label{equ:ivp}
    \mathbf{y} = \mathcal{G}_{X}(\boldsymbol{\theta}) + \boldsymbol\eta,
\end{equation}
where $X=\{\mathbf{x}_{1},\cdots,\mathbf{x}_{{d}_{\mathbf{y}}} \} \subset \overline D$ are the spatial points where we observe {the solution $u$ of our PDE}, ${\mathcal{G} _{X}:\mathcal{T}\rightarrow\mathbb{R}^{{d}_{\mathbf{y}}}}$ is the \emph{parameter-to-observation map} defined by $\mathcal{G}_{X}(\boldsymbol{\theta}) = \{u(\mathbf{x}_j; \boldsymbol{\theta})\}_{j=1}^{{{d}_{\mathbf{y}}}}$, and $\boldsymbol{\eta} \sim \mathcal{N}(0,\Gamma_{\eta})$ is an additive Gaussian noise term with covariance matrix $\Gamma_{\eta} = \sigma_{\eta}^2 I_{{d}_{\mathbf{y}}}$. Note that the assumption of Gaussianity and diagonal noise covariance is done for simplicity, but these assumptions can be relaxed \cite{lst18}. Likewise, the methodology generalises straightforwardly to general bounded linear observation operators applied to the PDE solution $u$. 

To solve the inverse problem we will adopt a Bayesian approach \cite{S10}. That is, prior to observing the data $\mathbf{y}$, $\boldsymbol{\theta}$ is assumed to be distributed according to a prior density $\pi_0(\btheta),$ and we are interested in the updated posterior density $\pi(\boldsymbol{\theta}|{\mathbf{y}})$. From  \eqref{equ:ivp} we have $
    \mathbf{y}|\btheta \sim \mathcal{N}(\mathcal{G}_{X}(\boldsymbol{\theta}),\Gamma_{\eta})$, so the {\em likelihood} is
\begin{equation}\label{equ:likelihood}
   L(\mathbf{y}|{\btheta}) \propto \exp{\left(-\frac{1}{2}\|\mathcal{G}_{X}(\boldsymbol{\theta})-\mathbf{y}\|^2_{\Gamma_{\eta}}\right)}\coloneqq\exp{\left(-\Phi(\boldsymbol{\theta},\mathbf{y})\right)},
\end{equation}where the function $\Phi:\mathcal{T}\times\Rdz\rightarrow \mathbb{R}$ is called the \emph{negative log-likelihood} or \emph{potential}  and $\|\mathbf{z}\|_{\Gamma_\eta} = \mathbf{z}^\mathrm{T} \Gamma_\eta^{-1} \mathbf{z}$ denotes the norm weighted by $\Gamma_\eta^{-1}$.
Then by Bayes' formula we have 
\begin{equation} \label{eq:post_pde}
    \pi({\btheta}|{\mathbf{y}}) \propto L(\mathbf{y}|\boldsymbol{\theta})\pi_0(\boldsymbol{\theta}).
\end{equation}
The posterior distribution $\pi({\btheta}|{\mathbf{y}})$ is in general intractable, and we need to resort to sampling methods such as MCMC to extract information from it. However, generating a sample typically involves evaluating the likelihood and hence the solution of the PDE \eqref{eq:PDE_lin}, which can be prohibitively  costly. This motivates the use of surrogate models to emulate the PDE solution, which in turn is used to approximate the posterior and hence accelerate the sampling process.


\subsection{Gaussian processes}\label{subsec:gp}
Gaussian process regression (GPR) is a flexible non-parametric model for Bayesian inference \cite{RW06}. In particular our starting point for approximating an arbitrary function $\mathbf{g}:\mathcal{T}\rightarrow \mathbb{R}^d$, for some $d \in \mathbb{N}$, in the absence of any observations is the following Gaussian process prior
\begin{equation}\label{equ:gpprior}
\mathbf{g}_{0}(\boldsymbol{\theta}) \sim \textrm{GP}(\mathbf{m}(\boldsymbol{\theta}), K(\boldsymbol{\theta},\boldsymbol{\theta}')),
\end{equation}
where $\mathbf{m}:\mathcal{T}\rightarrow\mathbb{R}^{d}$ is a mean function and $K(\boldsymbol{\theta},\boldsymbol{\theta}'):\mathcal{T}\times\mathcal{T}\rightarrow \mathbb{R}^{d \times d}$ is the matrix-valued covariance function which represents the covariance between the different entries of $\mathbf{g}$ evaluated at 
$\boldsymbol{\theta}$ and $\boldsymbol{\theta}'$. When emulating the forward map the function $\mathbf{g}$ corresponds to the PDE solution evaluated at $\dz$ different spatial points, and hence $d=\dz$. In contrast when emulating directly the log-likelihood $d=1$. Furthermore, the matrix $K(\boldsymbol{\theta},\boldsymbol{\theta}')$ is often assumed to take the form 
\begin{equation} \label{eq:simpl}
K(\btheta,\btheta')=k(\btheta,\btheta') I_{d}
\end{equation}
for some scalar-valued covariance function $k(\boldsymbol{\theta},\boldsymbol{\theta}'):\mathcal{T}\times\mathcal{T}\rightarrow \mathbb{R}$, implying that the entries of $\mathbf{g}$ are independent. As we will see later better emulators can be constructed by relaxing this independence assumption.  

The mean function and the covariance function fully characterise our Gaussian prior. A typical choice for $\bm$ is to set it to zero, while common choices for the covariance function $k(\btheta,\btheta')$ include
 the squared exponential covariance function 
\begin{equation*}
    k(\btheta,\btheta') = \sigma^2 \exp\left(-\frac{\norm{\btheta - \btheta'}^2}{2l^2}\right),
\end{equation*}
and the Mat\'ern covariance functions
\begin{equation*}
    k(\btheta,\btheta') =  \frac{\sigma^2}{\Gamma(\nu)2^{\nu-1}}\left(\sqrt{2\nu}\frac{\norm{\btheta-\btheta'}}{l}\right)^\nu B_{\nu}\left(\sqrt{2\nu}\frac{\norm{\btheta-\btheta'}}{l}\right).
\end{equation*}
For both kernels, the hyperparameter $\sigma^2 >0$ governs the magnitude of the covariance and the hyperparameter  $l>0$ governs the length-scale at which the entries of $\mathbf{g}_0(\btheta)$ and $\mathbf{g}_0(\btheta')$ are correlated. For the Mat\'ern covariance function the smoothness of the entries of $\mathbf{g}_0$ depends on the positive hyper-parameter $\nu$, while in the limit $\nu\rightarrow \infty$ we obtain the squared exponential covariance function which gives rise to infinitely differentiable sample paths for $\mathbf{g}_0$.

Now suppose that we are given data in the form of $N$ distinct  design points $\Theta = \{\btheta^i\}_{i = 1}^{N} \in \mathbb{R}^{\dtheta\times N}$ with corresponding function values
\[
\mathbf{g}(\Theta):=[\mathbf{g}(\btheta^1);\cdots; \mathbf{g}(\btheta^N)] \in \mathbb{R}^{(\dz \times N)\times 1}.
\]
Since we have assumed that the multi-output function $\mathbf{g}_{0}$ is a Gaussian process, the matrix vector
\[
[\mathbf{g}_{0}(\btheta^1); \cdots; \mathbf{g}_{0}(\btheta^N);\mathbf{g}_{0}(\tilde{\btheta})] \in \mathbb{R}^{(\dz \times (N+1))\times 1}
\]
for any test point $\tilde{\btheta}$ follows a multivariate Gaussian distribution. The conditional distribution of $\mathbf{g}_0(\tilde{\btheta})$ given the set of values $\mathbf{g}(\Theta)$ is then again Gaussian with mean and covariance given  by the standard formulas for the conditioning of Gaussian random variables \cite{RW06}. In particular, if we denote with $\mathbf{g}^{N}$ the Gaussian process \eqref{equ:gpprior} conditioned on the values $\mathbf{g}(\Theta)$ we have 

\begin{equation}\label{eq:gp_post}
\mathbf{g}^{N}(\btheta) \sim  \textrm{GP}(\mathbf{m}^{\mathbf{g}}_{N}(\btheta),K_{N}(\btheta,\btheta'))
\end{equation}
where the predictive mean vector $\mathbf{m}^{\mathbf{g}}_{N}$ and the predictive covariance matrix $K_{N}(\btheta,\btheta')$ are given by 
\begin{align}
\mathbf{m}^{\mathbf{g}}_{N}(\btheta) &= \mathbf{m}(\btheta)+K(\btheta,\Theta)K(\Theta,\Theta)^{-1}\left(\mathbf{g}(\Theta)-\mathbf{m}(\Theta)\right) \label{eq:pred_mean}\\    
K_{N}(\btheta,\btheta') &= K(\btheta,\btheta') -  K(\btheta,\Theta)K(\Theta,\Theta)^{-1}K(\btheta',\Theta)^{T}, \label{eq:pred_var}
\end{align}
with $\mathbf{m}(\Theta) = [\mathbf{m}(\btheta^{1});\cdots; \mathbf{m}(\btheta^{N})] \in \mathbb{R}^{(\dz \times N)\times 1}$, $K(\btheta,\Theta) = [K(\btheta,\btheta^1),\dots,K(\btheta,\btheta^N)] \in \mathbb{R}^{ \dz\times (\dz \times N)}$ and

$$K(\Theta,\Theta) = \begin{bmatrix}
K(\btheta^1,\btheta^1)&\dots&K(\btheta^1,\btheta^N)\\
\vdots& &\vdots\\
K(\btheta^N,\btheta^1)&\dots&K(\btheta^N,\btheta^N)\\
\end{bmatrix} \in \mathbb{R}^{ (\dz \times N) \times (\dz \times N)}$$
To avoid ambiguity in the notation, we use regular font for scalar values, bold font for vector values, and capital font for matrices (details in Table \ref{tab:symbols}).

\subsection{Gaussian emulators and approximate posteriors}\label{ssec:gp_posterior}
We now discuss two different approaches for constructing a Gaussian emulator and using it for approximating the posterior of interest. The first approach constructs an emulator for the forward map $\mathcal{G}_{X}$, while the second approach is based on constructing an emulator directly for the log-likelihood.

\subsubsection{Emulating the forward map}

Given the data set ${\mathcal{G}_{X}}(\Theta)$, we can now proceed with building our Gaussian process emulation for the forward map $\mathcal{G}_{X}$. One then needs to decide how to incorporate the emulation for the construction of an approximate posterior. In particular,  depending on what type of information we plan to utilize, different approximations will be obtained. If we use its predictive mean $\mathbf{m}^{\mathcal{G}_X}_{N}$ as a point estimator of the forward map $\mathcal{G}_{X}$, we obtain
\begin{equation} \label{eq:meanpost}
    \pi^{N,\mathcal{G}_{X}}_{\mathrm{mean}}(\btheta|\mathbf{y}) \propto \exp\left(-\frac{1}{2}\|\mathbf{m}^{\mathcal{G}_X}_{N}(\btheta)-\mathbf{y}\|^2_{\Gamma_\eta}\right)\pi_0(\btheta).
\end{equation}
Alternatively, we can try to exploit the full information given by the Gaussian process by incorporating its variance in the posterior approximation. A natural way to do this is to consider the following approximation\footnote{ The derivation of \eqref{eq:postinfl} results from the fact that the convolution of two Gaussian measures is Gaussian. A detailed derivation can be found in Appendix  for completeness, the formula was also derived in  \cite{10.1063/1.4985359,calvetti2018iterative}.}:
\begin{align} 
    \pi^{N,\mathcal{G}_{X}}_{\mathrm{marginal}}(\btheta|\mathbf{y}) &\propto \mathbb{E}\left(\exp\left(-\frac{1}{2}\|\mathcal{G}_{X}^N(\btheta)-\mathbf{y}\|^2_{\Gamma_\eta}\right)\pi_0(\btheta)\right) \nonumber \\
    &\propto \label{eq:postinfl}\left(\frac{\exp\left(-\frac{1}{2}\|\mathbf{m}^{\mathcal{G}_X}_{N}(\btheta)-\mathbf{y}\|^2_{(K_{N}(\btheta,\btheta)+\Gamma_{\eta})}\right)}{\sqrt{(2\pi)^{d_\mathbf{y}}\det\left(K_{N}(\btheta,\btheta)+\Gamma_{\eta}\right)}}\right) \pi_0(\btheta),
\end{align}
Comparing \eqref{eq:postinfl} with \eqref{eq:meanpost}, the likelihood function in the marginal approximation is Gaussian with additional uncertainty $K_{N}(\btheta,\btheta)$ from the emulator included into its covariance matrix. Hence, for a fixed parameter $\btheta$, the likelihood function in \eqref{eq:postinfl} will be less concentrated due to the variance inflation. When the magnitude of $K_{N}(\btheta,\btheta)$ is small compared to that of $\Gamma_{\eta}$, the marginal approximation will be similar to the mean-based approximation.
 
\subsubsection{Emulating the log-likelihood} \label{sec:emulatell}
Another way of building the emulator is to model the potential function $\Phi$ directly. We can convert the data set $\mathcal{G}_{X}(\Theta)$ into a data set of negative log-likelihood $\boldsymbol\Phi(\Theta) = \{\Phi(\btheta^i,\mathbf{y})\}_{i = 1}^{N}$. Again, if we only include the mean of the Gaussian process emulator the posterior approximation becomes
\begin{equation}\label{eq:potenmean}
    \pi^{N,\Phi}_{\mathrm{mean}}(\btheta|\mathbf{y}) \propto  \exp\left(-m^{\Phi}_{N}(\btheta)\right)\pi_0(\mathbf{\btheta}),
\end{equation}
while, in a similar fashion to the forward map emulation, we can take into account the covariance of our emulator to obtain the approximate posterior
\begin{align}\label{eq:poteninfl}
    \pi^{N,\Phi}_{\mathrm{marginal}}(\btheta|\mathbf{y}) &\propto \mathbb{E}\left((\exp{\left(-\Phi^N(\boldsymbol{\theta})\right)})\pi_{0}(\boldsymbol{\theta})\right) \nonumber   \\
    &\propto \exp{\left(-m^{\Phi}_{N}(\boldsymbol{\theta})+ \frac{1}{2}k_{N}(\btheta,\btheta)\right)}\pi_{0}(\boldsymbol{\theta}).
\end{align}
The derivation of \eqref{eq:poteninfl} is similar to that of \eqref{eq:postinfl}. Note that in this case, the following relationship holds between the two approximate posteriors
\[
\pi^{N,\Phi}_{\mathrm{marginal}}(\btheta|\mathbf{y}) \propto \pi^{N,\Phi}_{\mathrm{mean}}(\btheta|\mathbf{y}) \exp{\left(\frac{1}{2}k_{N}(\btheta,\btheta)\right)},
\]
which again illustrates a form of variance inflation for the marginal posterior approximation.

In summary, we have two methods for approximating posteriors: the mean-based approximation and the marginal approximation; and we have two types of emulators: the forward map emulator and the potential function emulator; thus by combination we have four types of approximations in total. The convergence properties of all these approximate posteriors where the subject of study in \cite{ST16,t20,helin2023introduction}, where it was proved under suitable assumptions that all of them converge to the true posterior as $N \rightarrow \infty$. However, in the case of small $N$, the difference between the approximate posteriors could be large and which one we choose is important. Furthermore, the type of Gaussian process emulator used plays an even bigger role in this case, and one would like to use a Gaussian prior that is as informative as possible. We discuss how to do this in the next section.




\begin{table}[!htbp]

\begin{tabularx}{\textwidth}{ c|X }
    \hline
    \textbf{Symbol}  & \textbf{Description} \\
    \hline
    $\btheta$ & Unknown parameter in PDE\\
    $\mathcal{T}$ & Space of unknown parameter\\
    $\mathbf{y}$ & Discrete observation of PDE solution\\
    ${d_{\btheta}}$, $d_\mathbf{y}$, $d_\mathbf{x}$ & Dimension of vector space\\
    $\boldsymbol{\eta}$, $\Gamma_{\eta}$, $\sigma^2_{\eta}$ & Gaussian noise $\boldsymbol{\eta}$ with zero mean and covariance matrix $\Gamma_{\eta} = \sigma^2_{\eta}I_{\dz}$\\
    $X = \{\mathbf{x}_1,\mathbf{x}_2,\cdots,\mathbf{x}_{d_\mathbf{y}}\}$ & Set of spatial points corresponding to the observation $\boldsymbol{y}$\\
    $\mathcal{G}_X$ &  $\mathcal{G}_X: \mathcal{T} \rightarrow \mathbb{R}^{d_\mathbf{y}}$ parameter-to-observation map\\
    $\pi(\btheta|\mathbf{y})$ & Posterior\\
    $L(\mathbf{y}|\btheta)$ & Likelihood\\
    $\pi_0(\btheta)$ & Prior\\
    $\Phi(\btheta,\mathbf{y})$ & Negative log-likelihood (or potential)\\
    $\sigma^2$, $l$ & Hyper-parameters in Gaussian covariance function (variance $\sigma^2$ and length-scale $l$)\\
    $\textrm{GP}(\mathbf{m}(\boldsymbol{\theta}), K(\boldsymbol{\theta},\boldsymbol{\theta}'))$ & Gaussian process with mean function $\mathbf{m}(\boldsymbol{\theta})$ and matrix-valued covariance function $K({\btheta},{\btheta}')$\\
    $k(\btheta,\btheta')$ & Scalar-valued covariance function \\
    $\mathcal{G}_X(\Theta)$ & Training data set of function values at $\Theta = \{\btheta^i\}_{i = 1}^{N} \in \mathbb{R}^{\dtheta\times N}$ \\
    $\mathcal{G}^{N}_{X}(\btheta)$ & Gaussian process conditioned on data $\mathcal{G}_X(\Theta)$\\
    $\mathbf{m}^{\mathcal{G}_X}_{N}(\btheta)$, $K_{N}(\btheta,\btheta')$ & Predictive mean and predictive covariance of $\mathcal{G}^{N}_{X}(\btheta)$\\   $\mathcal{L}^{\btheta}_{\mathbf{x}}$ & Differential operator of PDE with parameter $\btheta$\\
    $u$, $f$ & PDE solution $u$ and sourcing term $f$\\
    $\pi^{N,\mathcal{G}_X}_{\mathrm{mean}}$, $\pi^{N,\mathcal{G}_X,s}_{\mathrm{mean}}$, $\pi^{N,\mathcal{G}_X,\mathrm{PDE}}_{\mathrm{mean}}$ & Mean-based posterior with baseline, spatially correlated  and PDE-constrained emulator\\
    $\pi^{N,\mathcal{G}_X}_{\mathrm{marginal}}$, $\pi^{N,\mathcal{G}_X,s}_{\mathrm{marginal}}$, $\pi^{N,\mathcal{G}_X,\mathrm{PDE}}_{\mathrm{marginal}}$ & Marginal posterior with baseline, spatially correlated  and PDE-constrained emulator\\
    $\Phi(\Theta)$ & Training data set of potential function values at $\Theta$\\
    $\Phi^N(\btheta)$ & Gaussian process conditioned on data $\Phi(\Theta)$\\
    $m^{\Phi}_{N}(\btheta)$, $k_{N}(\btheta,\btheta')$ &  Predictive mean and covariance of $\Phi^N$\\
    $\pi^{N,\Phi}_{\mathrm{mean}}$, $\pi^{N,\Phi}_{\mathrm{marginal}}$ &  Mean-based and marginal posterior with emulation of potential function\\
    $k_p(\btheta,\btheta')$, $k_s(\mathbf{x},\mathbf{x}')$ & Scalar-valued covariance function for parameter and spatial coordinate\\
    $K_p(\btheta,\btheta')$, $K_s(\mathbf{x},\mathbf{x}')$ & Matrix-valued covariance function for parameter and spatial coordinate\\ 
    \hline    
\end{tabularx}
    \caption{The list of symbols and notations used in this paper.}
    \label{tab:symbols}
\end{table}

\section{Methodology} \label{sec:methodology}

Having described the different types of posterior approximations we will consider, in this section we discuss different modelling approaches for the prior distribution used in our Gaussian emulators. In doing this it is important to note that the function that we are interested to emulate, in this case the forward map $\mathcal{G}_{X}(\btheta)$, depends not only on the parameters $\btheta$ of our PDE, but also on the location of the spatial observations. Thus in terms of modelling, one would like to take this into account and build spatial correlation explicitly into the prior covariance. This can be done in two different ways, the first  by prescribing some explicit form of spatial correlation, and the second  by using the fact that we know that our forward map is associated with the solution of a linear PDE. We do this in Section  \ref{subsec:cor_PDE}. It is important to note that in both cases it is possible to calculate the gradients with respect to the parameters $\btheta$ in a closed form, which can then be used to sample from the approximate posterior distributions using gradient-based MCMC methods such as MALA. We discuss this in more detail in Section \ref{subsec:MCMC}.



\subsection{Correlated and PDE-informed priors}
\label{subsec:cor_PDE}




We now discuss two different approaches to incorporate spatial correlation in our prior covariance function for the forward map $\mathcal{G}_{X}(\btheta)$. Even though this is a function from the parameter space $\mathcal{T}$ to the observation space $\mathbb{R}^{\dz}$, for introducing more complicated spatial correlation it is useful to  think first about the PDE solution $u(\btheta,\mathbf{x})$ as a function from $\mathcal{T} \times \overline D $ to $\mathbb{R}$. We introduce the prior covariance function  $k((\btheta,\mathbf{x}),(\btheta',\mathbf{x}'))$ for $u(\btheta,\mathbf{x})$,
and choose a separable model
\begin{equation}\label{eq:kernel_separable}
k((\btheta,\mathbf{x}),(\btheta',\mathbf{x}'))=k_p(\btheta,\btheta')k_s(\mathbf{x}, \mathbf{x}'),
\end{equation}
where $k_p$ and $k_s$ are the covariance functions for the parameters $\btheta$ and the spatial points $\mathbf{x}$ respectively.  Separable kernels are a common modeling assumption in Gaussian processes. The resulting covariance function will have a high value only if the kernels for both variables have a high value. 

Using the fact that the forward map $\mathcal{G}_{X}$ relates to the point-wise evaluation of the function $u(\btheta, \mathbf{x})$ for $\mathbf{x} \in X$, and assuming zero mean, we then have the Gaussian prior 
\begin{equation}\label{eq:gp_prior_spatial}
\mathcal{G}_{X}(\btheta) \sim \textrm{GP}(0,K(\btheta,\btheta')),
\end{equation}
with
 \begin{equation*}
     K(\btheta,\btheta') = k_p(\btheta,\btheta')K_s(X, X),
 \end{equation*}
where $K_s$ is the covariance matrix with entries $(K_{s}(X,X))_{i,j}=k_{s}(\mathbf{x}_{i},\mathbf{x}_{j})$, $ \mathbf{x}_{i},\mathbf{x}_{j} \in X $.  This prior can then be updated to a posterior by conditioning on data $\mathcal G_X(\Theta)$ as in Section \ref{subsec:gp}, which gives
\begin{equation}\label{eq:gp_post_spatial}
\mathcal{G}_{X}(\btheta)|\mathcal G_X(\Theta) \sim \textrm{GP}(\bm_N^{\mathcal G_X}(\btheta),K_N(\btheta,\btheta')),
\end{equation}
with
 \begin{align*}
 \bm_N^{\mathcal G_X}(\btheta) &= K_{uu}(\btheta,\Theta) K_{uu}(\Theta,\Theta)^{-1} \mathcal G_X(\Theta),\\
     K_N(\btheta,\btheta') &= K(\btheta,\btheta') -  K_{uu}(\btheta,\Theta) K_{uu}(\Theta,\Theta)^{-1} K(\btheta', \Theta),
 \end{align*}
 and
 \begin{equation*}
     K_{uu}(\Theta,\Theta) = 
    \{k_p(\btheta^{i},\btheta^j)K_s(X,X)\} \in \mathbb{R}^{N \dz\times N \dz}, \quad \text{similarly} \quad K_{uu}(\btheta,\Theta) \in \mathbb{R}^{\dz\times N \dz}.
 \end{equation*}

The second way of introducing spatial correlation is to  explicitly take into account that the forward map is related to a PDE solution. Given the PDE system
\begin{align*}
   \mathcal{L}^{\btheta} u(\mathbf{x}) &= f(\mathbf{x}), \qquad \mathbf x \in D, \\
   \mathcal B u(\mathbf{x}) &= g(\mathbf{x}), \qquad \mathbf x \in \partial D, 
\end{align*}
as in Section \ref{sec:set_up}, we can build a joint prior between $u$, $f$ and $g$. In particular, if we take fixed points $\mathbf{x},\mathbf{x}_{f} \in D$ and $\mathbf{x}_{b} \in \partial D$ we have that
\begin{equation}
\label{eq:step_first}
    \begin{bmatrix}
    u(\boldsymbol{\theta},\mathbf{x})\\g(\btheta,{\mathbf{x}_{b}})\\f(\btheta,\mathbf{x}_{f})
    \end{bmatrix}\sim \textrm{GP}\left(\boldsymbol{0}, k_p(\btheta,\btheta')\begin{bmatrix}
    k_s(\mathbf{x},\mathbf{x}) & \mathcal{B}k_s(\mathbf{x},\mathbf{x}_{b}) & \mathcal{L}^{\btheta'}k_s(\mathbf{x},\mathbf{x}_{f})\\
    \mathcal{B}k_s(\mathbf{x}_{b},\mathbf{x}) & \mathcal{B}\mathcal{B}k_s(\mathbf{x}_{b},\mathbf{x}_{b}) & \mathcal{B}\mathcal{L}^{\btheta'}k_s(\mathbf{x}_{b},\mathbf{x}_{f})\\
    \mathcal{L}^{\btheta}k_s(\mathbf{x}_{f},\mathbf{x}_{b}) & \mathcal{L}^{\btheta}\mathcal{B}k_s(\mathbf{x}_{f},\mathbf{x}_{b}) & \mathcal{L}^{\btheta}\mathcal{L}^{\btheta'}k_s(\mathbf{x}_{f},\mathbf{x}_{f})\\
    \end{bmatrix}\right),
\end{equation}  
 where the above is a Gaussian process as a function of $\btheta$, and we have used known properties of linear 
operators applied to Gaussian processes (see e.g. 
\cite{matsumoto2023images}) in the derivation. The idea of a joint prior 
between $u$ and $f$ was also used in \cite{RRPK17,spitieris2023bayesian}, with 
the crucial difference that $u$ and $f$ were considered as 
functions of the spatial variable $\mathbf{x}$ only. In 
the context of inverse problems and emulators as considered in this 
work, we instead explicitly model the dependency of $u$ on $\boldsymbol{\theta}$, which requires an extension of the methodology.

Now as in the spatially correlated case, we can use the formula \eqref{eq:step_first} to obtain a (joint) prior for $\mathcal{G}_{X}(\btheta)$. More precisely, we have
\begin{equation}\label{eq:gp_prior_joint}
    \begin{bmatrix}
  \mathcal{G}_{X}(\btheta)\\g(\btheta,X_{g})\\f(\btheta,X_f)
    \end{bmatrix}\sim \textrm{GP}\left(\boldsymbol{0}, K(\btheta,\btheta')\right),
\end{equation}
 where 
\[
K(\btheta,\btheta') = k_p(\btheta,\btheta')\begin{bmatrix}
    K_s(X,X) & \mathcal{B}K_s(X,X_{g}) & \mathcal{L}^{\btheta'}K_s(X,X_f)\\
    \mathcal{B}K_s(X_{g},X) & \mathcal{B}\mathcal{B}K_s(X_{g},X_{g}) & \mathcal{B}\mathcal{L}^{\btheta'}K_s(X_{g},X_f)\\
    \mathcal{L}^{\btheta}K_s(X_f,X) & \mathcal{L}^{\btheta}\mathcal{B}K_s(X_f,X_{g}) & \mathcal{L}^{\btheta}\mathcal{L}^{\btheta'}K_s(X_f,X_f)\\
    \end{bmatrix}
\]
and $X_{g} \subset \partial D$ and $X_{f} \subset D$ are collections of $d_g$ and $d_f$ points at which $g$ and $f$ have been evaluated, respectively. Note that the marginal prior placed on $\mathcal G_X$ is the same as in \eqref{eq:gp_prior_spatial}.

We can then condition the joint Gaussian process prior \eqref{eq:gp_prior_joint} as in Section \ref{subsec:gp} on observations $\mathbf{g}(\Theta)$, where now \[
\mathbf{g} = \begin{bmatrix}
    \mathcal G_X(\cdot) \\ g(\cdot,X_g) \\ f(\cdot,X_f)
\end{bmatrix} : \mathcal T \rightarrow \mathbb{R}^{\dz + d_g + d_f}. \]
After a re-ordering of the observations $\mathbf{g}(\Theta)$, this results in the conditional distribution
\begin{equation*}
    \mathbf{g}(\btheta)|\mathbf{g}({\Theta}) \sim \textrm{GP}\left(\bm_N^\mathbf{g}(\btheta), 
    K_N(\btheta,\btheta')\right),
\end{equation*} 
where
\begin{align*}
\bm_N^\mathbf{g}(\btheta) &= \tilde K(\btheta,\Theta) \tilde K(\Theta,\Theta)^{-1} \mathbf{g}(\Theta),  \\  
    K_N^\mathbf{g}(\btheta,\btheta') &= 
    K(\btheta,\btheta') - \tilde K(\btheta,\Theta) \tilde K(\Theta,\Theta)^{-1} \tilde K(\btheta',\Theta)^\mathrm{T},
\end{align*}
with $K(\btheta,\btheta') = k_p(\btheta,\btheta') K_s(X,X)$ as before and 
\begin{align*}
    \tilde K(\btheta,\Theta) &= \begin{bmatrix}
    K_{uu}(\btheta,\Theta) & K_{ug}(\btheta,\Theta) & K_{uf}(\btheta,\Theta) \\ K_{ug}^\mathrm{T}(\btheta,\Theta) & K_{gg}(\btheta,\Theta) & K_{gf}(\btheta,\Theta) \\
    K_{uf}^\mathrm{T}(\btheta,\Theta) &K_{gf}^{T}(\btheta,\Theta) & K_{ff}(\btheta,\Theta)
    \end{bmatrix} \in \mathbb{R}^{(\dz+d_f+d_g) \times N(\dz+d_f+d_g)}, \\
    \tilde K(\Theta,\Theta) &= \begin{bmatrix}
    K_{uu}(\Theta,\Theta) & K_{ug}(\Theta,\Theta) & K_{uf}(\Theta,\Theta) \\ K_{ug}^\mathrm{T}(\Theta,\Theta) & K_{gg}(\Theta,\Theta) & K_{gf}(\Theta,\Theta) \\
    K_{uf}^\mathrm{T}(\Theta,\Theta) &K_{gf}^{T}(\Theta,\Theta) & K_{ff}(\Theta,\Theta)
    \end{bmatrix} \in \mathbb{R}^{N(\dz+d_f+d_g) \times N(\dz+d_f+d_g)}, \\
    \mathbf{g}(\Theta) &= \begin{bmatrix}
    \mathcal G_X(\Theta) \\
    g({\Theta},X_{g})\\f({\Theta},X_f)
    \end{bmatrix} \in \mathbb R^{N(\dz+d_f+d_g)},
\end{align*}
and
\begin{align*}
    &K_{uu}(\Theta,\Theta) = 
    \{k_p(\btheta^{i},\btheta^j)K_s(X,X)\} \in \mathbb{R}^{N \dz\times N \dz}, \quad \text{similarly} \quad K_{uu}(\btheta,\Theta) \in \mathbb{R}^{\dz\times N \dz}, \\
    &K_{ug}(\Theta,\Theta) = 
    \{k_p(\btheta^i,\btheta^{j})\mathcal{B}K_s(X,X_{g})\} 
    \in \mathbb{R}^{N\dz\times N d_g}, \quad \text{similarly} \quad K_{ug}(\btheta,\Theta) \in \mathbb{R}^{\dz\times N d_g},\\
    &K_{uf}(\Theta,\Theta) = 
    \{k_p(\btheta^i,\btheta^{j})\mathcal{L}^{\btheta^{j}}K_s(X,X_f)\}\in \mathbb{R}^{N \dz\times N d_f},\quad \text{similarly} \quad K_{uf}(\btheta,\Theta) \in \mathbb{R}^{\dz\times N d_f},\\
    &K_{gg}(\Theta,\Theta) = 
    \{k_p(\btheta^{i},\btheta^{j})\mathcal{B}\mathcal{B}K_s(X_{g},X_{g})\}\in \mathbb{R}^{N d_g \times N d_g},\quad \text{similarly} \quad K_{gg}(\btheta,\Theta) \in \mathbb{R}^{d_g \times N d_g},\\  
    &K_{gf}(\Theta,\Theta) = \{k_p(\btheta^{i},\btheta^{j})\mathcal{B}\mathcal{L}^{\btheta^{j}}K_s(X_{g},X_f)\} \in \mathbb{R}^{N d_g \times N d_f},\quad \text{similarly} \quad K_{gf}(\btheta,\Theta) \in \mathbb{R}^{d_g \times N d_f},\\
    &K_{ff}(\Theta,\Theta) = \{k_p(\btheta^{i},\btheta^{j})\mathcal{L}^{\btheta^{i}}\mathcal{L}^{\btheta^{j}}K_s(X_f,X_f) \} \in \mathbb{R}^{N d_f \times N d_f},\quad \text{similarly} \quad K_{ff}(\btheta,\Theta) \in \mathbb{R}^{d_f \times N d_f},\\  
    &g(\Theta,X_{g}) = \{ g(\btheta^{i},X_{g}) \}  \in \mathbb{R}^{N d_g}, \\ &f({\Theta},X_f) = \{ f(\btheta^{i},X_f) \}  \in \mathbb{R}^{N d_f}.
    \end{align*} 
The marginal posterior distribution on $\mathcal G_X(\btheta)$ can then be extracted from the above joint posterior by taking the first $\dz$ rows of $\bm_N^\mathbf{g}$ and the first $\dz$ rows and columns of $K_N^\mathbf{g}$, which gives

\begin{equation} \label{eq:pde_posterior}
\mathcal{G}_{X}(\btheta)|\mathcal G_X(\Theta), g(\Theta,X_g), f(\Theta,X_f) \sim \textrm{GP}(\mathbf{m}_{N,X_f,X_g}^{\mathcal G_X}(\btheta), K_{N,X_f,X_g}(\btheta,\btheta')),
\end{equation}
where 
\begin{align*}
\mathbf{m}_{N,X_f,X_g}^{\mathcal G_X}(\btheta) &= \begin{bmatrix}
    K_{uu}(\btheta,\Theta) & K_{ug}(\btheta,\Theta) & K_{uf}(\btheta,\Theta) 
    \end{bmatrix} \tilde K(\Theta,\Theta)^{-1} \mathbf{g}(\Theta),  \\  
    K_{N,X_f,X_g}(\btheta,\btheta') &= 
    K(\btheta,\btheta') - \begin{bmatrix}
    K_{uu}(\btheta,\Theta) & K_{ug}(\btheta,\Theta) & K_{uf}(\btheta,\Theta) 
    \end{bmatrix} \tilde K(\Theta,\Theta)^{-1} \begin{bmatrix}
    K_{uu}(\btheta',\Theta) \\ K_{ug}(\btheta',\Theta) \\ K_{uf}(\btheta',\Theta) 
    \end{bmatrix}.
\end{align*}
Compared to the spatially correlated posterior in \eqref{eq:gp_post_spatial}, note that in \eqref{eq:pde_posterior} we are updating our prior on $\mathcal G_X(\btheta)$ using the observations $g(\Theta,X_g)$ and $f(\Theta,X_f)$ as well as $\mathcal G_X(\Theta)$. Since the $g$ and $f$ are assumed known, these additional observations are cheap to obtain. It is also possible to extend the methodology to condition on training data $g(\Theta_g,X_g)$ and $f(\Theta_f,X_f)$, for point sets $\Theta_g$ and $\Theta_f$ different to $\Theta$, and this has been found to be beneficial in some of the numerical experiments (see Section \ref{sec:numerics}).


Note that when emulating the potential $\Phi$ instead of the forward map $\mathcal G_{X}$, we are emulating a scalar-valued function. Since $\Phi$ is a non-linear function of $\mathcal G_{X}$, it is not possible to extend the ideas of spatial correlation presented in this section to emulating $\Phi$, and in particular, it is not possible to construct a PDE-informed emulator in the same way.

\subsubsection{Computational implementation}
 We now have three different approaches for emulating the forward map  and defining the correlation between its components. We will refer to these as the independent, spatially correlated, and PDE-constrained model, respectively. Each of them can be combined with the mean-based or the marginal approximation of the posterior. We note here that for the computational implementation of the spatially correlated model, the introduction of the correlation matrix does not change the predictive mean of the Gaussian process, it only affects the predictive covariance (see Theorem \ref{thm:kronecker} below). Since the spatial correlation matrix is independent of $\btheta$, the covariance matrix between two sets of parameters $\Theta_1$ and $\Theta_2$ can be computed by the Kronecker product, that is, 
 \begin{equation}\label{eq:spkron}
     \underbrace{K(\Theta_1,\Theta_2)}_{N_1 {d}_\mathbf{y} \times N_2 {d}_\mathbf{y}} = \underbrace{K_p(\Theta_1,\Theta_2)}_{(N_1 \times N_2)}\otimes \underbrace{K_s(X,X)}_{({d}_\mathbf{y} \times {d}_\mathbf{y})}.
 \end{equation}
Hence, assuming a spatial correlation of the type \eqref{eq:kernel_separable} only affects approximate posteriors that take into account the uncertainty of the emulator.

\begin{thm} 
\label{thm:kronecker}
 Consider two Gaussian processes $\mathbf{g}_{0} (\btheta) \sim \textrm{GP}(\mathbf{m}(\btheta),k_p(\btheta,\btheta')I_{\dz})$ and $\mathbf{g}_{0,s} (\btheta) \sim \textrm{GP}(\mathbf{m}(\btheta),k_p(\btheta,\btheta')K_s(X,X))$, where $K_s(X,X)$ is the covariance matrix on the set of spatial points $X = \{\mathbf{x}_i\}_{i=1}^{\dz}$ and $k_p(\btheta,\btheta')$ is scalar-valued. Conditioning both Gaussian processes on a set of training points $\mathbf{g}(\Theta) = \{\mathbf{g}(\btheta_i)\}_{i=1}^{N}$, denote the corresponding conditional Gaussian processes by $\mathbf{g}^N (\btheta) \sim \textrm{GP}(\mathbf{m}_N^{\mathbf{g}}(\btheta),K_N(\btheta,\btheta'))$ and $\mathbf{g}^N_{s}(\btheta) \sim \textrm{GP}(\mathbf{m}_{N,s}^{\mathbf{g}}(\btheta),K_{N,s}(\btheta,\btheta'))$, respectively. Then we have,   
\begin{align*}
    \mathbf{m}_{N,s}^{\mathbf{g}}(\btheta) &= \mathbf{m}_N^{\mathbf{g}}(\btheta),\\
    K_{N}(\btheta,\btheta') &= k_{N,p}(\btheta,\btheta') I_{\dz}, \quad \text{and} \quad 
    K_{N,s}(\btheta,\btheta') = k_{N,p}(\btheta,\btheta') K_s(X,X),
\end{align*}
where $k_{N,p}(\btheta,\btheta')$ is scalar-valued.
\end{thm}

\begin{proof}
 We now prove the expression for  the predictive mean. Let $k_p(\btheta,\Theta) := [k_p(\btheta,\btheta^1); \dots ; k_p(\btheta,\btheta^N)] \in \mathbb{R}^{1 \times \dz}$, and denote by $K_p(\Theta,\Theta) \in \mathbb{R}^{\dz \times \dz}$ the matrix with entries $(K_p(\Theta,\Theta))_{i,j} = k_p(\btheta^i,\btheta^j)$. Then by \eqref{eq:pred_mean} we have
\begin{align*}
    &\mathbf{m}_{N,s}^{\mathbf{g}}(\btheta)  \\
    &= \mathbf{m}(\btheta) + \left(k_p(\btheta,\Theta)\otimes K_s(X,X)\right)^{T}
    \left(K_p(\Theta,\Theta)\otimes K_s(X,X)\right)^{-1}\left(\mathbf{g}(\Theta)-\mathbf{m}(\Theta)\right),
    \end{align*}
where $\otimes$ denotes the Kronecker product. Using properties of products and inverses of Kronecker products and the fact that $K_s(X,X)$ is symmetric positive definite, we then have
\begin{align*}
    &\mathbf{m}_{N,s}^{\mathbf{g}}(\btheta)  \\
    &= \mathbf{m}(\btheta) + \left(k_p(\btheta,\Theta)^{T}\otimes K_s(X,X)^{T}\right)
    \left(K_p(\Theta,\Theta)^{-1}\otimes K_s(X,X)^{-1}\right)\left(\mathbf{g}(\Theta)-\mathbf{m}(\Theta)\right)\\
    &= \mathbf{m}(\btheta) + \left(k_p(\btheta,\Theta)^{T}K_p(\Theta,\Theta)^{-1}\otimes K_s(X,X)^{T} K_s(X,X)^{-1}\right)\left(\mathbf{g}(\Theta)-\mathbf{m}(\Theta)\right)\\
    &= \mathbf{m}(\btheta) + \left(k_p(\btheta,\Theta)^{T}K_p(\Theta,\Theta)^{-1}\otimes I_{\dz}\right)\left(\mathbf{g}(\Theta)-\mathbf{m}(\Theta)\right)\\
    &= \mathbf{m}_{N}^{\mathbf{g}}(\btheta).
\end{align*}
The relationship between $K_{N,s}(\btheta,\btheta')$ and $K_{N}(\btheta,\btheta')$ can be shown in a similar way, using \eqref{eq:pred_var}.

\end{proof}
For the PDE-constrained model, since the covariance functions related to $f$ are obtained by applying the differential operator, the spatially correlated matrix in the joint prior \eqref{eq:step_first} also depends explicitly on the parameters $\btheta$. Therefore, its covariance matrix cannot be written in a Kronecker product structure as in \eqref{eq:spkron} and Theorem \ref{thm:kronecker} does not apply. Thus, incorporating the PDE constraints into the model also affects the predictive mean and hence the mean-based posterior is also changed. 

\subsection{MCMC algorithms}
\label{subsec:MCMC}

To extract information from the posterior, MCMC algorithms are powerful and popular tools \cite{RC04,BGJ11}. In this work, we will consider the Metropolis-Adjusted Langevin Algorithm (MALA) \cite{roberts1996exponential}, which is a type of MCMC algorithm that uses gradient information to accelerate the convergence of the sampling chain. Central to the idea of MALA, and any gradient-based sampling method in fact, is the overdamped Langevin  stochastic differential equation (SDE):
\begin{equation} \label{eq:langevin}
d\btheta = \nabla\log \pi(\btheta|{\mathbf{y}}) dt+ \sqrt{2} dW,
\end{equation}
where $W$ is a standard ${d_{\btheta}}$-dimensional Brownian motion.  Under mild conditions on the posterior $\pi$ \cite{RC04}, \eqref{eq:langevin} is ergodic and has $\pi$ as its stationary distribution, so that the probability density function of $\btheta(t)$ tends to $\pi$ as $t \rightarrow \infty$.

\begin{algorithm}[h!]
\caption{Metropolis-Adjusted Langevin Algorithm}\label{alg:MALA}

\begin{algorithmic}
\Require initial value $\btheta_0$, number of samples $N$, time-step $\gamma$, posterior $\pi(\btheta|\bz)$
\While{$n < N$} \vspace{0.5ex}
\State 1. Generate $\xi_n \sim \mathcal{N}(0,1)$. \vspace{0.5ex}
\State 2. Generate a candidate 
$$ \btheta' = \btheta_n + \gamma \nabla\log\pi(\btheta_n|\bz) + \sqrt{2\gamma}\xi_n. $$ \vspace{-2ex} 
\State 3. Compute the acceptance rate $$\alpha_n := \max\left(1,\frac{\pi(\btheta'|\bz)q(\btheta_{n}|\btheta')}{\pi(\btheta_n|\bz)q(\btheta'|\btheta_{n})}\right), $$
\State where $q(\btheta|\tilde \btheta) \propto \exp{\left( -\frac{1}{4\gamma} \|\btheta - \tilde \btheta - \gamma \nabla\log\pi(\tilde \btheta|\bz)\|^2\right)}$. \vspace{0.5ex}
\State 4. Generate $r \sim U[0,1].$ If $r > \alpha_n$, set $\btheta_{n+1} = \btheta'$; otherwise $\btheta_{n+1} = \btheta_n$. \vspace{0.5ex}
\EndWhile
\end{algorithmic}
\end{algorithm}

In  practice,  the dynamics \eqref{eq:langevin} is discretised with a simple Euler-Maruyama method with a time step $\gamma$. 
\begin{equation} \label{eq:EM}
    \btheta_{n+1} = \btheta_{n} + \gamma\nabla\log \pi(\btheta|\mathbf{y}) + \sqrt{2\gamma}\xi_n,
\end{equation} 
with $\xi_n \sim \mathcal{N}(0,1)$. Assuming that the dynamics of \eqref{eq:EM} remain ergodic the corresponding numerical invariant measure would not necessarily coincide with the posterior. To alleviate this bias, one needs to incorporate an accept-reject mechanism. This gives rise to MALA as descibed in Algorithm \ref{alg:MALA}.




 An advantage of using the Gaussian process emulator in the posterior is that, assuming the prior is differentiable, $\nabla\log \pi^N(\btheta|{\mathbf{y}})$ can be computed analytically for the mean-based and marginal approximations introduced in Section \ref{ssec:gp_posterior}, which enables us to easily implement the MALA algorithm.  Note that in contrast  since the true posterior involves (analytical or numerical) solution $u$ to the PDE \eqref{eq:pde1}-\eqref{eq:pde2}, it is usually impossible to compute these gradients analytically. The following Lemma gives us the gradient of the different approximate posteriors

\begin{lem} Given a Gaussian process $\mathcal{G}^N_X \sim \textrm{GP}(\mathbf{m}^{\mathcal{G}_X}_N(\btheta),K_N(\btheta,\btheta))$ emulating the forward map $\mathcal{G}_X$ with data $\mathcal{G}_X(\Theta)$, then the gradient of the mean-based approximation of the posterior
\begin{align*}
\nabla\log(\pi_\mathrm{mean}^{N,\mathcal{G}_X} (\btheta|\mathbf{y})) = -\frac{1}{\sigma^{2}_{\eta}} \nabla \mathbf{m}_N^{\mathcal{G}_{X}}(\btheta) ^{T}(\mathbf{m}_N^{\mathcal{G}_{X}}(\btheta)-\mathbf{y}) + \nabla \log \pi_0(\btheta),  
\end{align*}
and the gradient of the marginal approximation of the posterior
\begin{align*}\nabla\log(\pi_\mathrm{marginal}^{N,\mathcal{G}_X} (\btheta|\mathbf{y})) =
    &-\nabla \mathbf{m}_N^{\mathcal{G}_{X}}(\btheta)^T(K_N(\btheta,\btheta)+\Gamma_{\eta})^{-1}(\mathbf{m}_N^{\mathcal{G}_{X}}(\btheta)-\mathbf{y}) \\
    &- \frac{1}{2}(\mathbf{m}_N^{\mathcal{G}_{X}}(\btheta)-\mathbf{y})^T\nabla\left((K_N(\btheta,\btheta)+\Gamma_{\eta})^{-1}\right)(\mathbf{m}_N^{\mathcal{G}_{X}}(\btheta)-\mathbf{y})\\
    &- \frac{1}{2}\left( {\Tr \left((K_N(\btheta,\btheta)+\Gamma_{\eta} )^{-1}\right)\nabla(K_N(\btheta,\btheta))}\right) + \nabla \log \pi_0 (\btheta),\\
\end{align*}
where \[\nabla\left((K_N(\btheta,\btheta)+\Gamma_{\eta})^{-1}\right) =  -(K_N(\btheta,\btheta)+\Gamma_{\eta})^{-1}\nabla\left(K_N(\btheta,\btheta)\right) (K_N(\btheta,\btheta)+\Gamma_{\eta})^{-1},
\]
and \[\nabla K_N(\btheta,\btheta) = 2\nabla K(\btheta,\Theta)K(\Theta,\Theta)^{-1}K(\Theta,\btheta)\]

\end{lem}


\section{Numerical experiments} \label{sec:numerics}

We now discuss a number of different  numerical experiments related to inverse problems  for the PDE \eqref{eq:pde1}-\eqref{eq:pde2} in various set-ups in terms of the number of spatial and parameter dimensions as well as for different types of forward models. In cases where the PDE solution is not available in closed form, we use the finite element software Firedrake \cite{rathgeber2016firedrake} to obtain the "true" solution. Furthermore, in  all our numerical experiments we replace the uniform prior by a smooth approximation given by the  $\lambda-$Moreau-Yoshida envelope \cite{bauschke2011fixed} with $\lambda = 10^{-3}$. To further clarify the notation we use in our numerical experiment, we introduce part of them again in the following table (see Table \ref{tab:symbols2}).

\begin{table}[!htbp]

\begin{tabularx}{\textwidth}{ c|X }
    \hline
    \textbf{Symbol}  & \textbf{Description} \\
    \hline
    $\mathcal{G}_X(\Theta)$ & Training data set: point-wise evaluation of the PDE solution $u(\btheta,\mathbf{x})$ for $\mathbf{x} \in X = \{\mathbf{x}_i\}_{i=1}^{\dz}$, $\btheta \in \Theta = \{\btheta^i\}_{i=1}^N$\\
    $g(\Theta_g, X_g)$ & Additional training data for boundary condition: point-wise evaluation of the function $g(\btheta,\mathbf{x})$ for $\mathbf{x} \in X_g = \{\mathbf{x}_i\}_{i=1}^{d_g}$, $\btheta \in \Theta = \{\btheta^i\}_{i=1}^{N_g}$\\
    $f(\Theta_f, X_f)$ & Additional training data for the source function: point-wise evaluation of the function $f(\btheta,\mathbf{x})$ for $\mathbf{x} \in X = \{\mathbf{x}_i\}_{i=1}^{d_f}$, $\btheta \in \Theta = \{\btheta^i\}_{i=1}^{N_f}$\\
    $\Bar{N}$ & In practice, we use $N_g = N_f = \Bar{N}$\\ 
    \hline    
\end{tabularx}
    \caption{Symbols and notations used in numerical experiments.}
    \label{tab:symbols2}
\end{table}




\subsection{Examples in one spatial dimension}

\subsubsection{Constant diffusion coefficient}

We consider the following PDE in one spatial dimension
\begin{align} \label{eq:example0}
    -\frac{\mathrm{d}}{\mathrm{d}x}\left(e^{\theta}\frac{\mathrm{d}{u}(x)}{\mathrm{d}x}  \right) &= 1, \qquad x \in (0,1), \,\theta \in [-1,1], \\
    u(0) = 0,
    u(1) &= 0. \nonumber
\end{align} 
 In this case the dimension of the parameter space is $d_{\btheta}=1$, and the solution is available in closed form. More precisely, we have 
 \begin{equation*}
     u(x) = \frac{(x - x^2)}{2e^{\theta}}. 
 \end{equation*}
Given this explicit solution and the low dimension of the 
parameter space, it is possible to calculate the true and 
approximate posteriors without having to resort to Markov Chain 
Monte Carlo sampling. 
We now generate our observations $\mathbf{y}$ according to equation \eqref{equ:ivp} for the value of ${\btheta}^{\dagger} = 0.314$  for a varying number of spatial points $\dz$ (equally spaced in $[0,1]$) and for noise level $\sigma^2_{\eta} = 10^{-5}$. As we can see in Figure \ref{fig:true_post_dy} as we increase $\dz$ the true posterior $\pi(\theta|\mathbf{y})$ tends to get more and more concentrated around the value of ${\btheta}^{\dagger}$ which is consistent with what the theory would predict by a Bernstein-von-Mises theorem (see e.g. \cite{giordano2020consistency} for related results).



\begin{figure}[!ht]
    \centering
    \includegraphics[width=0.5\textwidth]{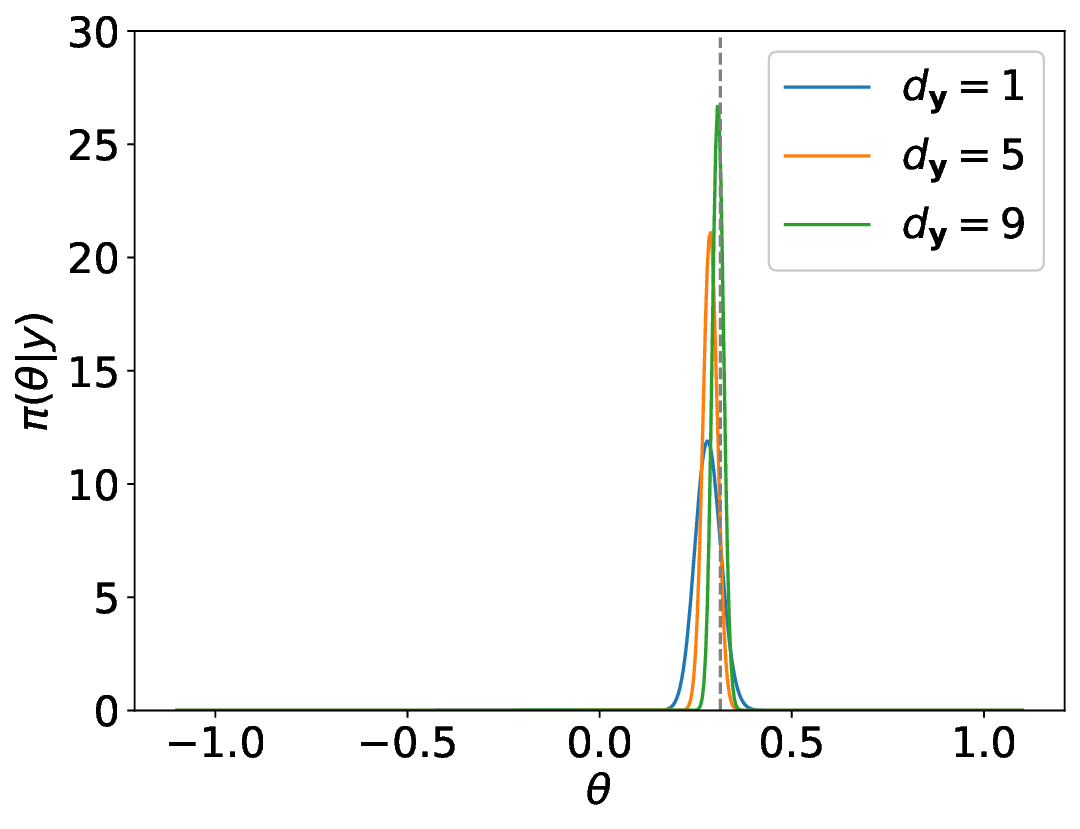}
    \caption{True posterior with different ${d}_{\mathbf{y}}$}
    \label{fig:true_post_dy}
\end{figure}

We now turn our attention to the different approximate posteriors discussed in Section 2.3 obtained for different Gaussian priors (independent, spatially correlated, and PDE-constrained). 

\paragraph{Baseline model:}
In the case of the simplest emulator with independent entries, we illustrate in Figure \ref{fig:basic0}, how the mean-based posterior $\pi^{N,\mathcal{G}_{X}}_{\mathrm{mean}}(\theta|\mathbf{y})$   and the marginal posterior $\pi^{N,\mathcal{G}_{X}}_{\mathrm{marginal}}(\theta|\mathbf{y})$ behave as a function of the number of training points $N$ (here $d_{\mathbf{y}}=5$). The location of the training points is chosen from the Halton sequence \cite{niederreiter1992random}.  Now, when comparing Figure \ref{fig:basic0}(a) and \ref{fig:basic0}(b) we see that the marginal posterior  is more spread than the mean-based posterior. This is due to the variance inflation associated with the marginal posterior which reflects better the  uncertainty of the emulator. For example, in the case $N=1$ the mean-based posterior has negligible posterior probability mass near $\btheta^{\dagger}$, while due to the variance inflation this is not the case for the marginal-based posterior. Furthermore, in Figure \ref{fig:basic0}(c) we plot the Hellinger distance between the approximate posteriors and the true posterior as a function of the number of training points $N$. As we can see the error for the marginal-based posterior is smaller than the error for the mean-based posterior for small $N$ while the two errors behave in the same way as $N$ increases. This can be further understood by Figure \ref{fig:basic0}(d) where we plot the average variance of our emulator for different values of $N$ and see that the value of $N$ for which the error between the two posteriors is equal corresponds to the value of $N$ for which the average variance of the emulator is of the same order as the variance of the observational noise $\sigma^{2}_{\eta}$.

\begin{figure}
\centering
\begin{subfigure}{.45\textwidth}
  \centering
  \includegraphics[width=\linewidth]{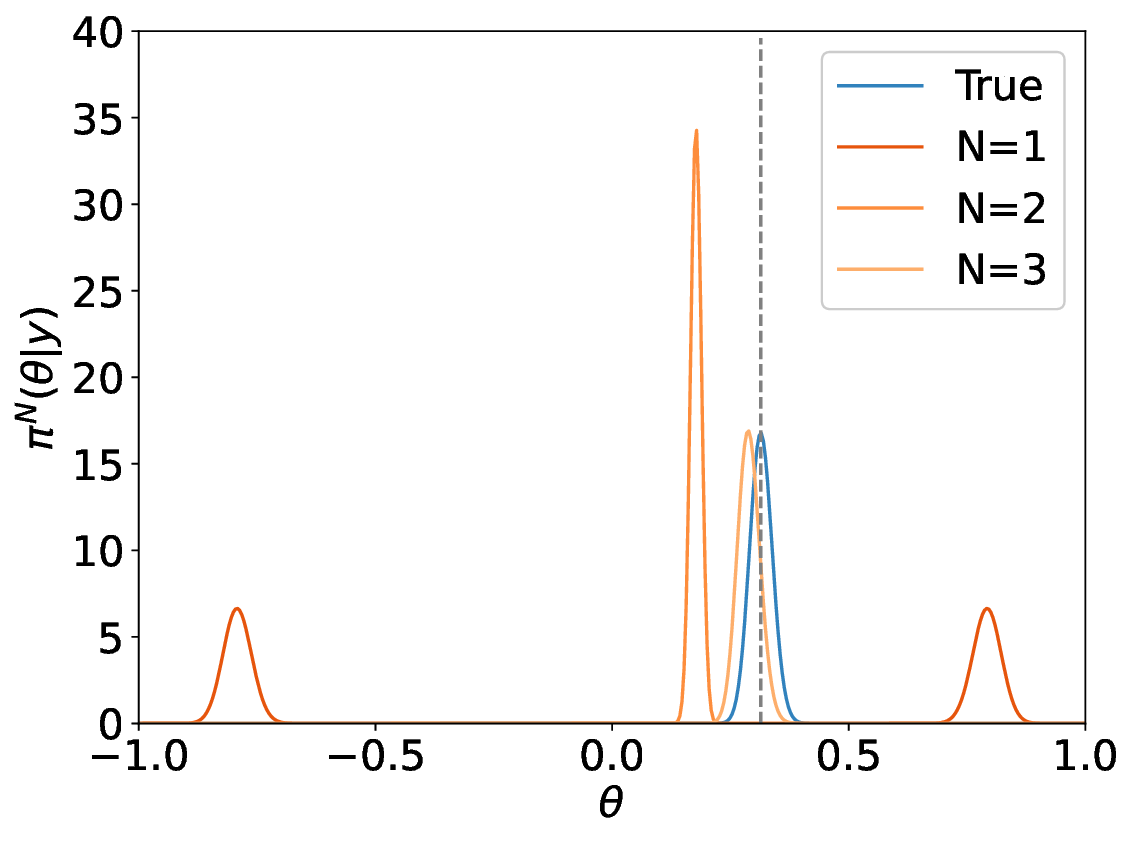}
  \caption{$\pi^{N,\mathcal{G}_{X}}_{\mathrm{mean}}(\theta|\mathbf{y})$}
  \label{fig:exp0_basic_mean}
\end{subfigure}
\begin{subfigure}{.45\textwidth}
  \centering
  \includegraphics[width=\linewidth]{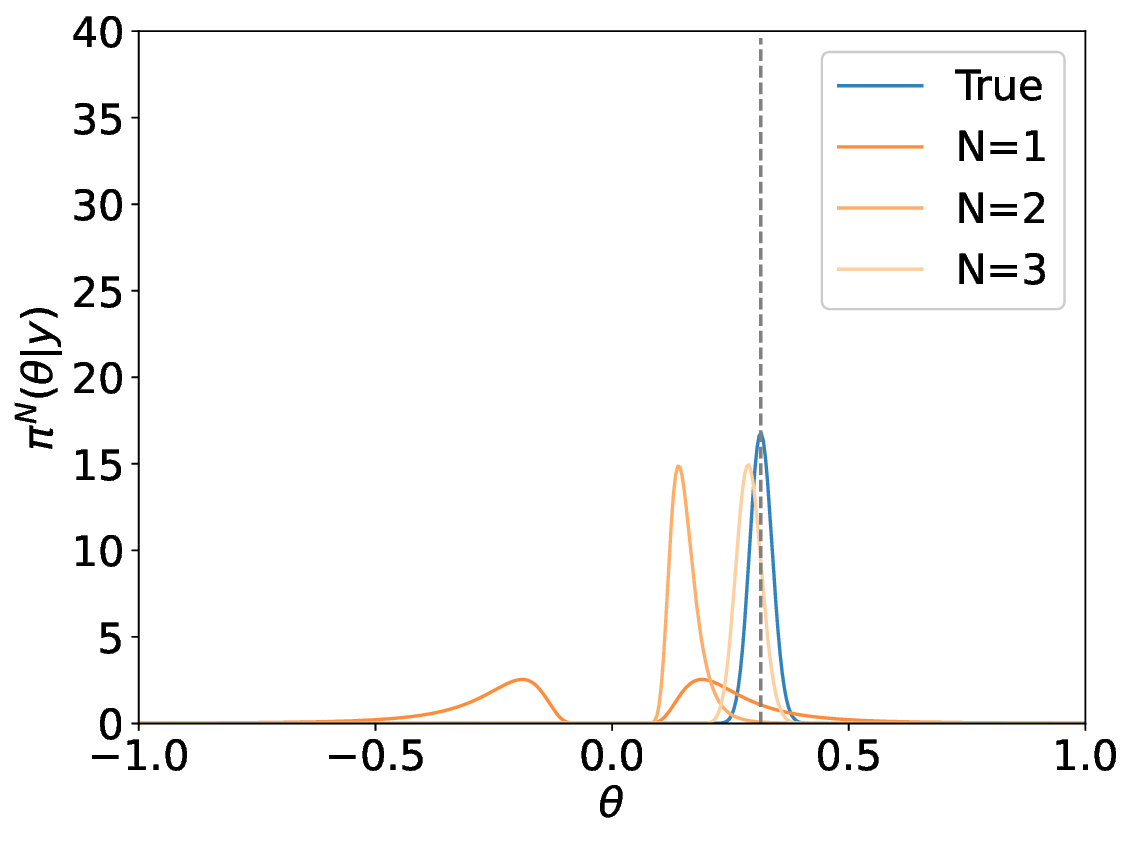}
  \caption{$\pi^{N,\mathcal{G}_{X}}_{\mathrm{marginal}}(\theta|\mathbf{y})$}
  \label{fig:exp0_basic_infl}
\end{subfigure}
\begin{subfigure}{.45\textwidth}
  \centering
  \includegraphics[width=\linewidth]{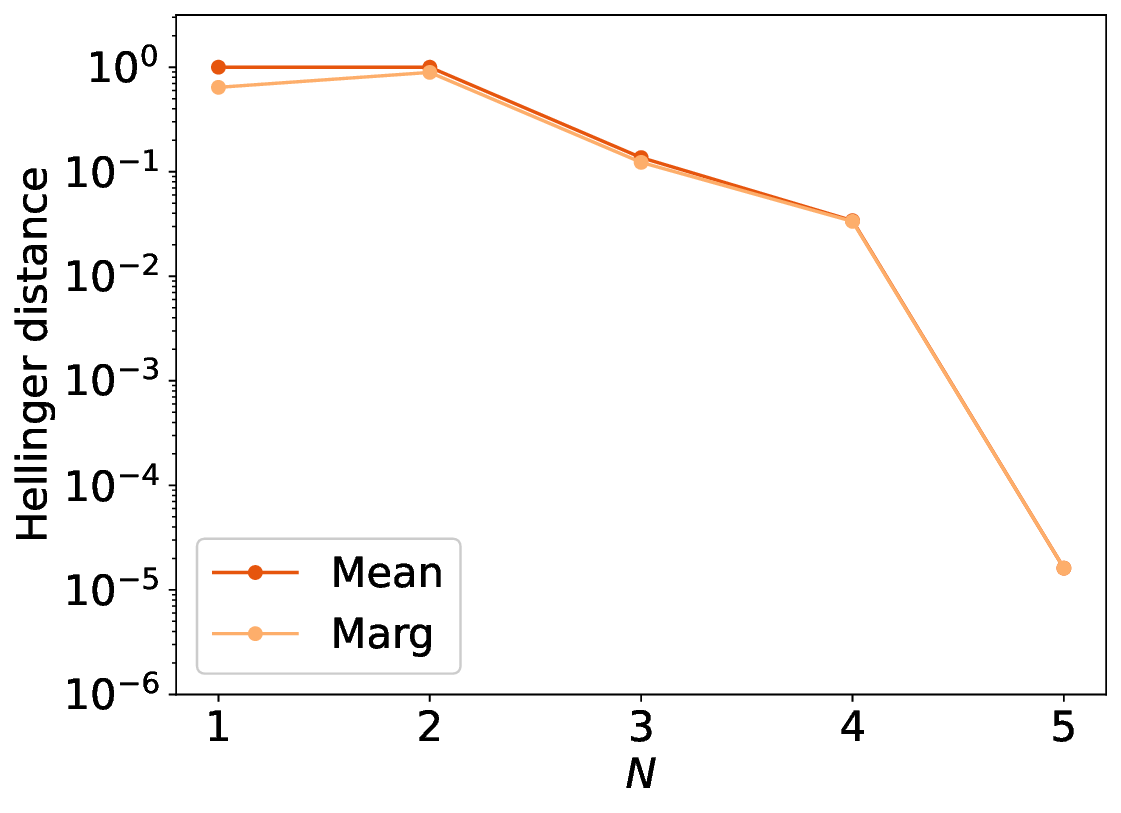}
  \caption{Hellinger distance}
  \label{fig:exp0_basic_hell}
\end{subfigure}
\begin{subfigure}{.45\textwidth}
  \centering
  \includegraphics[width=\linewidth]{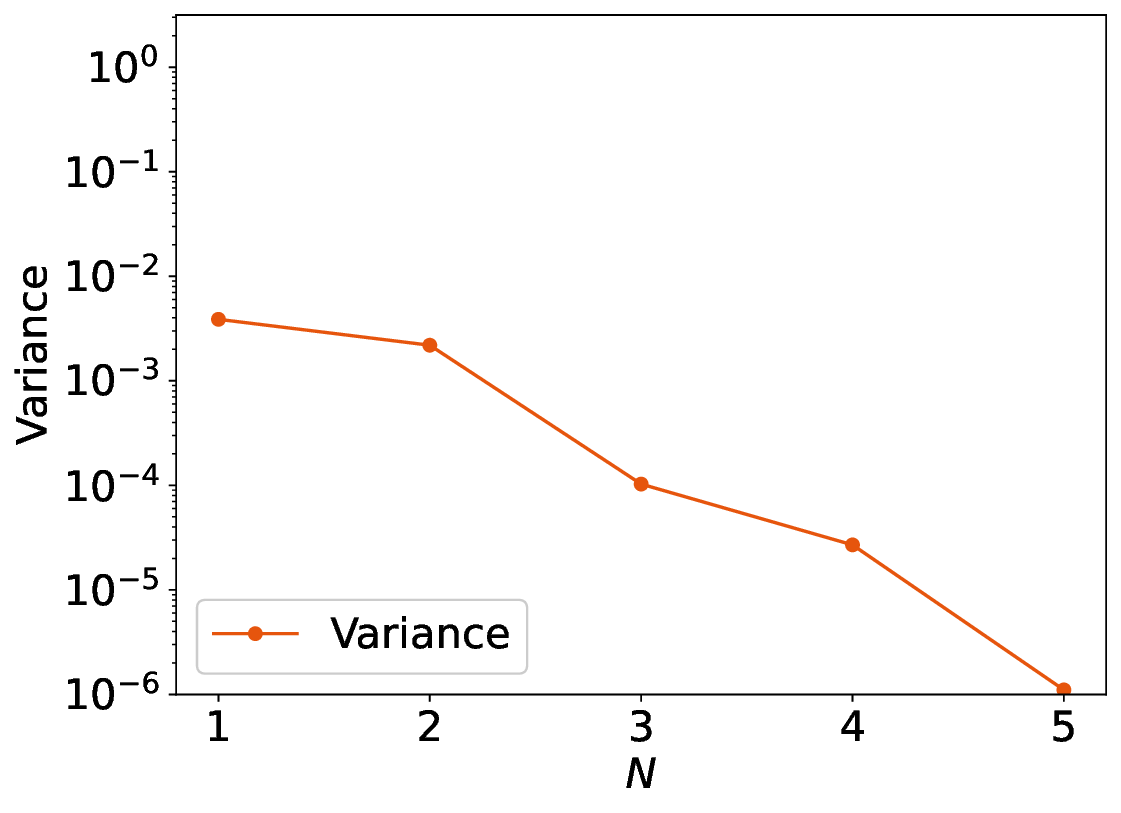}
  \caption{Average variance of the emulator}
  \label{fig:exp0_basic_var}
\end{subfigure}
\caption{ (\ref{fig:exp0_basic_mean}) Baseline model mean-based posterior with different $N$. (\ref{fig:exp0_basic_infl}) Baseline model marginal posterior with different $N$. (\ref{fig:exp0_basic_hell}) Hellinger distance between approximated posteriors and true posteriors when N increases. \eqref{fig:exp0_basic_var} Average predictive variance of the Gaussian process emulator as $N$ increases. $\mathcal{G}_X$ is the discretised solution $u$ in \eqref{eq:example0}.}
\label{fig:basic0}
\end{figure}

\paragraph{Spatially correlated model:} As discussed in Section \ref{subsec:cor_PDE} the introduction of spatial correlation doesn't change the predictive mean of the Gaussian  processes. We hence now compare in Figure \ref{fig:spc_gp} the two different marginal posteriors $\pi^{N,\mathcal{G}_X}_{\mathrm{marginal}}$, $\pi^{N,\mathcal{G}_X,s}_{\mathrm{marginal}}$, where the latter includes spatial correlation. We again choose  $\dz=5$. In particular, as we can see in Figure \ref{fig:spc_gp}(a) (here $N=2$),  introducing spatial correlation seems to improve the accuracy of the approximate posterior and place more mass near $\theta^\dagger$. The fact that the spatially correlated model has an increased variance at $N=2$ (see Figure \ref{fig:exp0_spc_var}) leads to similar behavior as in Figure \ref{fig:basic0} with $\pi^{N,\mathcal{G}_X,s}_{\mathrm{marginal}}$ being more spread than $\pi^{N,\mathcal{G}_X}_{\mathrm{marginal}}$. Furthermore, as we can see in Figure \ref{fig:spc_gp}(b) as we increase the number of training points for our Gaussian process the Hellinger distance between the true posterior and $\pi^{N,\mathcal{G}_X,s}_{\mathrm{marginal}}$ is smaller than the one of the baseline model. 

\paragraph{PDE-constrained model:} We now compare the behaviour of the PDE-constrained model with the other two models, both for mean-based approximate posterior, as well as for marginal posterior (again here $\dz=5$). In particular, as we can see in Figures \ref{fig:pde_model0}(a) and  \ref{fig:pde_model0}(b) for $N=2$, $\pi^{N,\mathcal{G}_X,\mathrm{PDE}}_{\mathrm{mean}}$ and $\pi^{N,\mathcal{G}_X,\mathrm{PDE}}_{\mathrm{marginal}}$ are indistinguishable from the true posterior when using $\Bar{N}=10$, $d_{f}=5$ showing much better approximation properties than the other two models. This is consistent with what we observe in terms of Hellinger distance, since both $\pi^{N,\mathcal{G}_X,\mathrm{PDE}}_{\mathrm{mean}}$ and $\pi^{N,\mathcal{G}_X,\mathrm{PDE}}_{\mathrm{marginal}}$ have similar errors over a different range of values for $N_{f}$. It is also worth noting that  when comparing with the Hellinger distance from Figures Figures \ref{fig:basic0}(c) and \ref{fig:spc_gp}(c) we see that the PDE-based model achieves the same order of error with only using half of the training points ($N=2$ instead of $N=4$).   Furthermore, as we can see in Figure \ref{fig:pde_model0}(d) the average variance of the PDE-constrained emulator converges to zero very fast as the number of extra training points for $f$ increases, implying that at least in this simple example adding the PDE knowledge leads to an extremely good approximation of the forward map.

\begin{figure}
\centering
\begin{subfigure}{.5\textwidth}
  \centering
  \includegraphics[width=\linewidth]{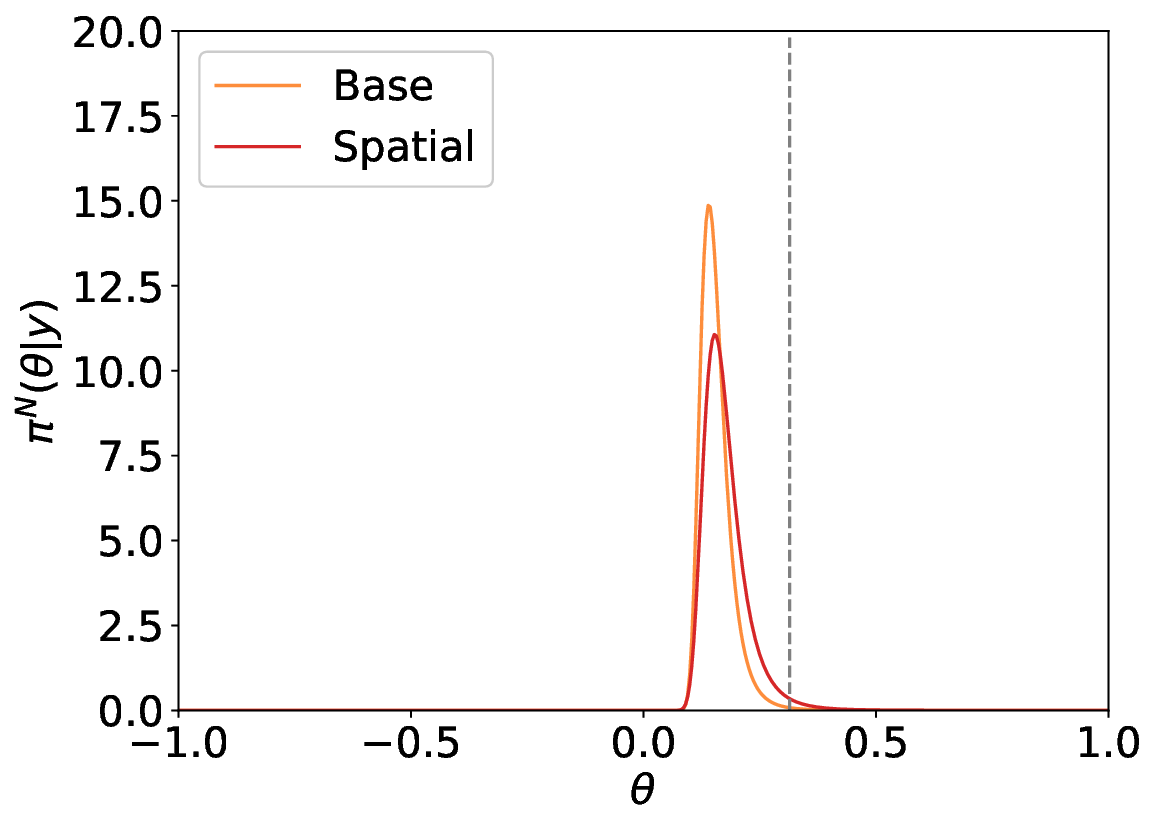}
  \caption{$\pi^{N,\mathcal{G}_{X}}_{\mathrm{marginal}}(\theta|\mathbf{y})$, $\pi^{N,\mathcal{G}_{X},s}_{\mathrm{marginal}}(\theta|\mathbf{y})$ ($N = 2$)}
  \label{fig:exp0_spc_infl}
\end{subfigure}%
\\
\begin{subfigure}{0.45\textwidth}
  \centering
  \includegraphics[width=\linewidth]{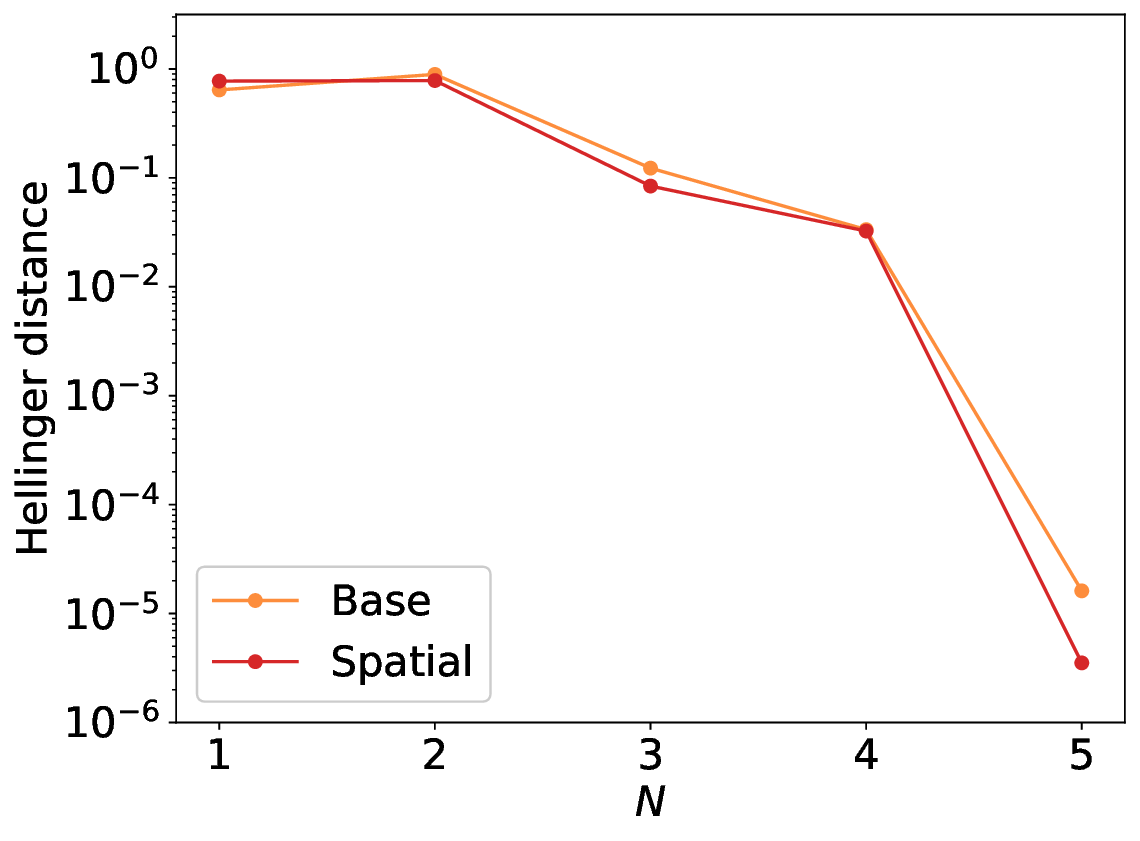}
  \caption{Hellinger distance}
  \label{fig:exp0_spc_hell}
\end{subfigure}
\begin{subfigure}{.45\textwidth}
  \centering
  \includegraphics[width=\linewidth]{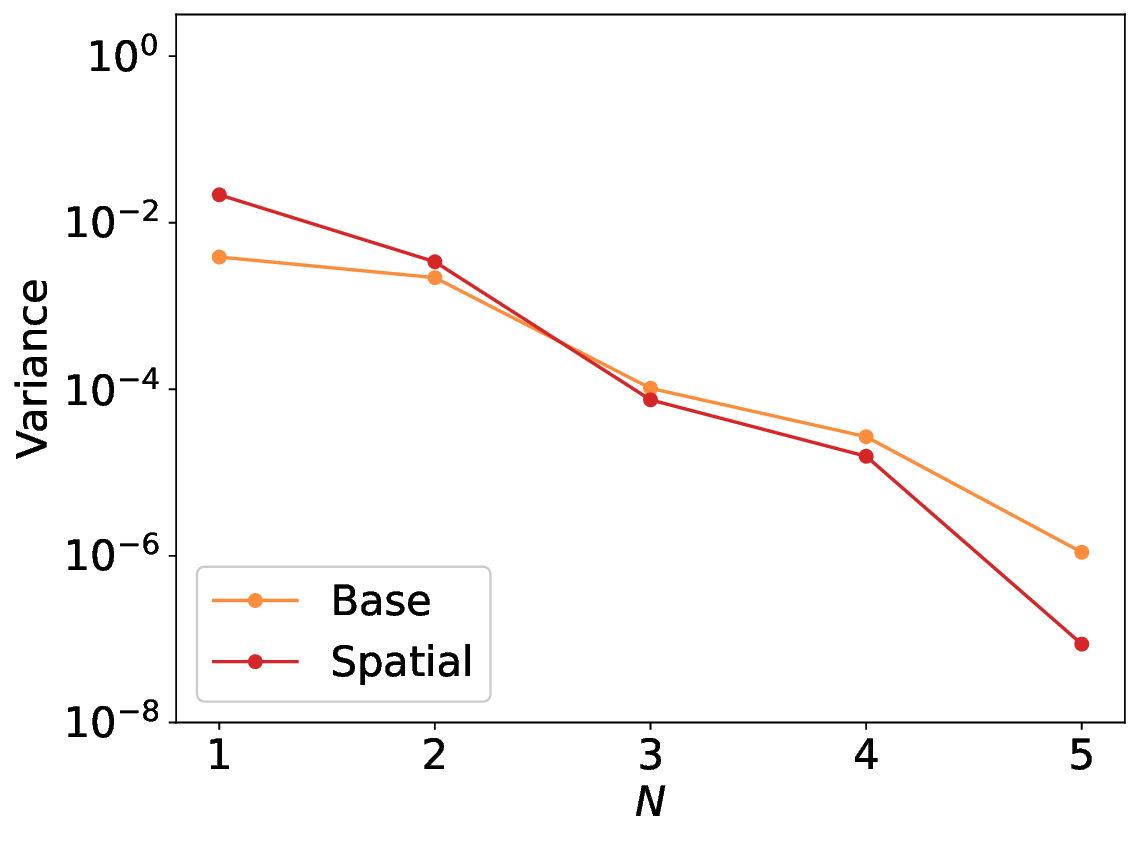}
  \caption{Average variance of the emulator}
  \label{fig:exp0_spc_var}
\end{subfigure}
\caption{  (\ref{fig:exp0_spc_infl}) Baseline and spatially correlated model marginal posterior for $N=2$. (\ref{fig:exp0_spc_hell}) Hellinger distance between approximated posteriors and true posterior as $N$ increases. \eqref{fig:exp0_spc_var} Average predictive variance of the Gaussian process emulator as $N$ increases. $\mathcal{G}_X$ is the discretised solution $u$ in \eqref{eq:example0}.}
\label{fig:spc_gp}
\end{figure}


\begin{figure}
\centering
\begin{subfigure}{.45\textwidth}
  \centering
  \includegraphics[width=\linewidth]{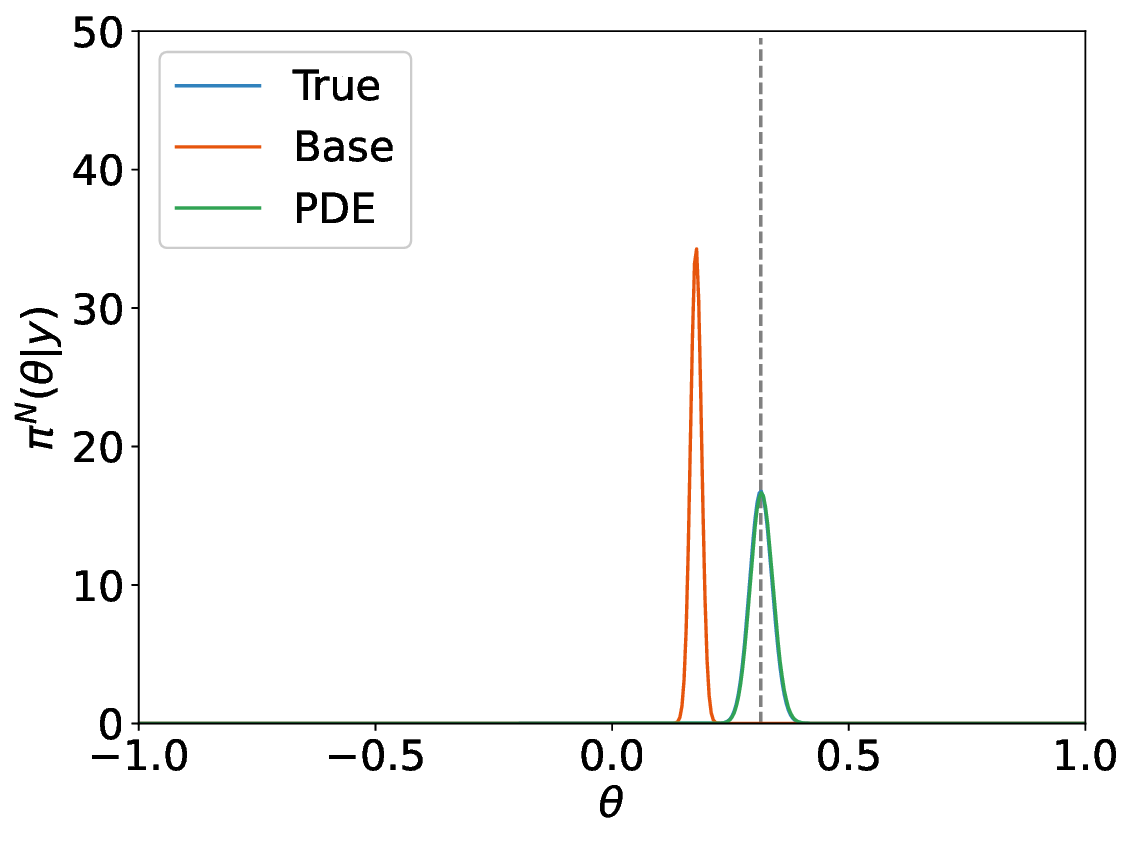}
  \caption{$\pi^{N,\mathcal{G}_{X}}_{\mathrm{mean}}(\theta|\mathbf{y})$}
  \label{fig:exp0_pde_mean}
\end{subfigure}%
\begin{subfigure}{.45\textwidth}
  \centering
  \includegraphics[width=\linewidth]{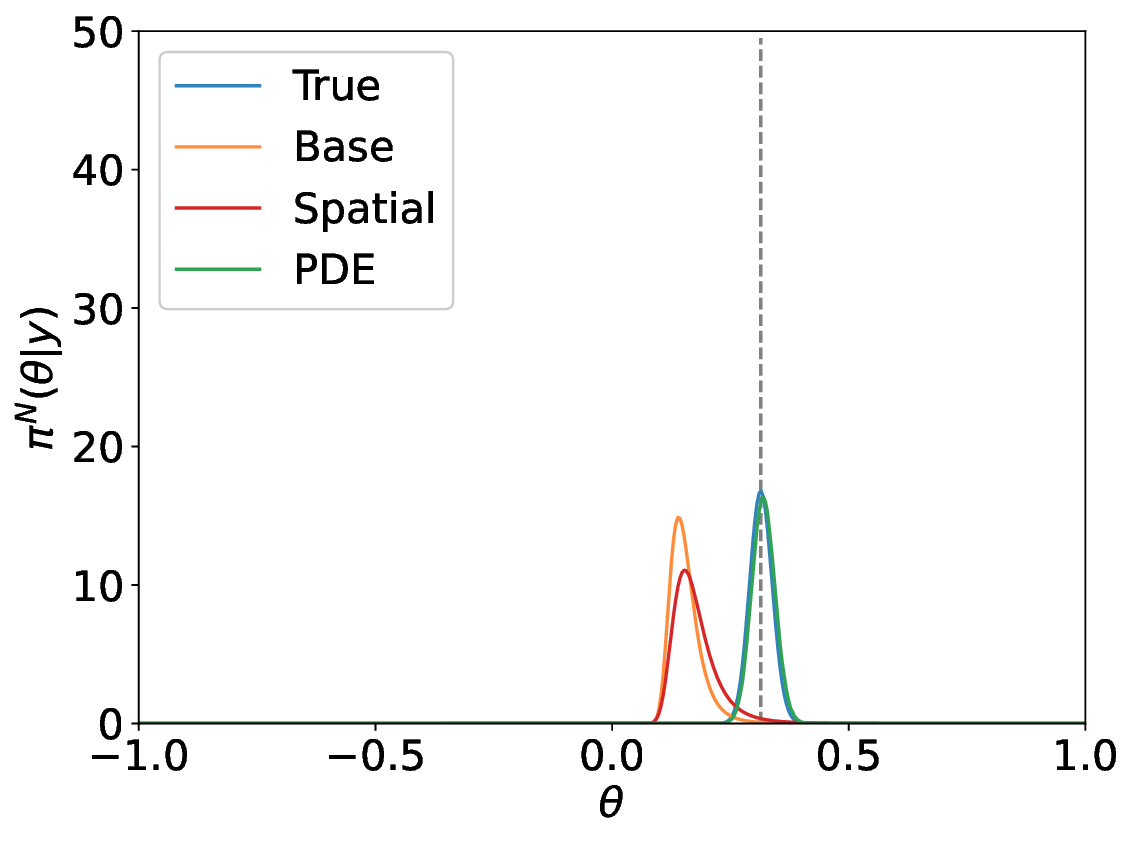}
  \caption{$\pi^{N,\mathcal{G}_{X}}_{\mathrm{marginal}}(\theta|\mathbf{y})$}
  \label{fig:exp0_pde_infl}
\end{subfigure}
\begin{subfigure}{.45\textwidth}
  \centering
  \includegraphics[width=\linewidth]{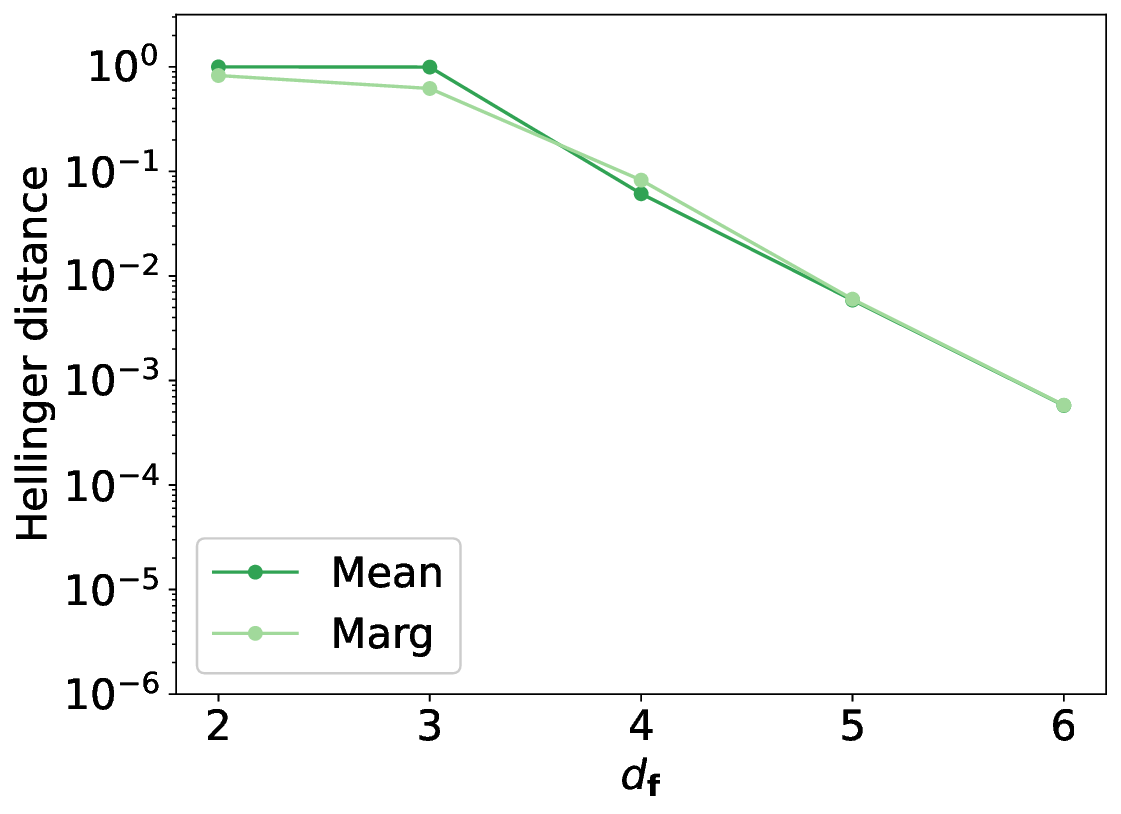}
  \caption{Hellinger distance}
  \label{fig:exp0_pde_hell}
\end{subfigure}
\begin{subfigure}{.45\textwidth}
  \centering
  \includegraphics[width=\linewidth]{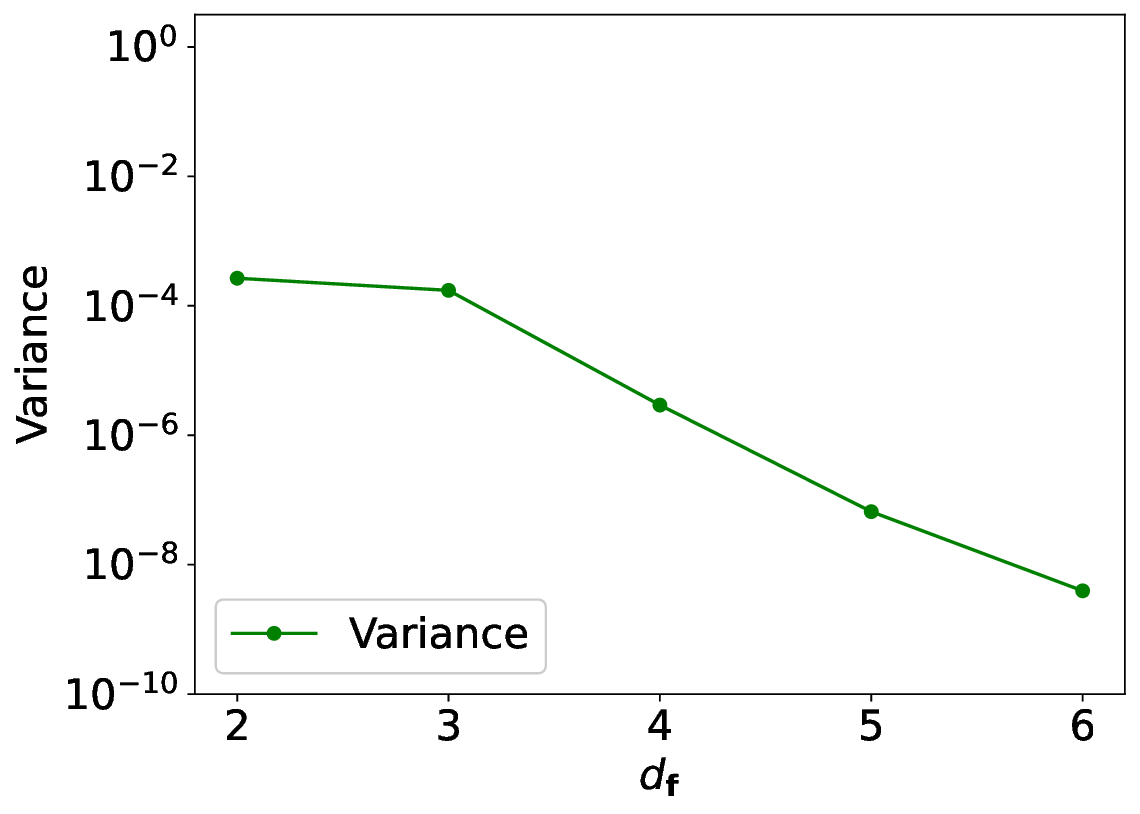}
  \caption{Average variance of the emulator}
  \label{fig:exp0_pde_var}
\end{subfigure}
\caption{Comparison of different models when $N=2$, for PDE model $d_f = 5$. (\ref{fig:exp0_pde_mean}) Mean-based posteriors (\ref{fig:exp0_pde_infl}) Marginal posteriors (\ref{fig:exp0_pde_hell}) Hellinger distance between approximated posteriors and true posterior as $d_f$ increases. (\ref{fig:exp0_pde_var}) Average predictive variance of emulator as $d_f$ increases. $\mathcal{G}_X$ is the discretised solution $u$ in \eqref{eq:example0}.}
\label{fig:pde_model0}
\end{figure}

\subsubsection{Two dimensional piece-wise constant diffusion coefficient}
\label{sec:exp1}
We now consider a slightly more general problem than \eqref{eq:example0}. In particular,   we consider the same elliptic equation as  in \eqref{eq:example0} but use a 2-dimensional piece-wise constant diffusion coefficient. In particular, we now have the following equation
\begin{align} \label{eq:example1}
    -\frac{\mathrm{d}}{\mathrm{d}x}(\exp(\kappa(x, \btheta))\frac{\mathrm{d}}{\mathrm{d}x}  {u}(x)) &= 4x, \qquad x \in (0,1), \,\btheta \in [-1,1], \\
    u(0) &=0, \nonumber \\
    u(1) &=2, \nonumber
\end{align}
where $\kappa$ is defined as piece-wise constant over four equally spaced intervals. More precisely, we consider
\begin{equation}\label{eq:pc}
 \kappa(x,\btheta)=
 \left\{\begin{array}{lr}
        0, & \text{for } x \in [0,\frac{1}{4})\\
        \\
        \theta_{1}, & \text{for } x \in [\frac{1}{4},\frac{1}{2})\\
        \\
       \theta_{2}, & \text{for } x \in [\frac{1}{2},\frac{3}{4})\\
       \\
       1  & \text{for } x \in [\frac{3}{4},1]
        \end{array}\right.
\end{equation}
Unlike equation \eqref{eq:example0}, it is not possible to obtain an analytic 
solution for \eqref{eq:example1} so we use instead Firedrake to obtain its 
solution.

For the PDE constrained model, we first test the effectiveness of additional training data $g(\Theta_g, X_g)$ and $f(\Theta_f, X_f)$. We let the size of point set $\Theta_g$ and $\Theta_g$ to be the same and denoted by $\Bar{N}$. In Figure \ref{fig:exp1_dfnf}, we test the impact of $d_{f}$ and $\Bar{N}$ to the accuracy of the PDE constrained emulator. We use fixed same hyperparameters for all the models and $d_g = 2$. We see that as $d_{f}$ increases the accuracy of emulators gradually increases. While for $\Bar{N}$, we see that a certain amount of additional point can improve the accuracy, but including more points cannot make further improvement. 

\begin{figure}
\centering
\begin{subfigure}{1.1\textwidth}
  \centering
  \includegraphics[width=\linewidth]{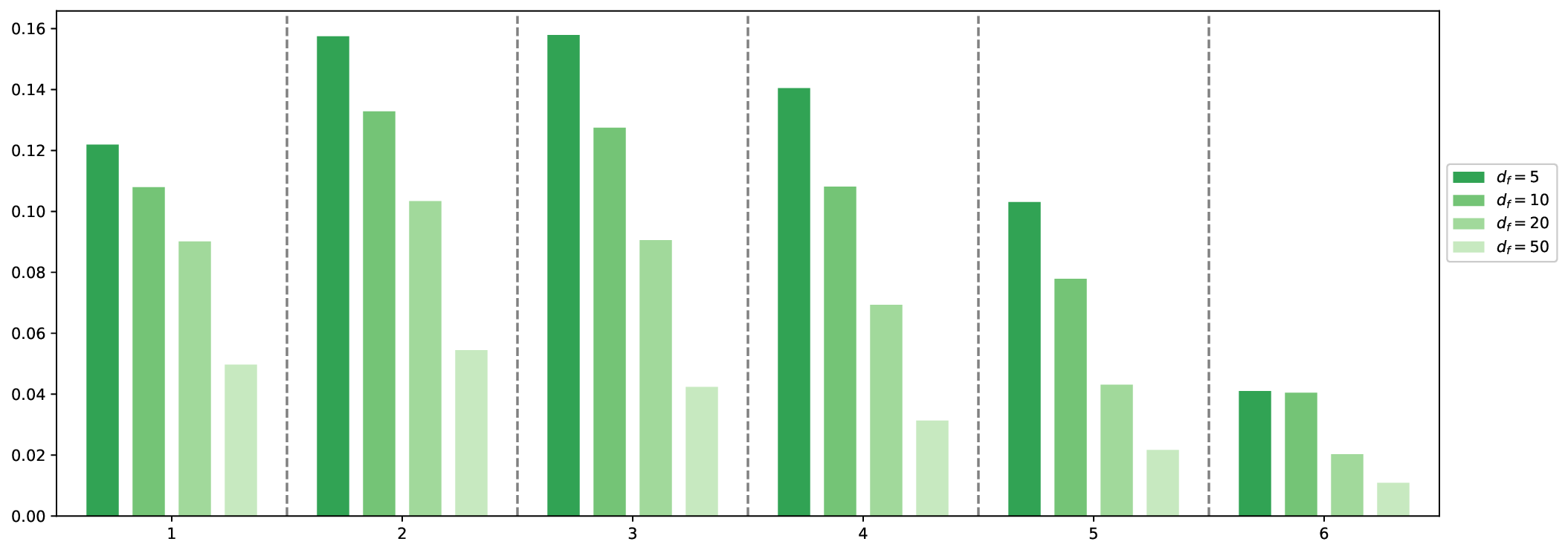}
  \caption{}
  \label{fig:exp1_df}
\end{subfigure}
\begin{subfigure}{1.1\textwidth}
  \centering
  \includegraphics[width=\linewidth]{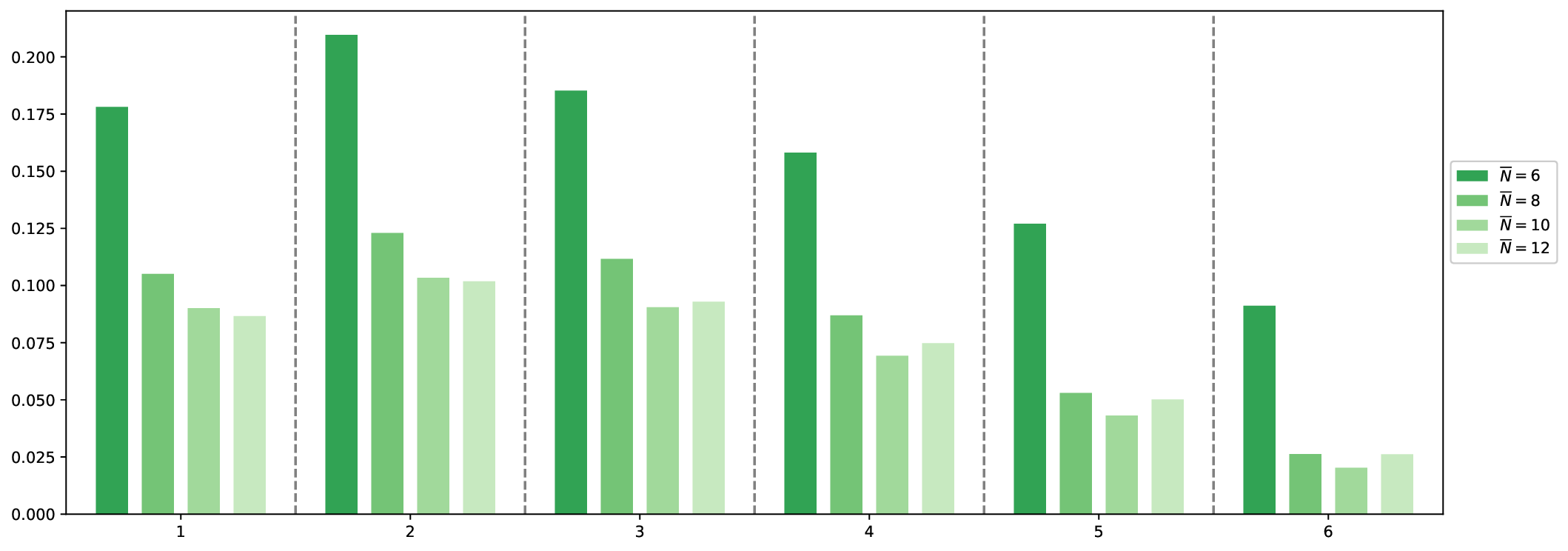}
  \caption{}
  \label{fig:exp1_nf}
\end{subfigure}
\caption{Error between the predictive mean of PDE constrained emulators and the ground truth at observation points ($\btheta = \btheta^{\dagger}$) for different (a) $d_{f}$ ($\Bar{N} = 10$) (b) $\Bar{N}$ ($d_{f} = 20$)}
\label{fig:exp1_dfnf}
\end{figure}

Throughout this numerical experiment, we take the prior of the parameters to be the uniform distribution on $[-1,1]^{2}$, 
 and we generate our data $\bz$ according to equation \eqref{equ:ivp} for the value $\btheta^{\dagger}=[0.098, 0.430]$ for $\dz=6$ (equally spaced points in [0,1]) and for noise level $\sigma^{2}_{\eta}=10^{-4}$. For the baseline and spatially correlated model, we have used $N=4$ training points (chosen to be the first $4$ points in the Halton sequence), while additionally for the PDE-constrained model, we have used $\Bar{N} = 10$ (chosen to be the following $10$ points in the Halton sequence) and $d_{f} = 20$. For the covariance kernels, we choose $k_p$ to be the squared exponential kernel and $k_s$ to be the Mat\`ern kernel with $\nu = \frac{5}{2}$.

Unlike \eqref{eq:example0} we now do not perform exact integration but use the MALA algorithm to obtain our samples. In particular, for all our approximate  posteriors we have used $10^{6}$ samples. In addition, since in this case, we do not have an analytic expression for the solution, we do not have direct access to the true posterior. We circumvent this problem by considering the results obtained by a mean-based approximation with the baseline model for $N=10^{2}$ training points as the ground truth.

\begin{figure}
\centering

\begin{subfigure}{.31\textwidth}
  \centering
  \includegraphics[width=\linewidth]{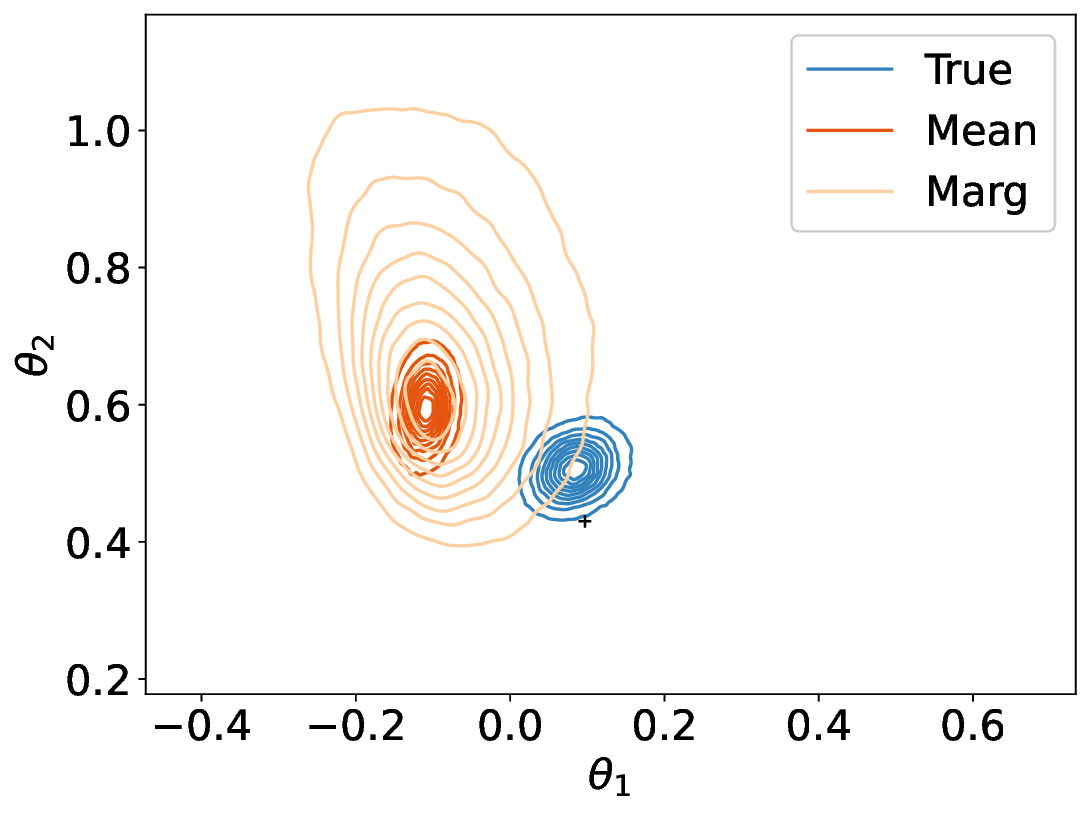}
  \caption{Baseline}
  \end{subfigure}
\begin{subfigure}{.31\textwidth}
  \centering
  \includegraphics[width=\linewidth]{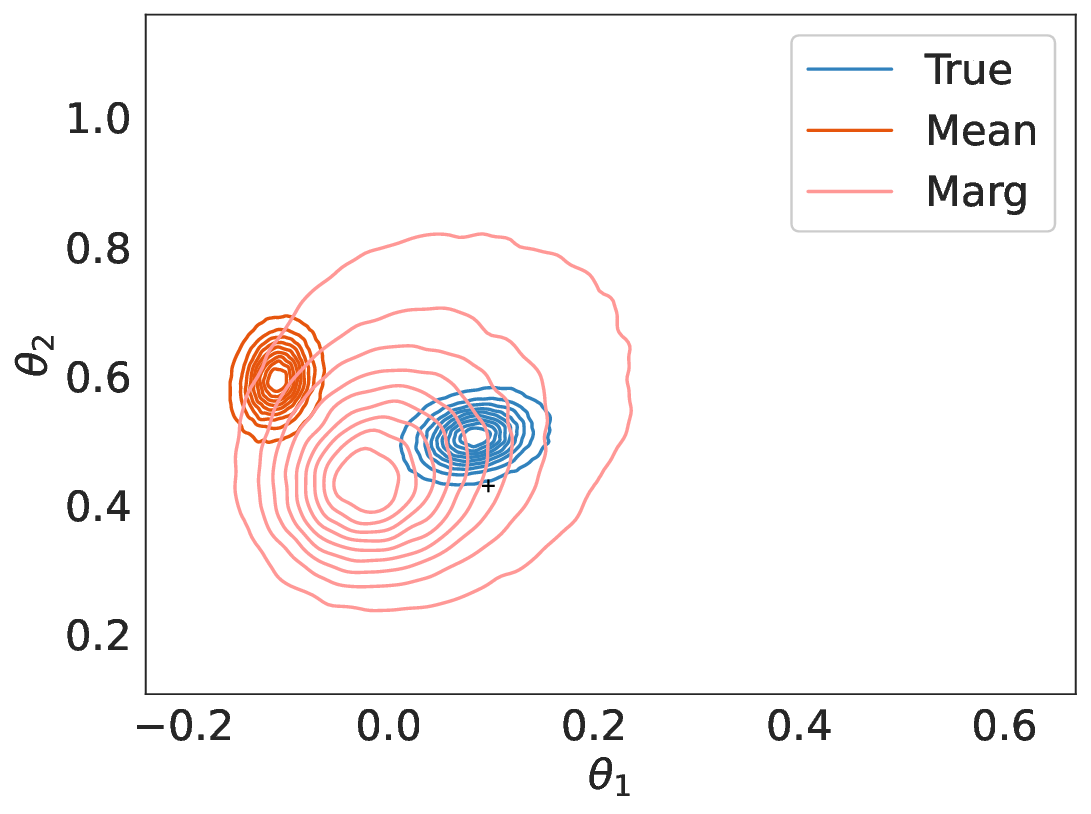}
  \caption{Spatially correlated}
\end{subfigure}
\begin{subfigure}{.31\textwidth}
  \centering
  \includegraphics[width=\linewidth]{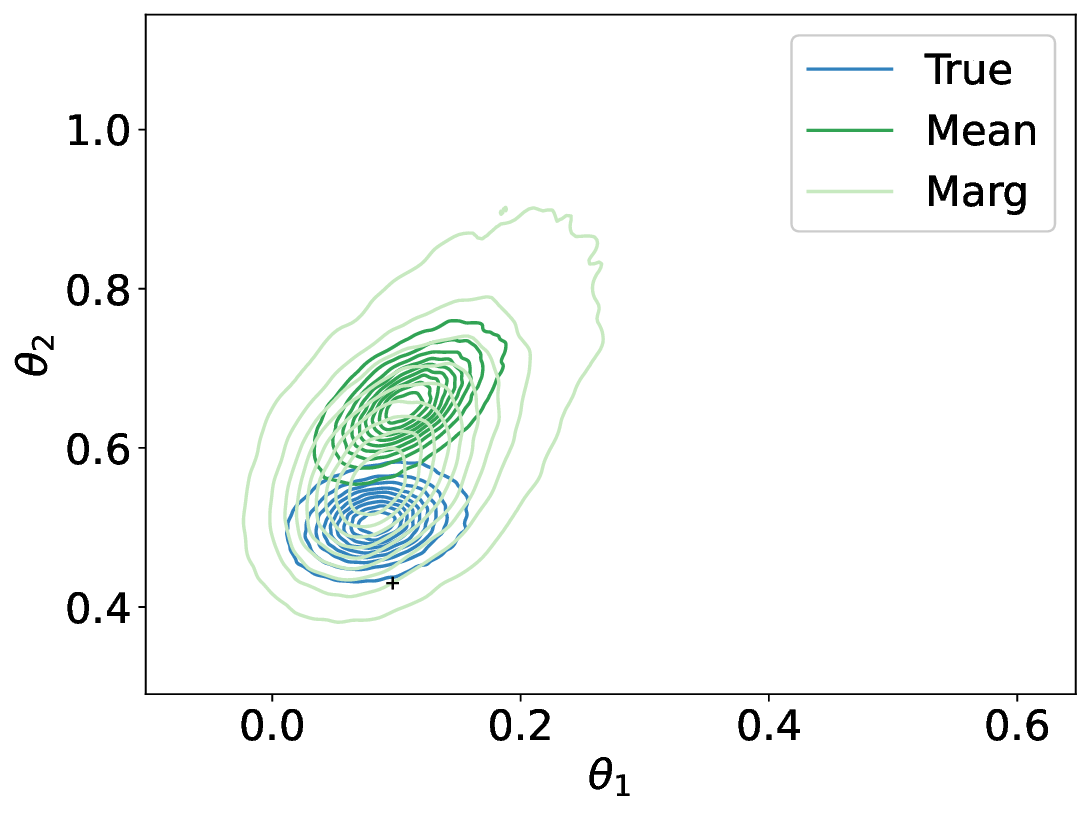}
  \caption{PDE-constrained}
\end{subfigure}
\caption{Contour plots of the approximate mean-based and inflated-based posteriors.  (a) baseline model (b) spatially correlated (c) PDE-constrained. $"+"$ denotes $\btheta^{\dagger}$. $\mathcal{G}_X$ is the discretised solution $u$ in \eqref{eq:example1}. }
\label{fig:model1_all}
\end{figure}

As we can see in Figures \ref{fig:model1_all}(a)-(c) all the mean-based posteriors are failing to put significant posterior mass near the true parameter value $\btheta^{\dagger}$. The situation improves when the uncertainty of the emulator is taken into account as we can see for the marginal-based posteriors. Out of the three different models, the PDE-constrained one seems to be performing best since it is placing the most posterior mass around the true value $\btheta^{\dagger}$. This is further illustrated in Figure \ref{fig:exp1_marg} where we plot the $\theta_{1}$ and $\theta_{2}$ marginals for all the mean-based posterior approximations $\pi^{N,\mathcal{G}_X}_{\mathrm{mean}}$, $\pi^{N,\mathcal{G}_X,s}_{\mathrm{mean}}$, $\pi^{N,\mathcal{G}_X,\mathrm{PDE}}_{\mathrm{mean}}$ and the marginal-based posterior approximations   $\pi^{N,\mathcal{G}_X}_{\mathrm{marginal}}$, $\pi^{N,\mathcal{G}_X,s}_{\mathrm{marginal}}$, $\pi^{N,\mathcal{G}_X,\mathrm{PDE}}_{\mathrm{marginal}}$.  Note that the marginal plot could be misleading the overall performance of the approximations, for example in Figure \ref{fig:exp1_mean2} the baseline model seems to be better than the PDE-constrained model, but from the Figure \ref{fig:model1_all} we know that is not true. When we increase $d_{f}$ from $20$ to $50$, the accuracy of approximation improves. We see that in Figure \ref{fig:exp1_marg2}, the marginal plot of the mean-based approximate  

\begin{figure}
\centering
\begin{subfigure}{.45\textwidth}
  \centering
  \includegraphics[width=\linewidth]{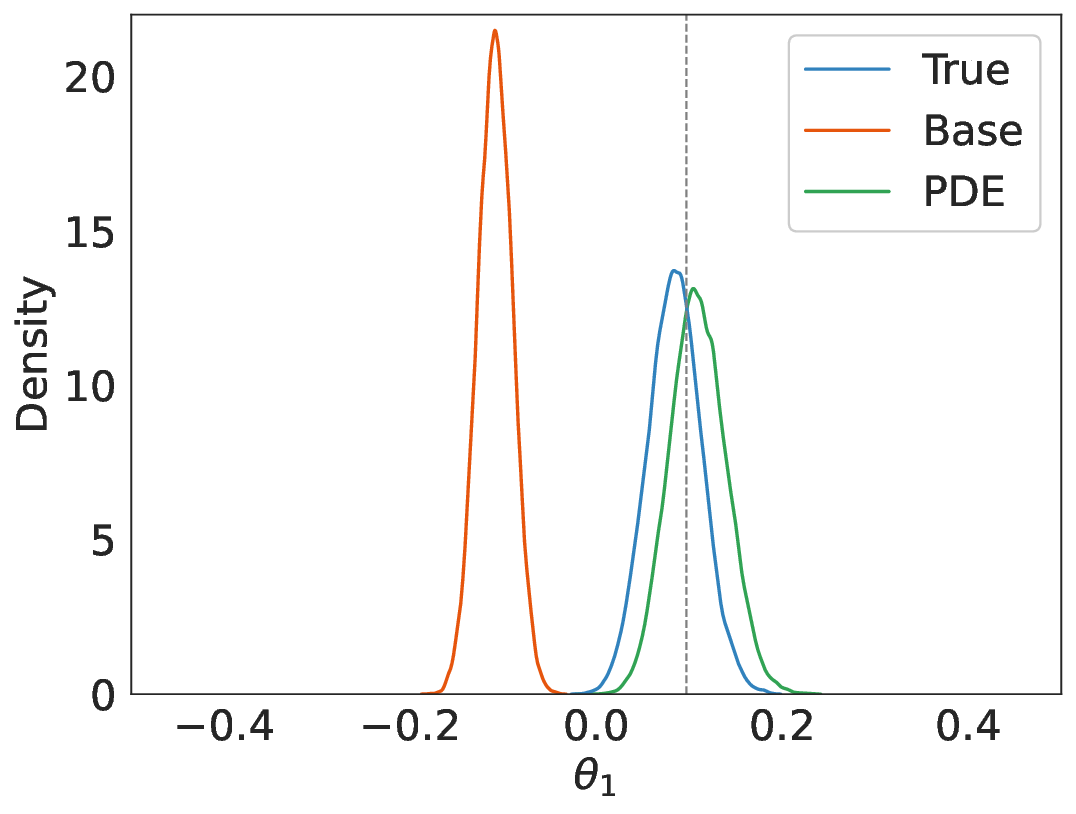}
  \caption{}
  \label{fig:exp1_mean1}
\end{subfigure}%
\begin{subfigure}{.45\textwidth}
  \centering
  \includegraphics[width=\linewidth]{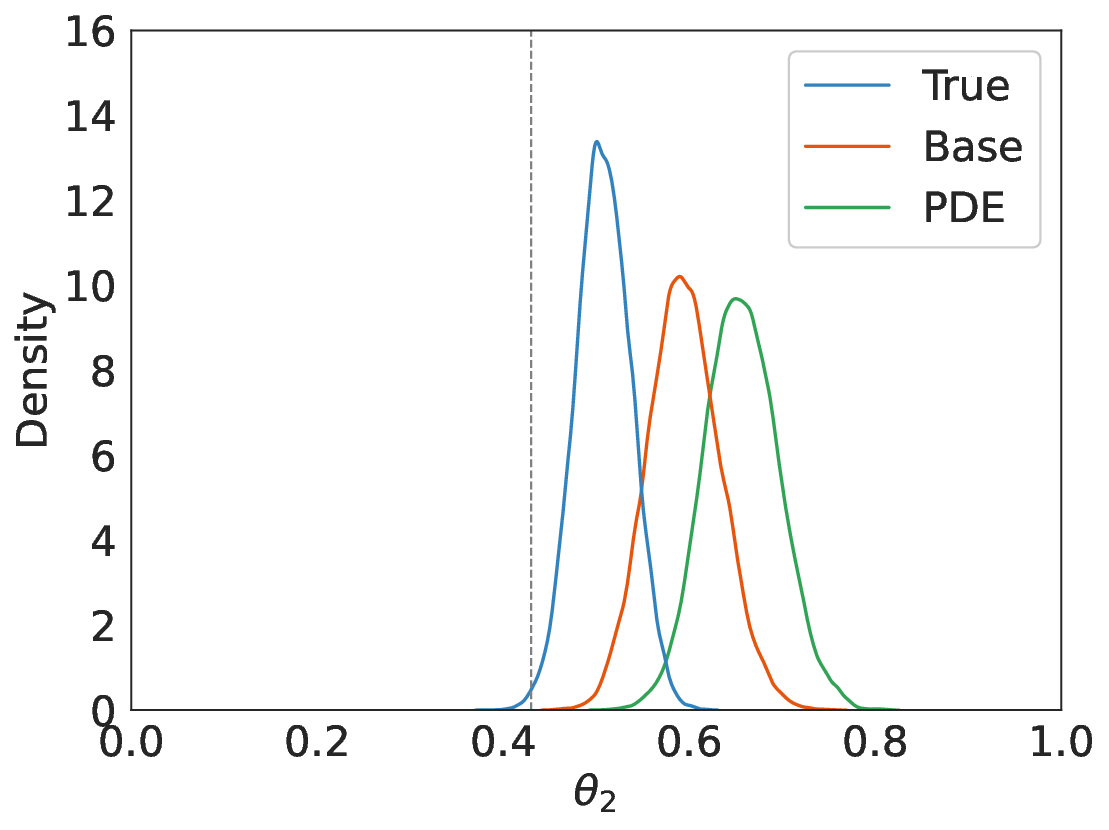}
  \caption{}
  \label{fig:exp1_mean2}
\end{subfigure}
\begin{subfigure}{.45\textwidth}
  \centering
  \includegraphics[width=\linewidth]{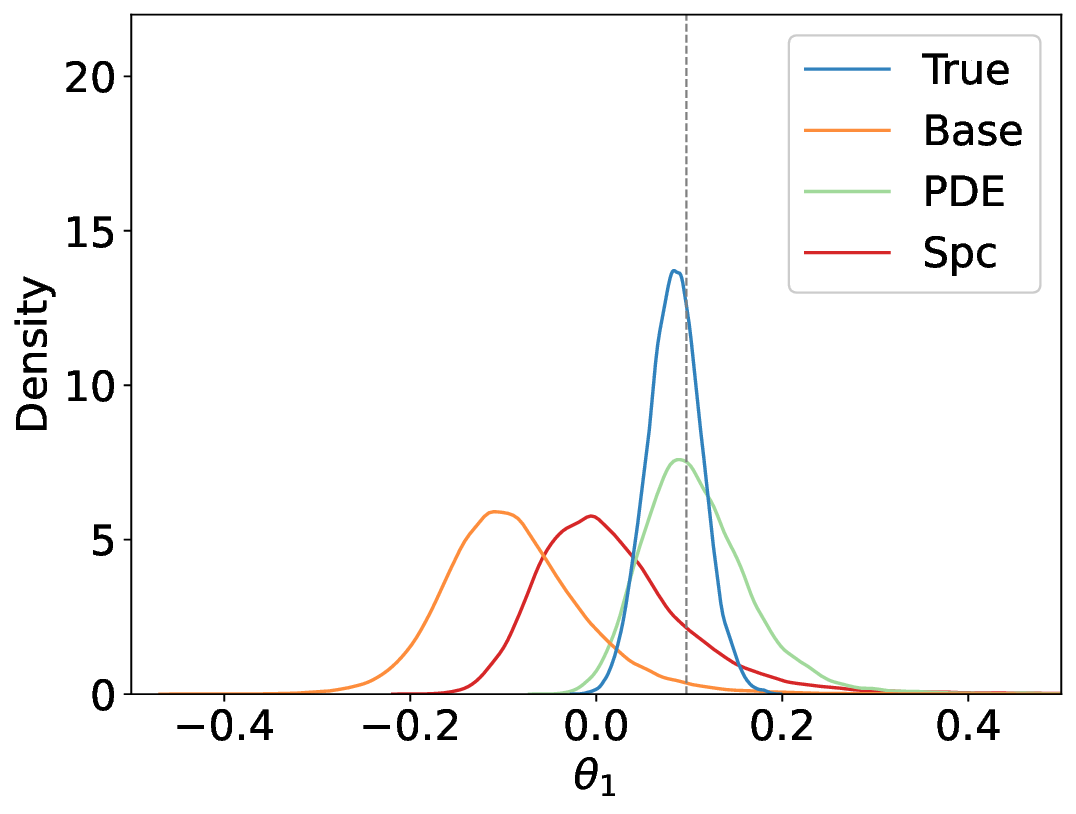}
  \caption{}
  \label{fig:exp1_infl1}
\end{subfigure}
\begin{subfigure}{.45\textwidth}
  \centering
  \includegraphics[width=\linewidth]{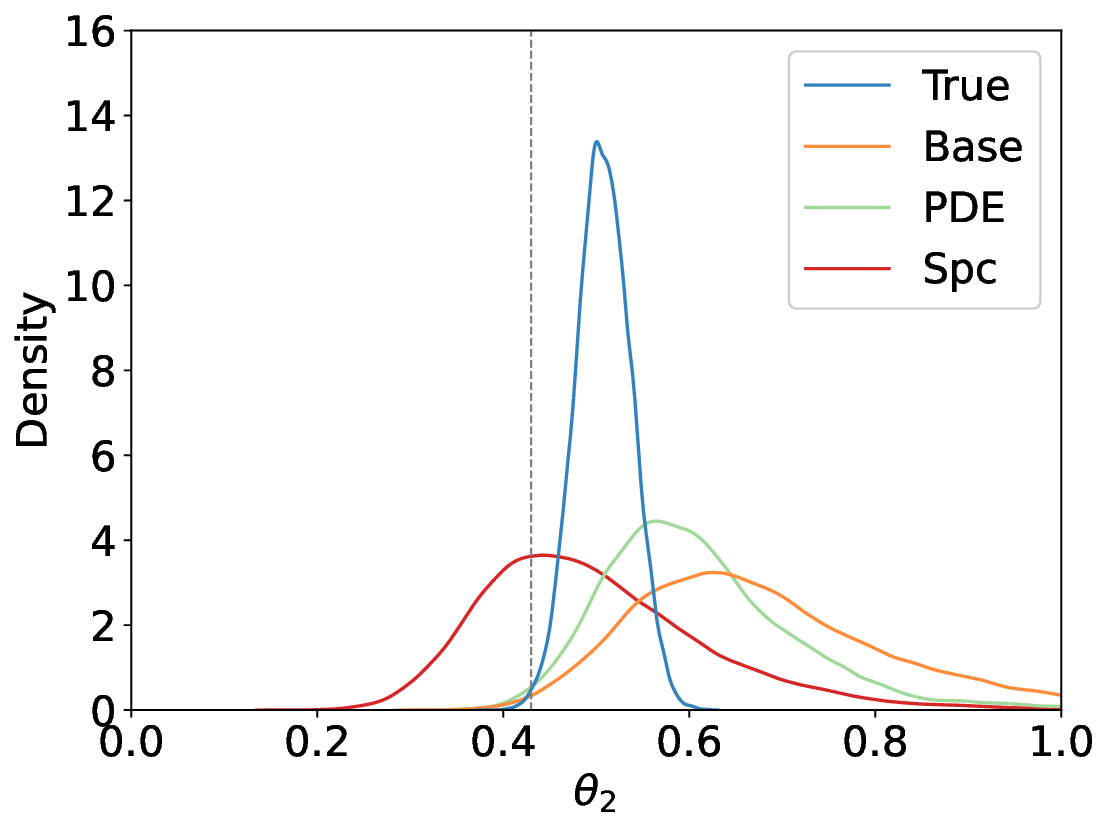}
  \caption{}
  \label{fig:exp1_infl2}
\end{subfigure}
\caption{Comparison of different models' marginal distribution when $N=4$, for PDE model $d_f = 20$ and $\Bar{N} = 20$.  (a) Mean-based approximation $\theta_1$. (b) Mean-based approximation $\theta_2$. (c) Marginal approximation $\theta_1$. (d) Marginal approximation $\theta_2$. $\mathcal{G}_X$ is the discretised solution $u$ in \eqref{eq:example1} with diffusion coefficient \eqref{eq:pc}. }
\label{fig:exp1_marg}
\end{figure}

\begin{figure}
\centering
\begin{subfigure}{.45\textwidth}
  \centering
  \includegraphics[width=\linewidth]{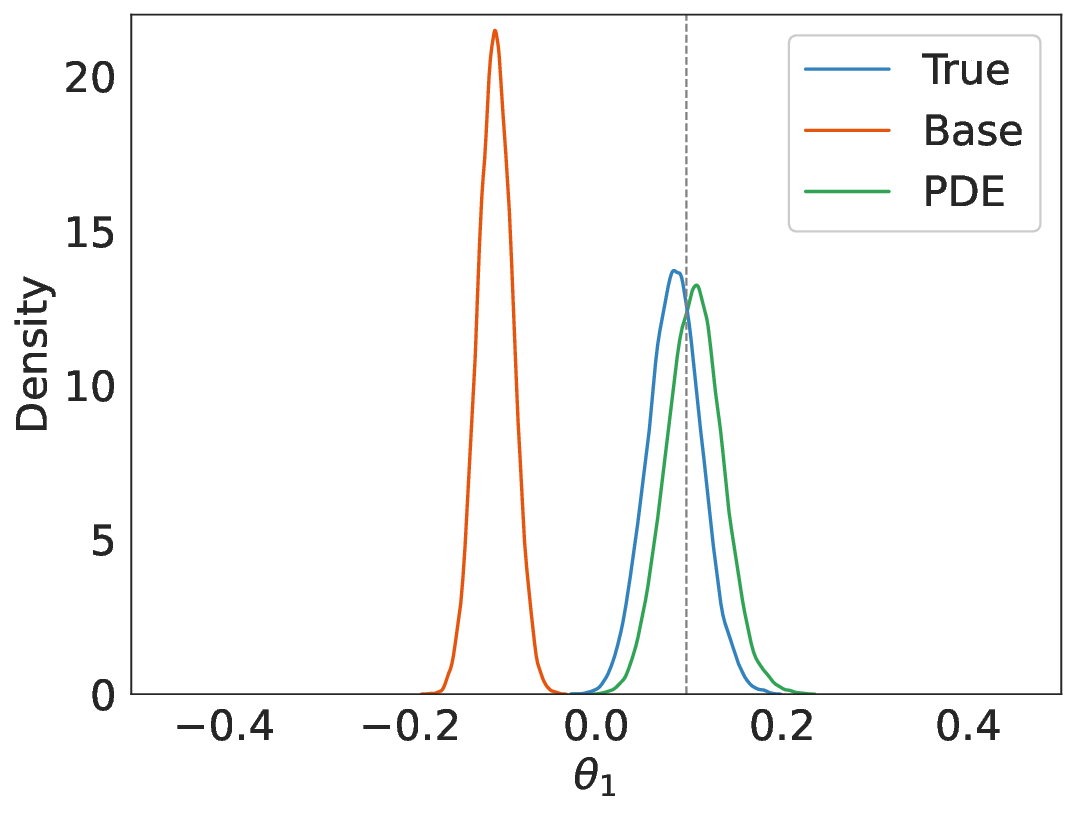}
  \caption{}
  \label{fig:exp1_mean1a}
\end{subfigure}%
\begin{subfigure}{.45\textwidth}
  \centering
  \includegraphics[width=\linewidth]{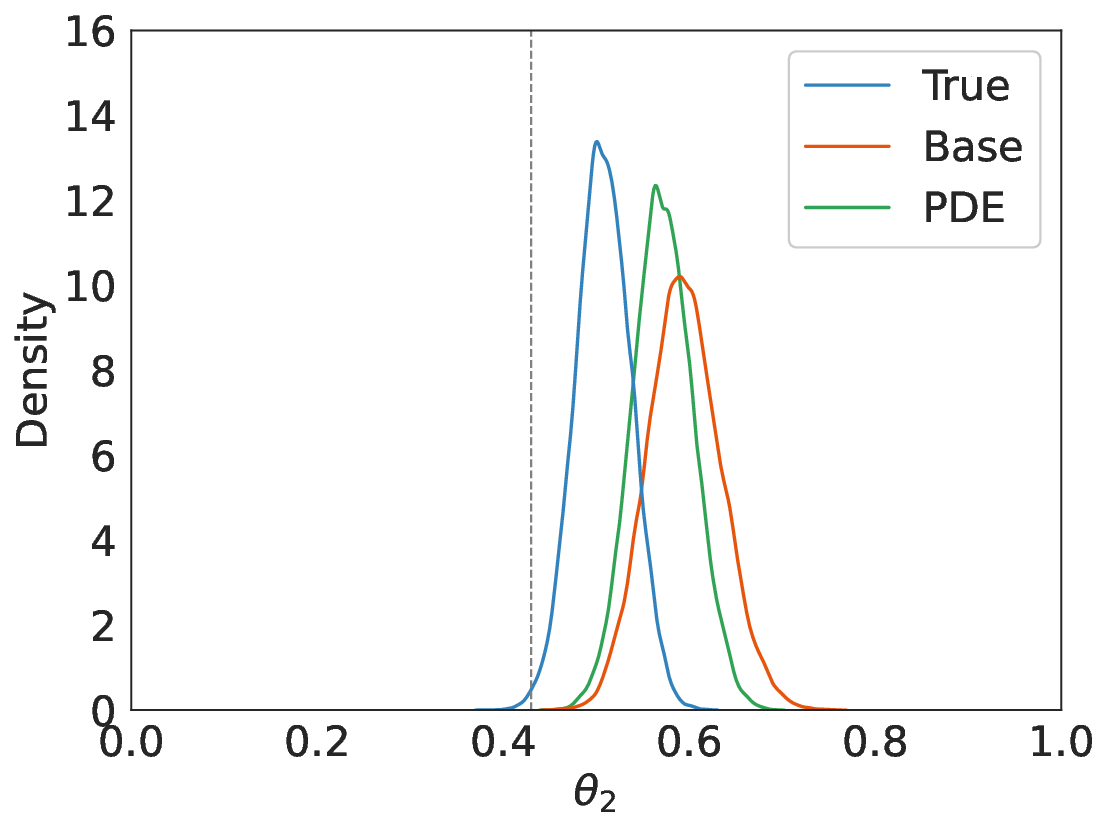}
  \caption{}
  \label{fig:exp1_mean2a}
\end{subfigure}

\caption{Comparison of different models' marginal distribution when $N=4$, for PDE model $d_f = 50$ and $\Bar{N} = 20$.  (a) Mean-based approximation $\theta_1$. (b) Mean-based approximation $\theta_2$. (c) Marginal approximation $\theta_1$. (d) Marginal approximation $\theta_2$. $\mathcal{G}_X$ is the discretised solution $u$ in \eqref{eq:example1} with diffusion coefficient \eqref{eq:pc}. }
\label{fig:exp1_marg2}
\end{figure}

\subsubsection{Integral observation operator}
We now  investigate the proposed method with a different form of observation operator. In terms of the PDE problem, we study again \eqref{eq:example1}.
However, instead of  point-wise observations $\mathcal{G}_{X}(\boldsymbol{\theta}) = \{u(x_j; \boldsymbol{\theta})\}_{j=1}^{{{d}_{\mathbf{y}}}}$ as in \eqref{equ:ivp}, we  observe local averages $\mathcal{G}_{X}(\boldsymbol{\theta}) = \{\int_{a_j}^{b_j}u(x; \boldsymbol{\theta}) dx\}_{j=1}^{{{d}_{\mathbf{y}}}}$ for non-overlapping intervals $[a_j,b_j \subset [0,1]$.

For the inverse problem setting, we have $\btheta^{\dagger} = [0.098, 0.430]$ which is the same as before, $\dz = 16$ (equally spaced sub-intervals of $[0,1]$) and $\sigma_\eta^2 = 10^{-6}$. We again do not conduct precise integration as in \eqref{eq:example0}, but use MALA algorithm to obtain our samples. We utilize $10^{6}$ samples for all our approximate posterior. We treat the sampling results obtained by a mean-based approximation with the baseline model for $N=10^{2}$ training points as the ground truth. In Figure \ref{fig:exp3_marg}, we plot again the $\theta_{1}$ and $\theta_{2}$ marginals for all the mean-based posterior approximations and the marginal posterior approximations. The result is similar to the previous example that the PDE-constrained model performs better than the other two models.

\begin{figure}
\centering
\begin{subfigure}{.45\textwidth}
  \centering
  \includegraphics[width=\linewidth]{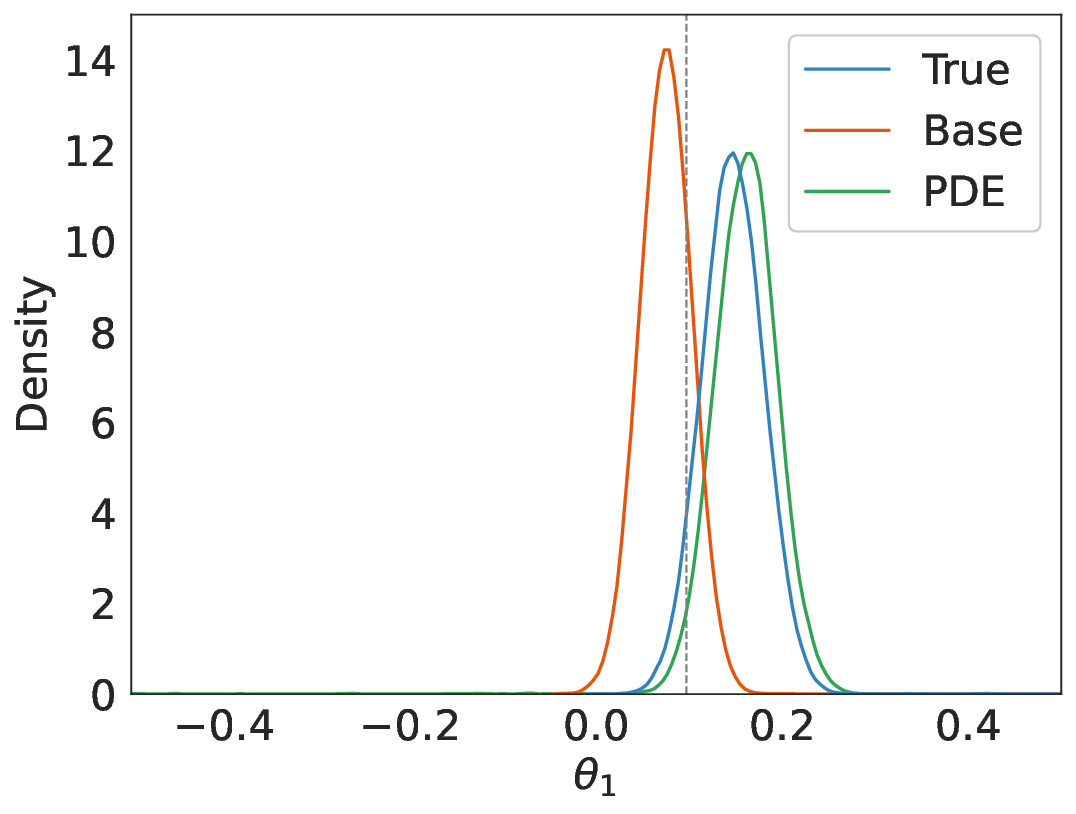}
  \caption{}
  \label{fig:exp3_mean1}
\end{subfigure}%
\begin{subfigure}{.45\textwidth}
  \centering
  \includegraphics[width=\linewidth]{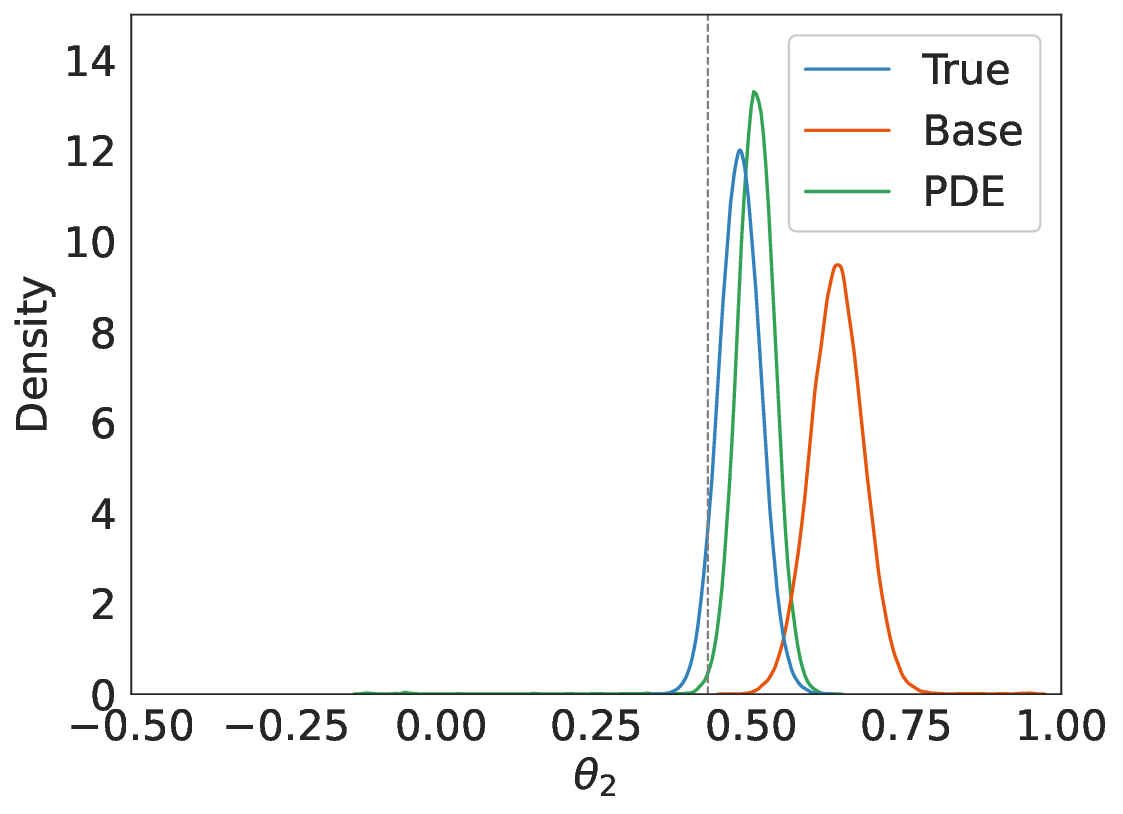}
  \caption{}
  \label{fig:exp3_mean2}
\end{subfigure}
\begin{subfigure}{.45\textwidth}
  \centering
  \includegraphics[width=\linewidth]{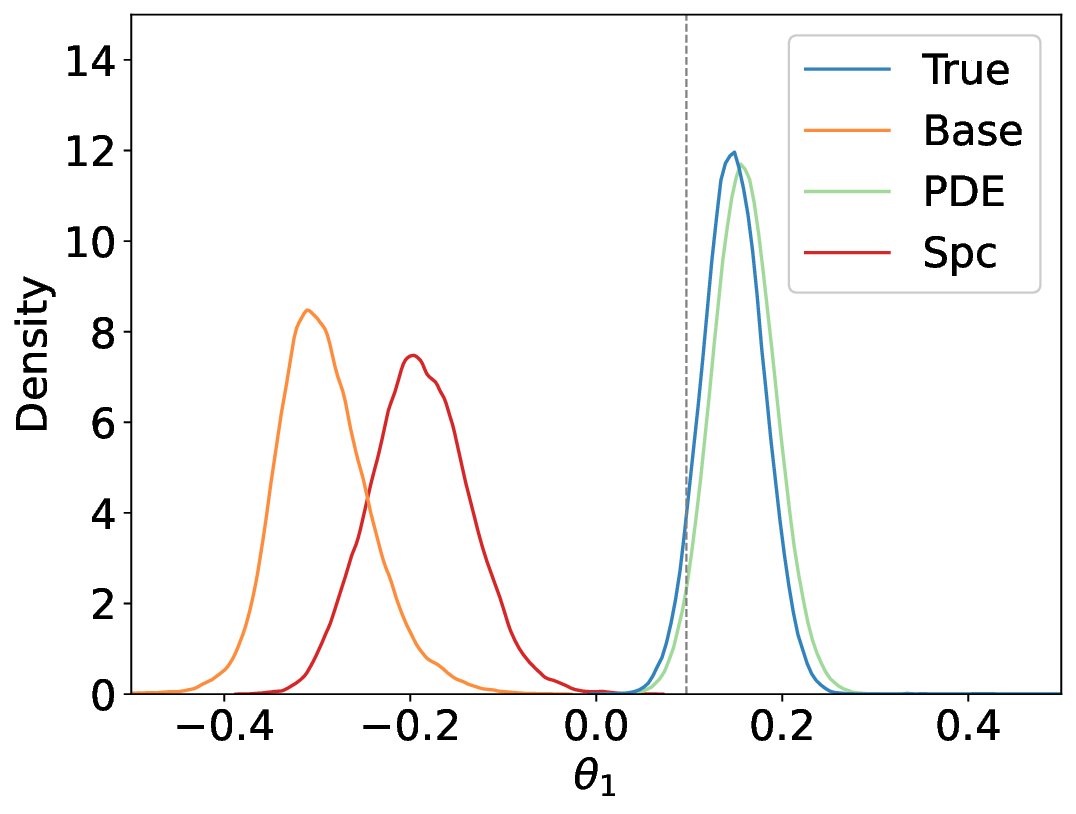}
  \caption{}
  \label{fig:exp3_infl1}
\end{subfigure}
\begin{subfigure}{.45\textwidth}
  \centering
  \includegraphics[width=\linewidth]{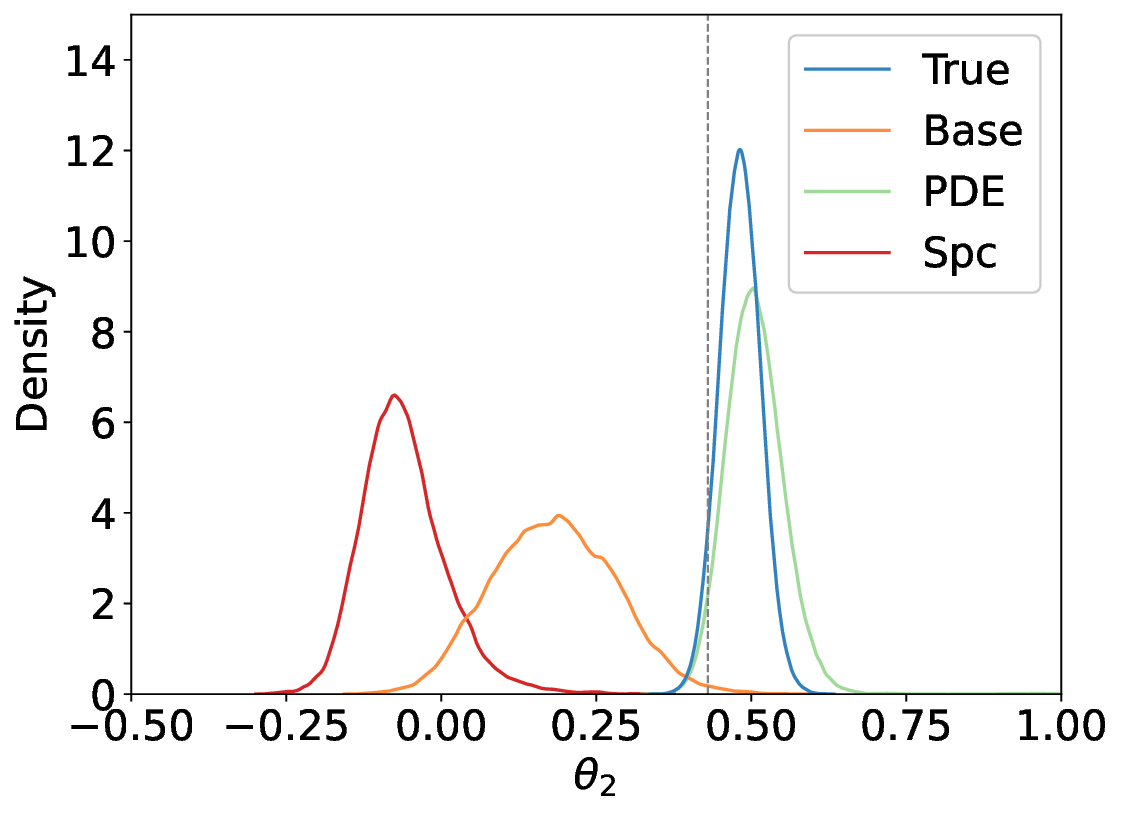}
  \caption{}
  \label{fig:exp3_infl2}
\end{subfigure}
\caption{Comparison of different models' marginal distribution when $N=4$, for PDE model $\Bar{N} = 10$ and $d_{f} = 50$.  (a) Mean-based approximation $\theta_1$. (b) Mean-based approximation $\theta_2$. (c) Marginal approximation $\theta_1$. (d) Marginal approximation $\theta_2$. $\mathcal{G}_X$ is the integrals of solution $u$ in \eqref{eq:example1} with diffusion coefficient \eqref{eq:pc}.}
\label{fig:exp3_marg}
\end{figure}

\subsubsection{Parametric expansion for the diffusion coefficient}
In this example, we study again \eqref{eq:example1}, but this time instead of working with a piecewise constant diffusion coefficient we assume that the diffusion coefficient satisfies the following parametric expansion
\begin{equation}\label{eq:KL}
\kappa(\btheta,x) = \exp{\left(\sum_{n=1}^{2}\sqrt{a_n}\theta_n b_n(x)\right)}
\end{equation}
where $a_n = \frac{8}{\omega^2_n + 16}$  and $b_{n}(x) = A_n (\sin(\omega_n x) + \frac{\omega_n}{4}\cos(\omega_n x))$, $\omega_n$ is the $n_{th}$ solution of the equation $\tan(\omega_n) = \frac{8\omega_n}{\omega^2_n - 16}$ and $A_n$ is a normalisation constant which makes $\|b_n\| = 1$.

In terms of the inverse problem setting, we are using the same parameters as before ($\btheta^{\dagger} = [0.098,0.430]$, $\dz = 6$, noise level $\sigma^2_{\eta} = 10^{-4}$). The number of training points for all the emulators has been set to $N=4$ (chosen using the Halton sequence), while in the case of the PDE-constrained emulator we have used $\Bar{N}=10$ and $d_{f}=8$.  Furthermore, throughout this numerical experiment, we take the prior of the parameters to be the uniform distribution on $[-1,1]^{2}$. For the choices of kernels, we use the squared exponential kernel for both $k_p$ and $k_s$.


As in the previous experiments, we produce $10^{6}$ samples of the posteriors using MALA, and use the results obtained by a mean-based approximation with the baseline model for $N=10^{2}$ training points as the ground truth. 
\begin{figure}
\centering
\begin{subfigure}{.45\textwidth}
  \centering
  \includegraphics[width=\linewidth]{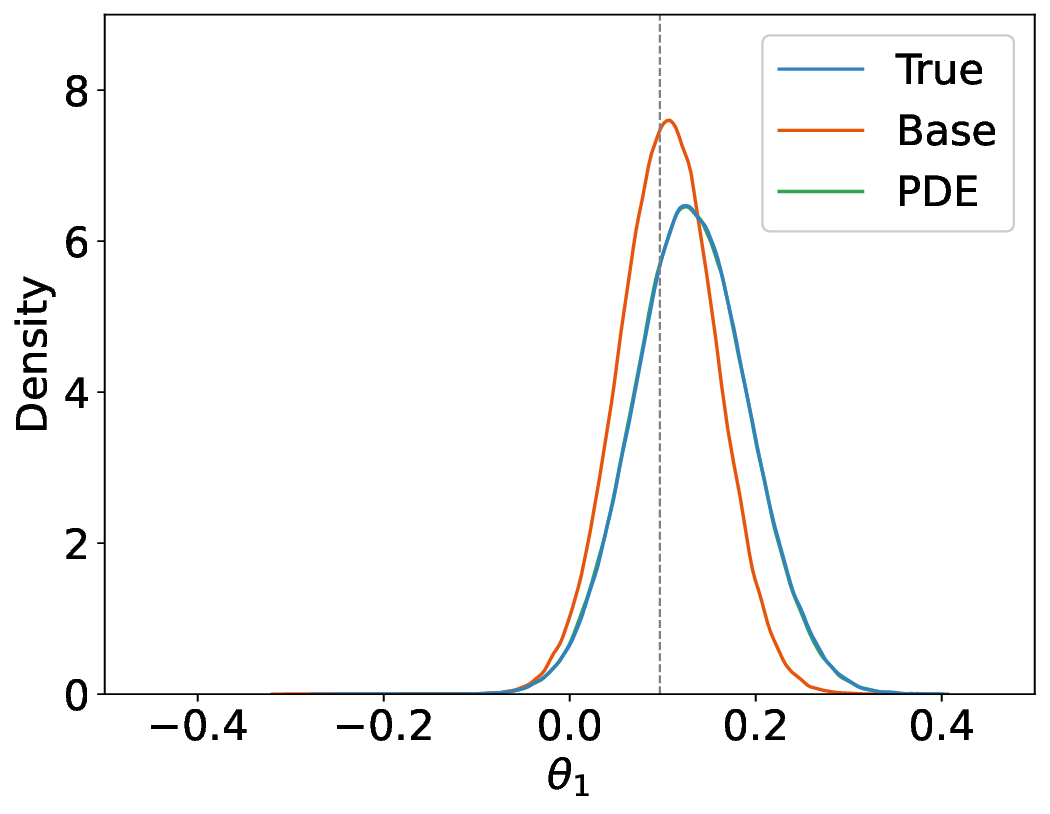}
  \caption{}
  \label{fig:exp2_mean1}
\end{subfigure}%
\begin{subfigure}{.45\textwidth}
  \centering
  \includegraphics[width=\linewidth]{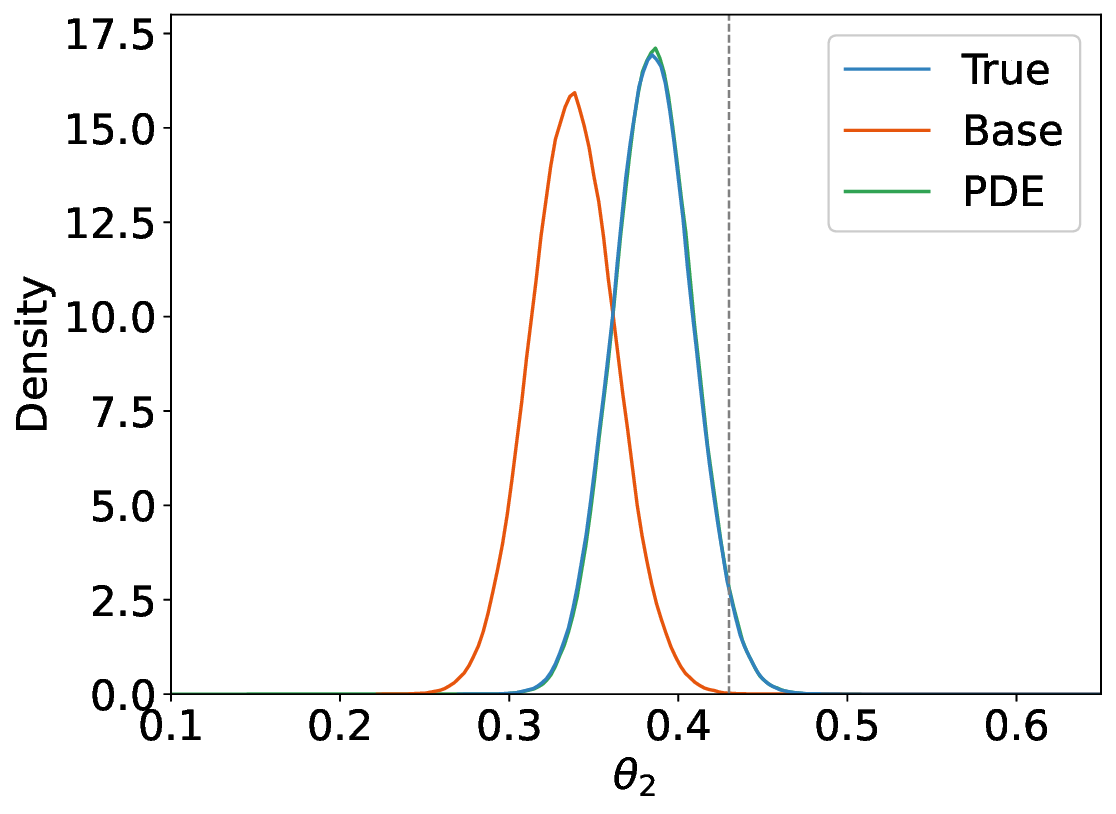}
  \caption{}
  \label{fig:exp2_mean2}
\end{subfigure}
\begin{subfigure}{.45\textwidth}
  \centering
  \includegraphics[width=\linewidth]{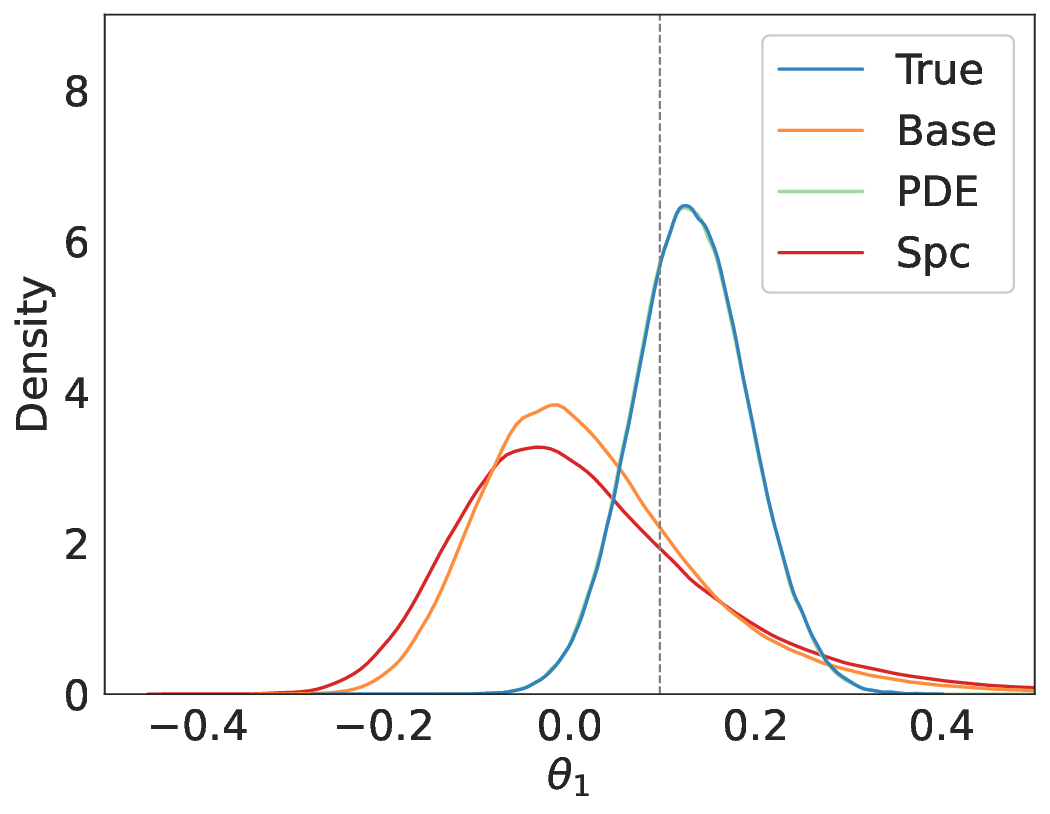}
  \caption{}
  \label{fig:exp2_infl1}
\end{subfigure}
\begin{subfigure}{.45\textwidth}
  \centering
  \includegraphics[width=\linewidth]{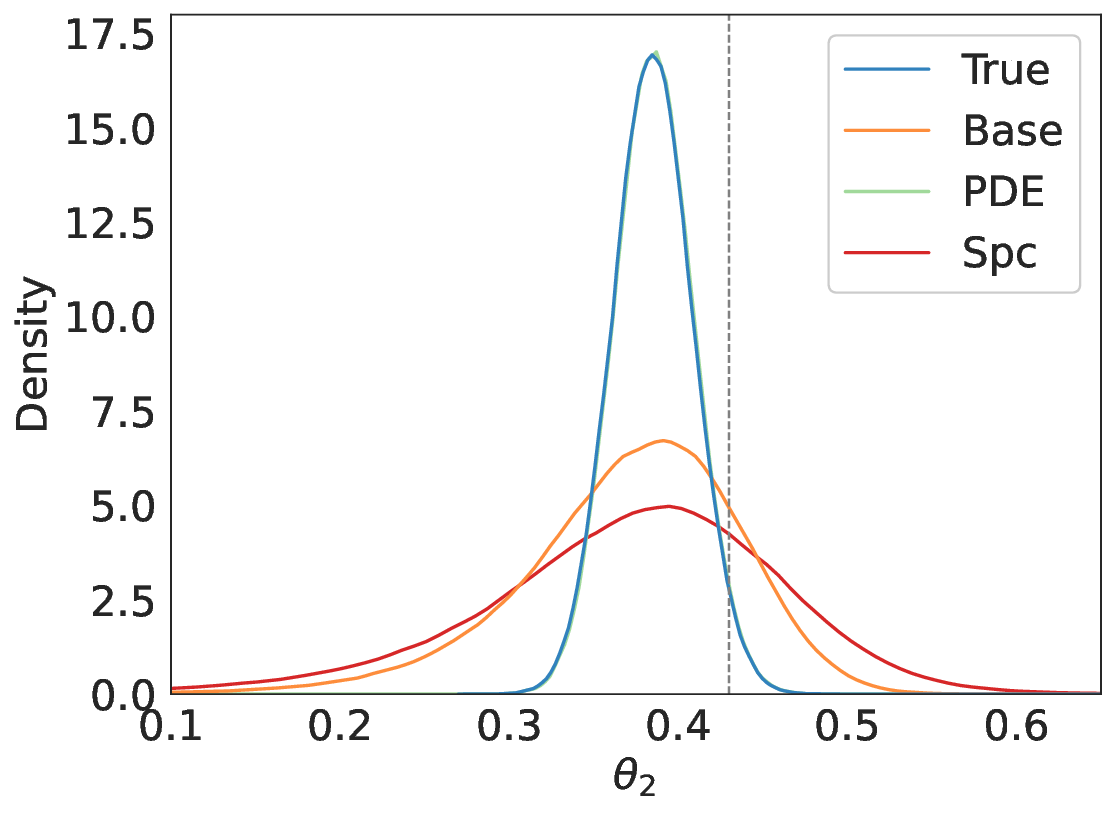}
  \caption{}
  \label{fig:exp2_infl2}
\end{subfigure}
\caption{Comparison of different models' marginal distribution when $N=4$, for PDE model $\Bar{N} = 10$ and $d_f = 8$. (a) Mean-based approximation  for the $\theta_1$ marginal. (b) Mean-based approximation for the  $\theta_2$ marginal. (c) Marginal approximation of the $\theta_1$ marginal. (d) Marginal approximation of the $\theta_2$ marginal. $\mathcal{G}_X$ is the discretised solution $u$ in \eqref{eq:example1} with diffusion coefficient \eqref{eq:KL} and $\dtheta = 2$. }
\label{fig:exp2_marg}
\end{figure}
We now plot in Figure \ref{fig:exp2_marg} the $\theta_{1}$ and $\theta_{2}$ marginals for the different Gaussian emulators both in the case of mean-based and marginal posterior approximations. In particular,  as we can see in  Figure \ref{fig:exp2_marg}(a)-(b) for the mean-based posterior approximations,  the baseline and spatially correlated model fail to capture the true posterior while this is not the case for the PDE-constrained model since the agreement with the true posterior is excellent. When looking at the marginal approximations in Figure \ref{fig:exp2_marg}(c)-(d) we can see that the marginals for the baseline and spatially correlated model move closer towards the true value $\btheta^{\dagger}$ and exhibit variance inflation. This is, however, not the case for the PDE-constrained model since again it is in excellent agreement with the true posterior.

\subsubsection{Ten-dimensional parametric expansion diffusion coefficient}

We will  now increase the dimension of the diffusion coefficient from $\mathbf{d_\theta} = 2$ to $ \mathbf{d_\theta} = 10$ in \eqref{eq:KL}, to test the proposed method in a relatively high dimensional space. In particular,  we divide the interval $[0,1]$ into 12 sub-intervals of equal length and fix the value of $\kappa$ to be $0$ and $1$ at the 2 ends, respectively. The values of on the remaining ten intervals are our unknown $\btheta$. With regard to the inverse problem setting, we set $$\btheta^{\dagger} = [0.098, 0.430, 0.206, 0.090, -0.153, 0.292, -0.125, 0.784, 0.927, -0.233]$$ and we increase the number of observation points to $\dz = 20$. The level of noise is same as before ($\sigma_{\eta}^2 = 10^{-4}$). The number of training points for all emulators is again set to be $N = 4$, and for the PDE-constrained emulator we use $\Bar{N} = 50$, $d_f = 25$ and $d_g = 2$.  
For the choices of kernels, we use the squared exponential kernel for both $k_p$ and $k_s$.

We now use the MALA algorithm to obtain $10^{7}$ samples of the approximate posteriors. In this relatively high-dimensional setting, we need longer chains for the sampling algorithm to converge. Meanwhile, computation of a suitable "ground truth" is prohibitively expensive, so we only compare the sampling result with the true parameter $\btheta^{\dagger}$. The number of training points $N = 4$ is far from enough for the baseline Gaussian process model to give an accurate prediction. From the Figure \ref{fig:exp6_mean}, we can see that the mean-based posterior approximation with the baseline model can only give a reasonable approximation for the first few variables, for the rest of the variables the approximation could not put any density around the true value. Adding spatial correlation into the model helps the approximation move toward the true value (Figure \ref{fig:exp6_infl}), but it still cannot correctly approximate the posterior for the last few variables. The performance of the PDE-constrained model is much better than the other models, it is placing the posterior mass around the true value for all variables. 

\begin{figure}
\centering
\begin{subfigure}{1.1\textwidth}
  \centering
  \includegraphics[width=\linewidth]{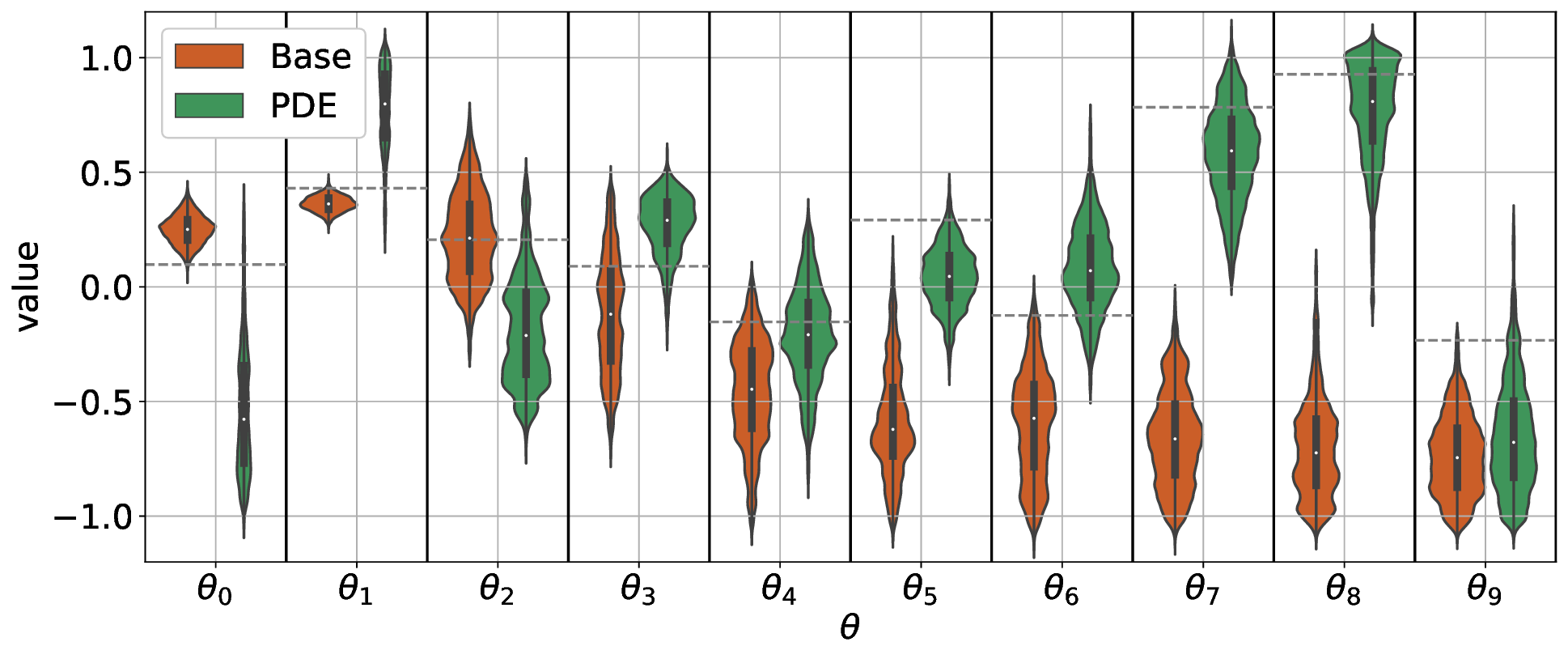}
  \caption{}
  \label{fig:exp6_mean}
\end{subfigure}
\begin{subfigure}{1.1\textwidth}
  \centering
  \includegraphics[width=\linewidth]{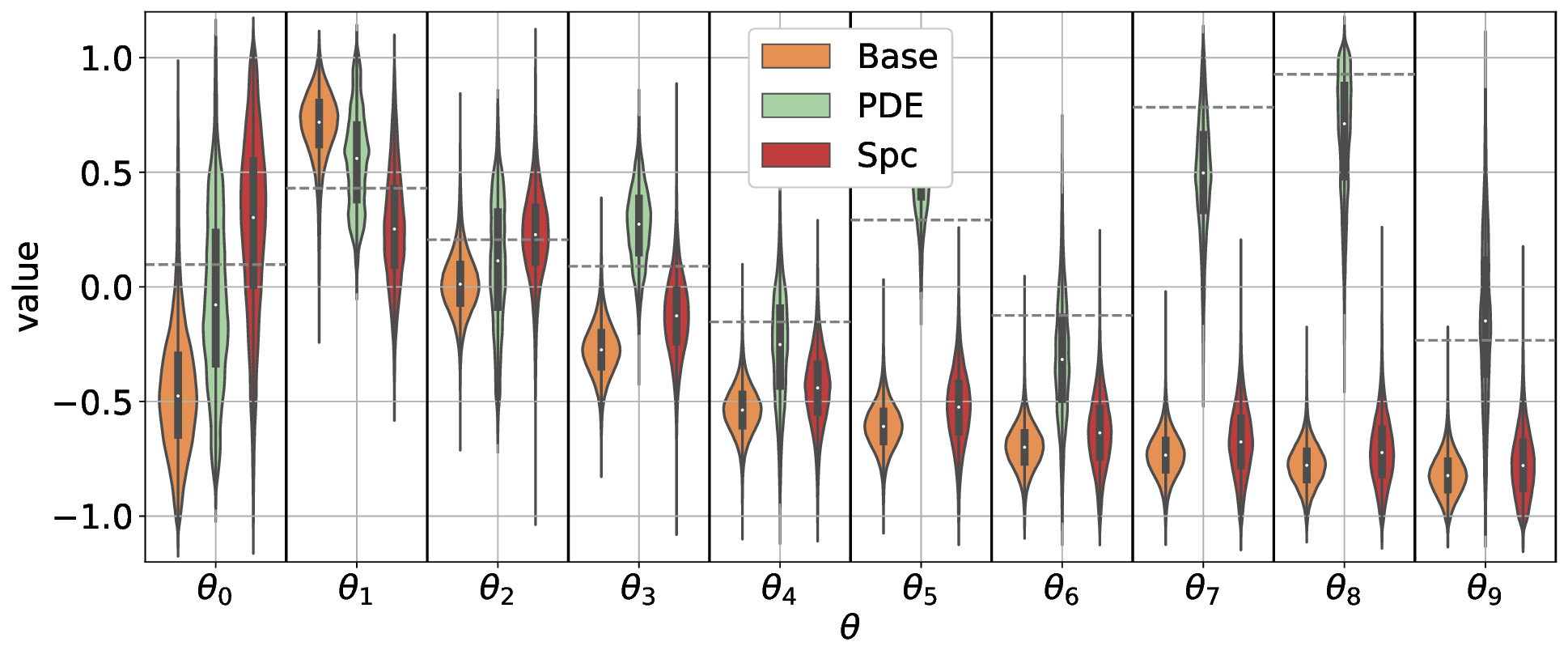}
  \caption{}
  \label{fig:exp6_infl}
\end{subfigure}
\caption{Comparison of different models' marginal distribution when $N=4$, for PDE model $\Bar{N} = 50$ and $d_f = 25$. (a) Mean-based approximation (b) Marginal approximation. $\mathcal{G}_X$ is the discretised solution $u$ in \eqref{eq:example1} with diffusion coefficient \eqref{eq:KL} and $\dtheta = 10$.}
\label{fig:exp2_marga}
\end{figure}

\subsection{Two spatial dimensions}
\label{sec:exp5}

\subsubsection{Two-dimensional piece-wise constant diffusion coefficient}
In this example, we increase the spatial dimension from $d_\mathbf{x} = 1$ to $d_\mathbf{x} = 2$ and use a 2 dimensional piece-wise constant as the diffusion coefficient. The values of the diffusion coefficient are set in a similar way to the previous example, depending only on the first dimension of $\mathbf{x}$:
\begin{equation}\label{eq:pc2}
 \kappa(\mathbf{x},\btheta)=
 \left\{\begin{array}{lr}
        0, & \text{for } x_1 \in [0,\frac{1}{4}),\\
        \\
        \theta_{1}, & \text{for } x_1 \in [\frac{1}{4},\frac{1}{2}),\\
        \\
       \theta_{2}, & \text{for } x_1 \in [\frac{1}{2},\frac{3}{4}),\\
       \\
       1,  & \text{for } x_1 \in [\frac{3}{4},1].
        \end{array}\right.
\end{equation} The boundary conditions are a mixture of Neumann and Dirichlet conditions, given by \begin{align*}
 &\partial_{x_1} u(x_1,0) = \partial_{x_1} u(x_1,1) = 0, &\text{for } x_1 \in [0,1], \\
 &  u(0,x_2) = 1, \quad u(1,x_2) = 0, &\text{for } x_2 \in [0,1].
\end{align*}
These boundary conditions define a {\em flow cell}, with no flux at the top and bottom boundary ($x_2=0,1$) and flow from left to right induced by the higher value of $u$ at $x_1=0$.

Again, we take the prior of the parameters to be the uniform distribution on $[-1,1]^{2}$, approximated by the $\lambda-$Moreau-Yoshida envelope with $\lambda = 10^{-3}$. For the observation, we generate our data $\bz$ according to equation \eqref{equ:ivp} for the value  $\btheta^{\dagger}=[0.098, 0.430]$ for $\dz=6$ (chosen to be the first $6$ points in the Halton sequence) and for noise level $\sigma^{2}_{\eta}=10^{-5}$. In addition, for the baseline and spatially correlated model, we have used $N=4$ training points (chosen to be the first $4$ points in the Halton sequence), while additionally for the PDE-constrained model, we have used $\Bar{N} = 30$, $d_{{f}}=30$ and  $d_g=8$, corresponding to 2 equally spaced points on each boundary. For the covariance kernels, we let $k_p$ be the squared exponential kernel and $k_s$ be the Mat\`ern kernel with $\nu = \frac{5}{2}$.

We plot the mean-based approximate posteriors marginals in Figure \ref{fig:exp5_mean1} and \ref{fig:exp5_mean2}. We can see that in this case, the PDE-constrained model significantly improves the approximation accuracy, which is different from the previous piece-wise constant diffusion coefficient example in 1 spatial dimension. In Figures \ref{fig:exp5_infl1} and \ref{fig:exp5_infl2}, we compare the marginal approximation for the three models. We see that the PDE-constrained model performs better than the other two models.

\begin{figure}
\centering
\begin{subfigure}{.45\textwidth}
  \centering
  \includegraphics[width=\linewidth]{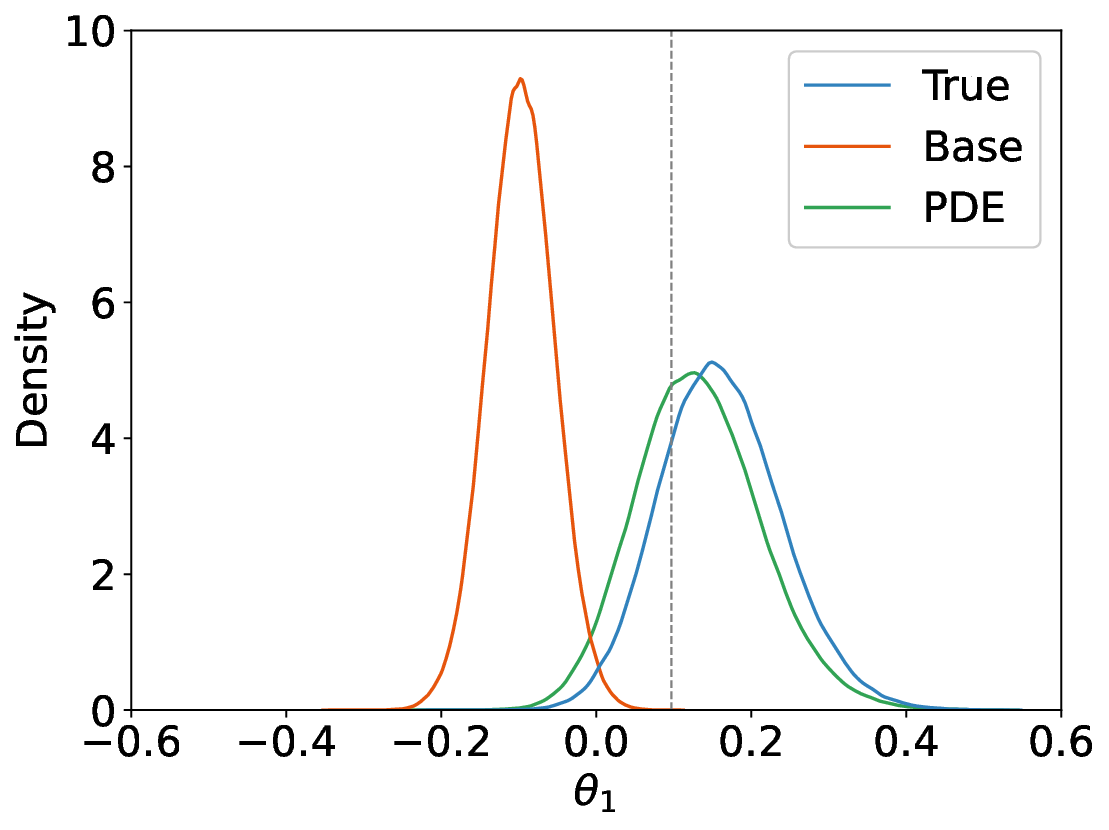}
  \caption{}
  \label{fig:exp5_mean1}
\end{subfigure}%
\begin{subfigure}{.45\textwidth}
  \centering
  \includegraphics[width=\linewidth]{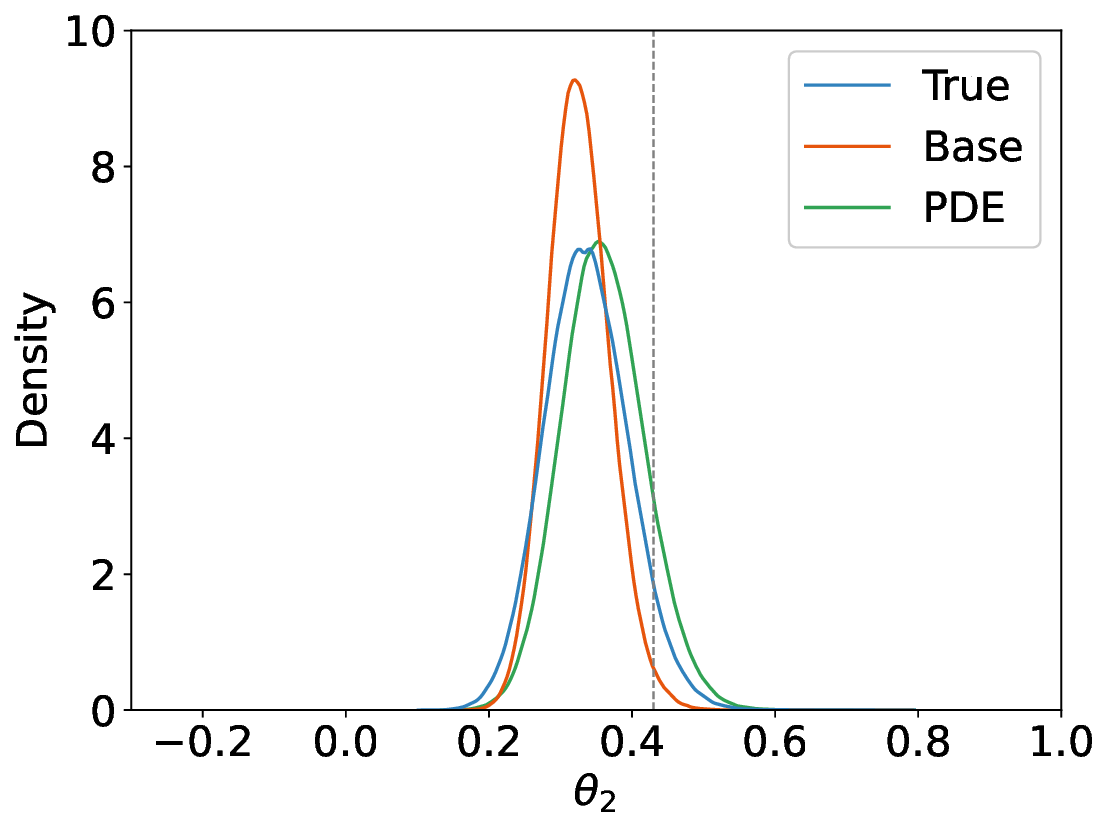}
  \caption{}
  \label{fig:exp5_mean2}
\end{subfigure}
\begin{subfigure}{.45\textwidth}
  \centering
  \includegraphics[width=\linewidth]{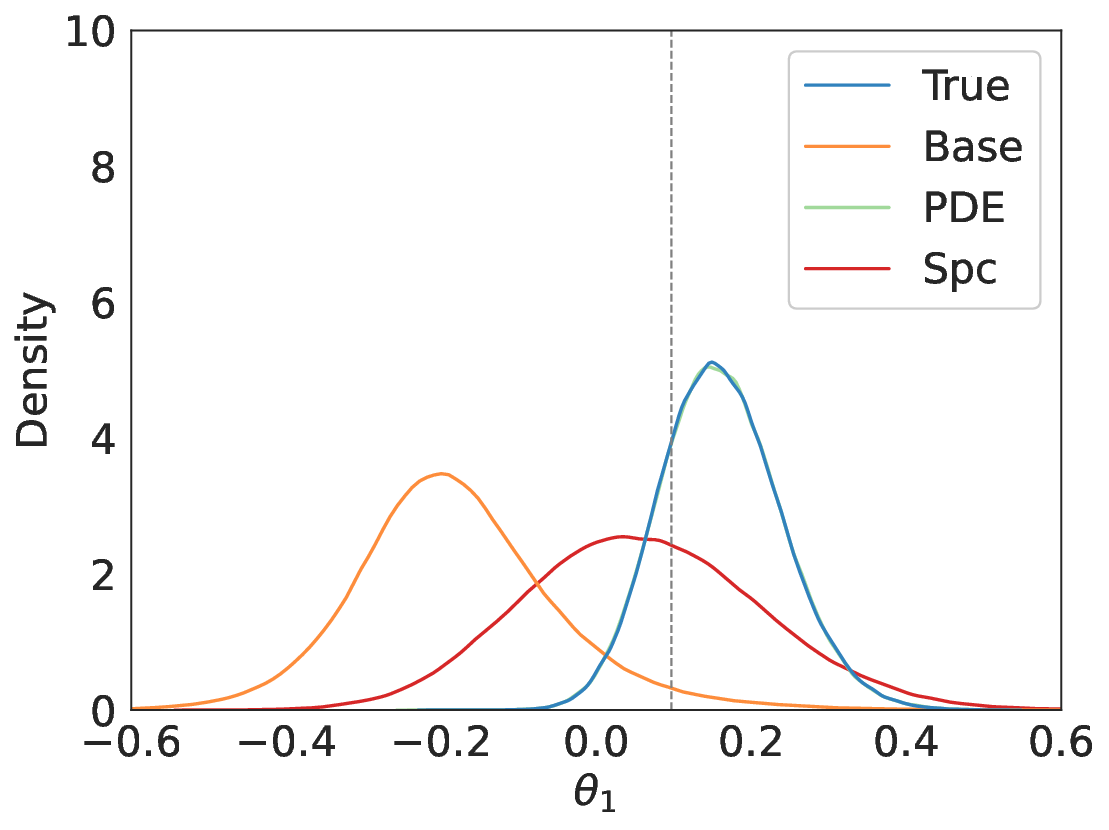}
  \caption{}
  \label{fig:exp5_infl1}
\end{subfigure}
\begin{subfigure}{.45\textwidth}
  \centering
  \includegraphics[width=\linewidth]{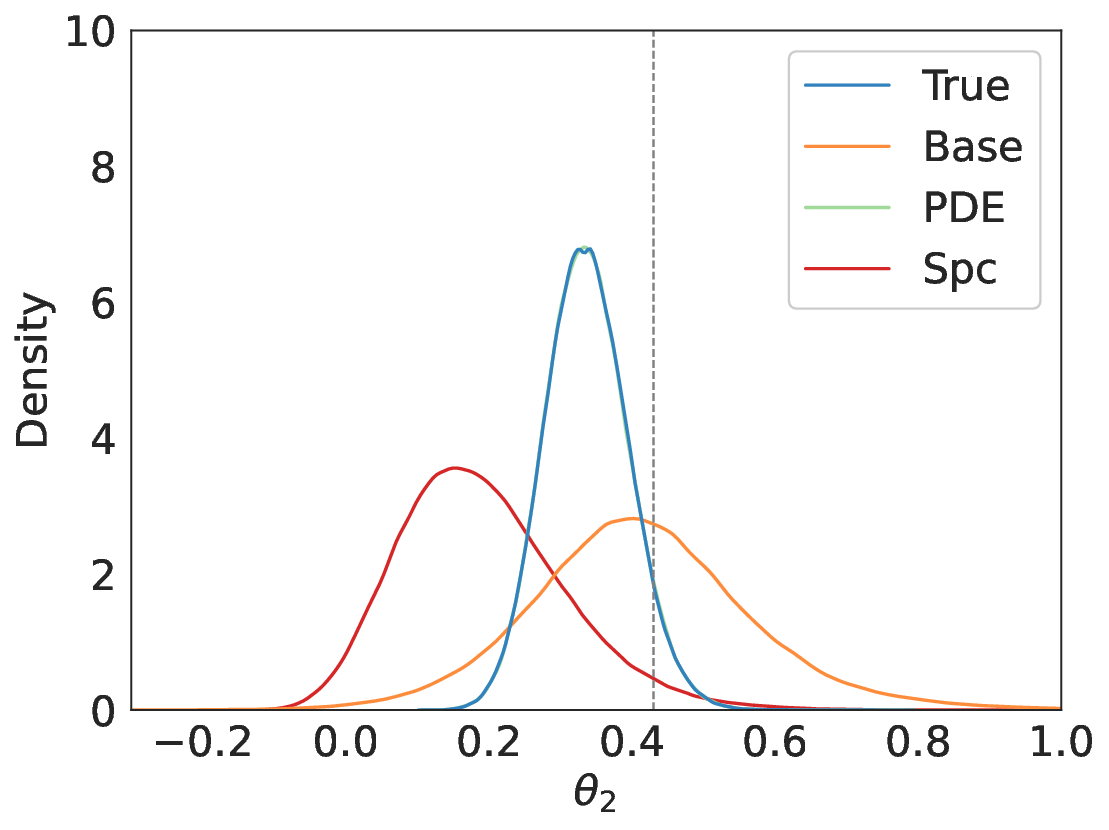}
  \caption{}
  \label{fig:exp5_infl2}
\end{subfigure}
\caption{Comparison of different models' marginal distribution when $N=4$, for PDE model $\Bar{N} = 30$ and $d_f = 30$. (a) Mean-based approximation  for the $\theta_1$ marginal. (b) Mean-based approximation for the  $\theta_2$ marginal. (c) Marginal approximation for the  $\theta_1$ marginal. (d) Marginal approximation for the  $\theta_2$ marginal. $\mathcal{G}_X$ is the discretised solution $u$ with $d_\mathbf{x} = 2$ and diffusion coefficient \eqref{eq:pc2}.}
\label{fig:exp5_marg}
\end{figure}

\subsection{Emulating the negative log-likelihood function}

As discussed in Section \ref{sec:emulatell}, we can emulate the negative log-likelihood (also called potential function) directly with Gaussian process regression. Since emulation of log-likelihood simplifies the structure of the problem, we are not able to incorporate spatial correlation or PDE constraints into the emulator. We have mean-based approximation \eqref{eq:potenmean} and marginal approximation \eqref{eq:poteninfl}. We test their performance using previous examples: problem \eqref{eq:example1} with diffusion coefficient \eqref{eq:pc} with $d_\mathbf{x} = 1$ and $d_\mathbf{x} = 2$. All parameters are kept the same as in Section \ref{sec:exp1} and Section \ref{sec:exp5}. Due to its simplified structure, the value of $d_\mathbf{x}$ makes no difference for the emulator since the only information taken by the emulator is the training data $\Phi(\Theta)$. 

In Figure \ref{fig:exp_pot_N4}, we compare the mean-based approximation with emulation of the log-likelihood $\Phi$ and the observation operator $\mathcal G_X$ using baseline model. We see that the results are very different in both examples. For the $d_\mathbf{x}=1$ example, emulating log-likelihood function performs better than emulation of observation with baseline model, the approximated posterior is closer to the true posterior. For the $d_\mathbf{x} = 2$ case, its performance is much worse. Hence, emulating the log-likelihood with a small amount of data could be less reliable compared to emulating observation.
If we increase the number of training data to $N=10$ for the $d_\mathbf{x} = 2$ case, we can see the improvement of accuracy (Figure \ref{fig:exp5_pot_N10}), but it is still worse than emulating observation with baseline model. 

Similarly, marginal approximations of the posterior with emulation of the log-likelihood appear to also be less reliable when the amount of training data is small, see Figure \ref{fig:exp_pot_infl}.  Including more training point can again improve the performance. The main advantage of emulating the log-likelihood function directly is its computational cost, which is much smaller than emulating in observation space. Detailed computational times are listed in the following section.

\begin{figure}
\centering

\begin{subfigure}{.45\textwidth}
  \centering
  \includegraphics[width=\linewidth]{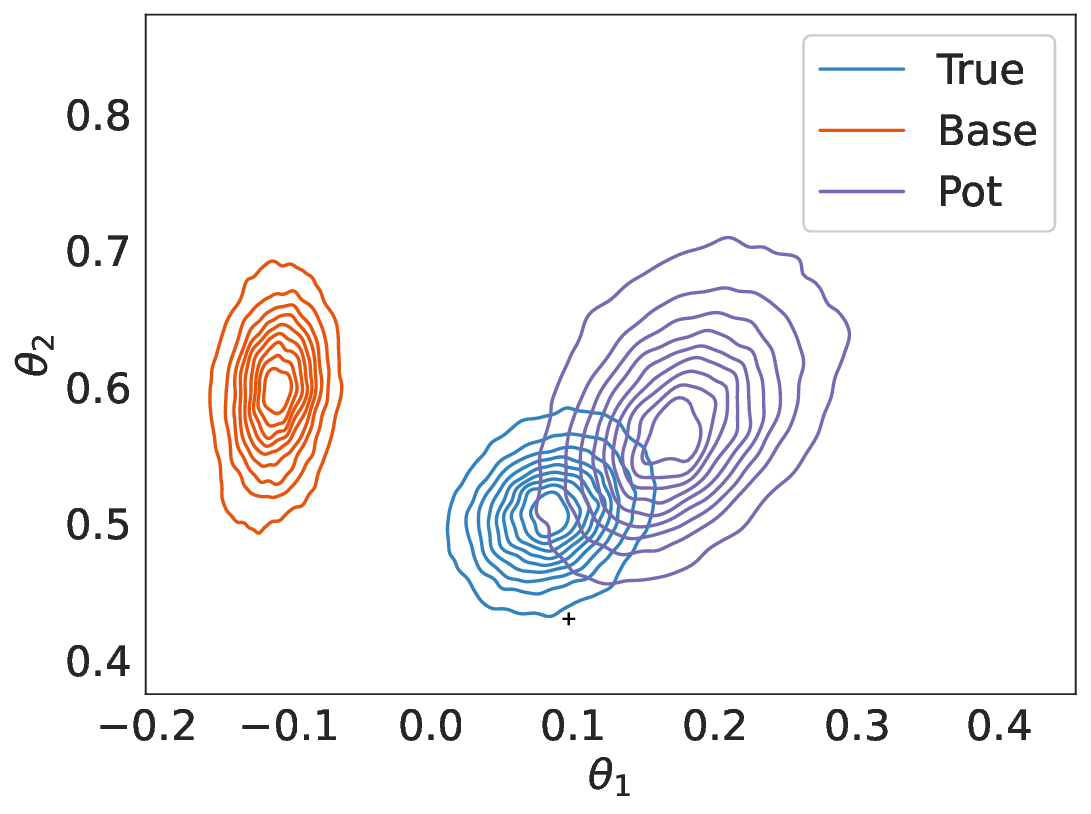}
  \caption{}
  \label{fig:exp1_pot}
\end{subfigure}
\begin{subfigure}{.45\textwidth}
  \centering
  \includegraphics[width=\linewidth]{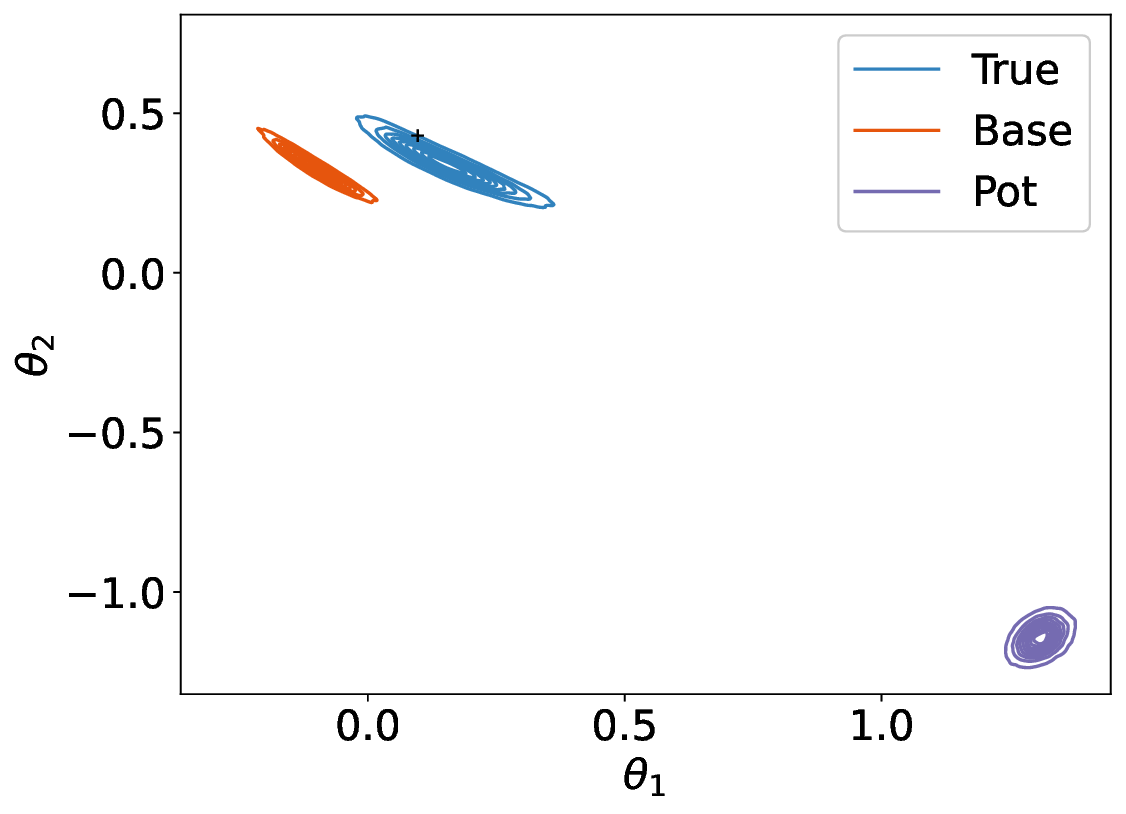}
  \caption{}
  \label{fig:exp5_pot}
\end{subfigure}
\caption{Comparison of emulating log-likelihood function and emulating observations when $N=4$. Both approximation are mean-based approximation. $\mathcal{G}_X$ is the negative log-likelihood function in: (a) problem \eqref{eq:example1} with diffusion coefficient \eqref{eq:pc} with $d_\mathbf{x} = 1$; (b) problem \eqref{eq:example1} with diffusion coefficient \eqref{eq:pc} with $d_\mathbf{x} = 2$.}
\label{fig:exp_pot_N4}
\end{figure}

\begin{figure}
\centering
\begin{subfigure}{.45\textwidth}
  \centering
  \includegraphics[width=\linewidth]{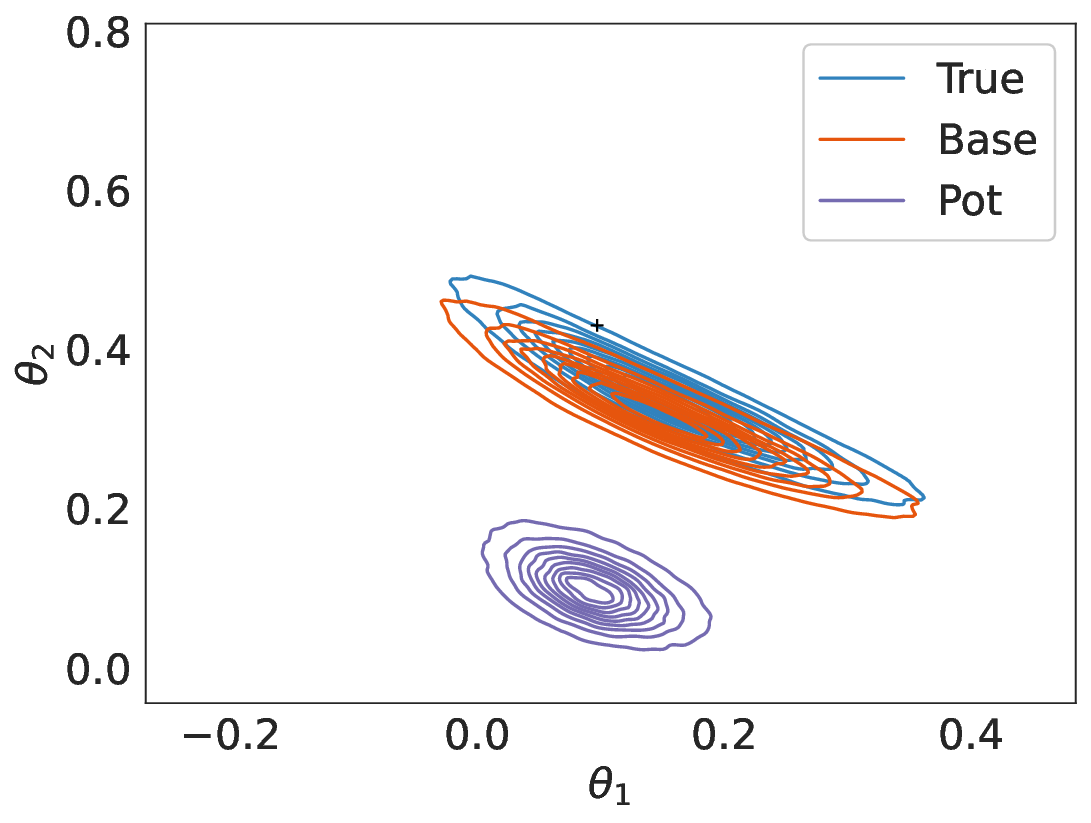}
  \caption{}
  \label{fig:exp1_pot}
\end{subfigure}
\caption{Accuracy of emulator is improved when N increases ($N=10$). $\mathcal{G}_X$ is the negative log-likelihood function in problem \eqref{eq:example1} with diffusion coefficient \eqref{eq:pc} with $d_\mathbf{x} = 2$ and mean-based approximation.}
\label{fig:exp5_pot_N10}
\end{figure}

\begin{figure}
\centering

\begin{subfigure}{.45\textwidth}
  \centering
  \includegraphics[width=\linewidth]{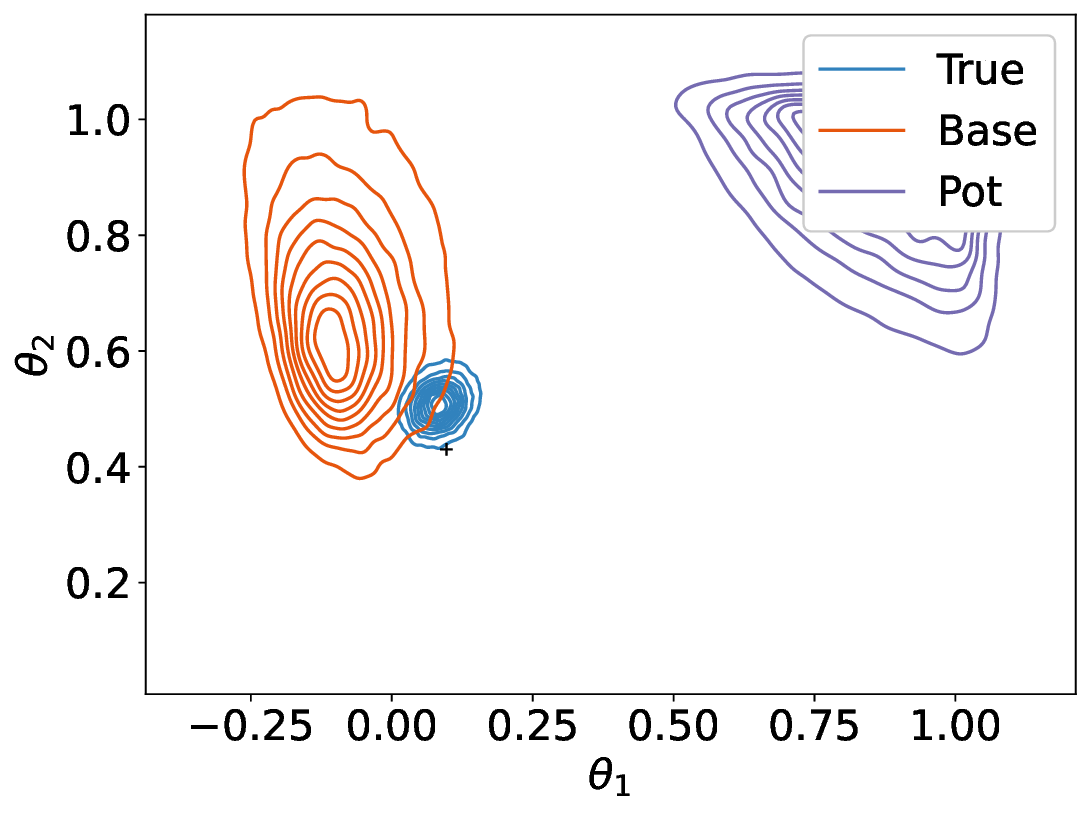}
  \caption{}
  \label{fig:exp_pot_infl_N4}
\end{subfigure}
\begin{subfigure}{.45\textwidth}
  \centering
  \includegraphics[width=\linewidth]{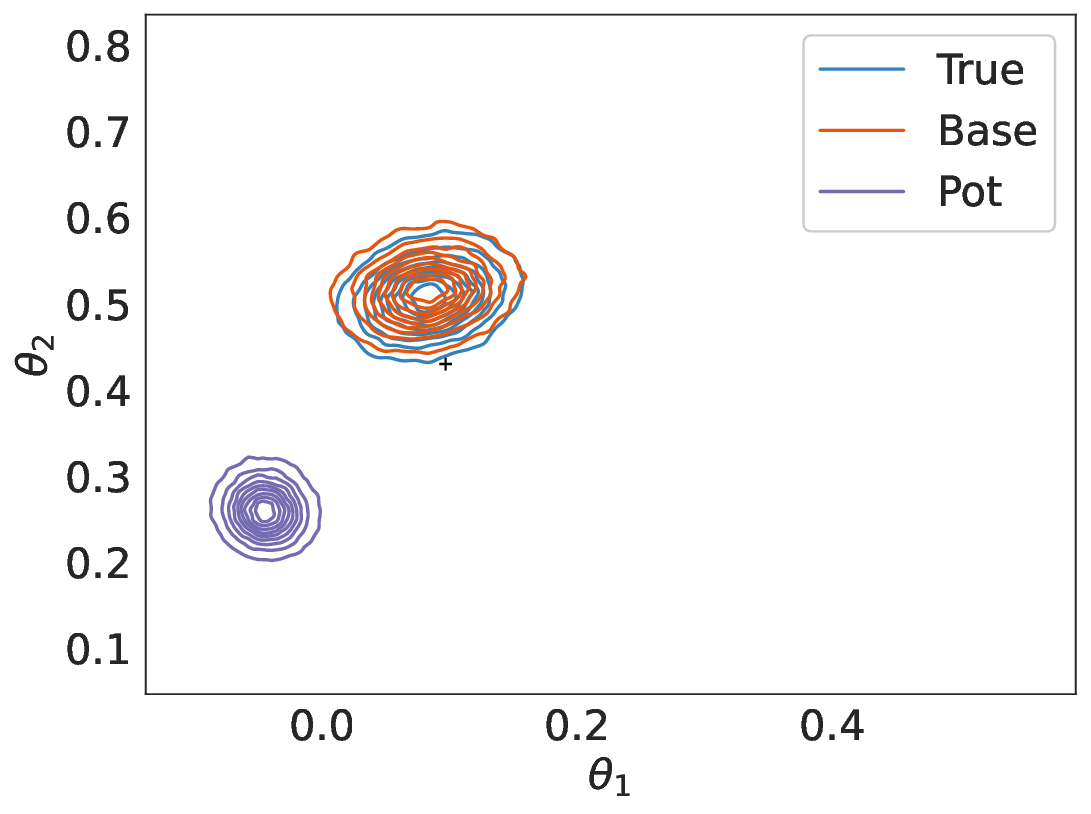}
  \caption{}
  \label{fig:exp_pot_infl_N10}
\end{subfigure}
\caption{(a) Marginal approximation with $N=4$ (b) Marginal approximation with $N=10$. $\mathcal{G}_X$ is the negative log-likelihood function in problem \eqref{eq:example1} with diffusion coefficient \eqref{eq:pc} with $d_\mathbf{x} = 1$.}
\label{fig:exp_pot_infl}
\end{figure}

\subsection{Computational timings}
In this section, we discuss computational timings. We focus on the computational gains resulting from using Gaussian process emulators instead of the PDE solution in the posterior (see Table \ref{table:1}) and the relative costs of sampling from the various approximate posteriors (see Tables \ref{table:2}, \ref{table:3} and \ref{table:4}).

Table \ref{table:1} below gives average computational timings comparing the evaluation of the solution of the PDE using Firedrake with using the Gaussian process surrogate model. For the baseline surrogate model, the two primary costs are (i) computing the coefficients $\boldsymbol{\alpha} = K(\Theta,\Theta)^{-1} \mathcal G_X(\Theta)$, which is an {\em offline} cost and only needs to be done once, and (ii) computing the predictive mean $m_N^f(\btheta) = K(\btheta,\Theta) \boldsymbol{\alpha}$, which is the {\em online} cost and needs to be done for every new test point $\btheta$. We see that evaluating $m_N^f(\btheta)$ is orders of magnitude faster than evaluating $\mathcal G_X(\btheta)$.

\begin{table}[h!]
\centering
\begin{tabular}{|c | c c c|} 
 \hline
 Set-up & $\mathcal G_X(\btheta)$ & $m_N^{\mathcal G_X}(\btheta)$ & $\boldsymbol{\alpha}$ \\ [0.5ex] 
 \hline\hline
 $\dtheta = 2, \dz=6, D = (0,1), N=4$ & $3.2 \times 10^{-1}$s & $ 1.0 \times 10^{-4}$s & $2.5 \times 10^{-4}$s \\ 
 $\dtheta = 2, \dz=6, D = (0,1), N=20$ & $3.2 \times 10^{-1}$s & $ 1.3 \times 10^{-4}$s & $6.8 \times 10^{-4}$s \\
  $\dtheta = 10, \dz=18, D = (0,1), N=4$ & $3.2 \times 10^{-1}$s & $ 1.6 \times 10^{-4}$s & $4.5 \times 10^{-4}$s \\ 
  $\dtheta = 2, \dz=6, D = (0,1)^2, N=4$ & $7.6 \times 10^{0}s$ & $ 1.0 \times 10^{-4}s$ & $5.3 \times 10^{-4}s$ \\ 
[1ex] 
 \hline
\end{tabular}
\caption{Timings of PDE solution vs baseline Gaussian process emulator}
\label{table:1}
\end{table}

In Tables \ref{table:2}, \ref{table:3} and \ref{table:4}, we compare average computational timings of drawing one sample from the approximate posterior with different models. In Table \ref{table:2}, we see that the mean-based approximation with the PDE-informed prior is more expensive than the one with the baseline prior, by a factor of 2-4 depending on the setting. This is to be expected, since the PDE-informed posterior mean $\mathbf{m}^{{\mathcal G}_X}_{N,X_f,X_g}$ involves matrices of larger dimensions than the baseline posterior mean $\mathbf{m}^{{\mathcal G}_X}_{N}$. 

Table \ref{table:3} investigates the different marginal approximations. Compared to the mean-based approximations in Table \ref{table:1}, we see that the marginal approximations are more expensive by a factor of around 2 for the baseline model and around 3-10 for the PDE-constrained model. Within the different marginal approximations, the spatially correlated model is not much more expensive than the baseline model. Depending on the setting, the PDE-constrained model is 2-30 times more expensive.

In Table \ref{table:4}, we can see that emulating the log-likelihood significantly reduces the cost of sampling from the mean-based and marginal approximations, by around a factor of 20 compared to the baseline model for emulating the observations. 


\begin{table}[h!]
\centering
\begin{tabular}{|c |c c|} 
 \hline
 Set-up & $\pi^{N,\mathcal{G}_X}_{\mathrm{mean}}$ & $\pi^{N,\mathcal{G}_X,\mathrm{PDE}}_{\mathrm{mean}}$  \\ [0.5ex] 
 \hline\hline
 $\dtheta = 2, \dz=6, D = (0,1), N=4$ & $ 8.5 \times 10^{-4}$s & $ 1.2 \times 10^{-3}$s ($\Bar{N} = 10, d_f = 20$) \\ 
  $\dtheta = 2, \dz=6, D = (0,1), N=20$ & $ 9.3 \times 10^{-4}$s & $ 1.4 \times 10^{-3}$s ($\Bar{N} = 10, d_f = 20$)\\ 
  $\dtheta = 10, \dz=18, D = (0,1), N=4$ & $ 2.6 \times 10^{-3}$s & $1.2 \times 10^{-2}$s  ($\Bar{N} = 50, d_f = 25$) \\ 
  $\dtheta = 2, \dz=6, D = (0,1)^2, N=4$  & $ 8.5 \times 10^{-4}$s & $1.6 \times 10^{-3}$s ($\Bar{N} = 30, d_f = 30$) \\ 
[1ex] 
 \hline
\end{tabular}
\caption{Timings of different mean-based approximations (baseline and PDE-constrained)}
\label{table:2}
\end{table}

\begin{table}[h!]
\centering
\begin{tabular}{|c |c cc|} 
 \hline
 Set-up & $\pi^{N,\mathcal{G}_X}_{\mathrm{marginal}}$ & $\pi^{N,\mathcal{G}_X,s}_{\mathrm{marginal}}$&$\pi^{N,\mathcal{G}_X,\mathrm{PDE}}_{\mathrm{marginal}}$  \\ [0.5ex] 
 \hline\hline
 $\dtheta = 2, \dz=6, D = (0,1), N=4$ & $ 1.7 \times 10^{-3}$s&$2.2\times10^{-3}$s& $ 3.2 \times 10^{-3}$s  \\ 
  $\dtheta = 2, \dz=6, D = (0,1), N=20$ & $ 2.0 \times 10^{-3}$s & $ 2.6 \times 10^{-3}$s&$ 5.6 \times 10^{-3}$s \\ 
  $\dtheta = 10, \dz=18, D = (0,1), N=4$ & $ 3.4 \times 10^{-3}$s & $ 3.6 \times 10^{-3}$s& $1.1 \times 10^{-1}$s  \\ 
  $\dtheta = 2, \dz=6, D = (0,1)^2, N=4$  & $ 1.7 \times 10^{-3}$s & $2.2 \times 10^{-3}$s& $4.8 \times 10^{-2}$s  \\ 
[1ex] 
 \hline
\end{tabular}
\caption{Timings of different marginal approximations (baseline, spatially correlated and PDE-constrained); $\overline N$ and $d_f$ are as in Table \ref{table:2}}
\label{table:3}
\end{table}

\begin{table}[h!]
\centering
\begin{tabular}{|c |c c|} 
 \hline
 Set-up & $\pi^{N,\Phi}_{\mathrm{mean}}$ & $\pi^{N,\mathcal{G}_X,\Phi}_{\mathrm{marginal}}$\\ [0.5ex] 
 \hline\hline
 $\dtheta = 2, \dz=6, D = (0,1), N=4$ & $ 3.4 \times 10^{-5}$s & $ 5.8 \times 10^{-5}$s  \\  
 $\dtheta = 2, \dz=6, D = (0,1)^2, N=4$ & $ 3.4 \times 10^{-5}$s & $ 5.8 \times 10^{-5}$s  \\ 
[1ex] 
 \hline
\end{tabular}
\caption{Timings of mean-based and marginal approximation when emulating the log-likelihood}
\label{table:4}
\end{table}



\section{Conclusions, discussion and actionable advice} \label{sec:conclusions}
Bayesian inverse problems in PDEs pose significant computational challenges. Application of state-of-the-art sampling methods, including MCMC methods, is typically computationally infeasible due to the large computational cost of simulating the underlying mathematical model for a given value of the unknown parameters. A solution to alleviate this problem is to use a surrogate model to approximate the PDE solution in Bayesian posterior distribution. In this work we considered the use of Gaussian process surrogate models, which are frequently used in engineering and geo-statistics applications and  offer the benefit of built-in uncertainty quantification in the variance of the emulator.

The focus of this work was on practical aspects of using Gaussian process emulators in this context, providing efficient MCMC methods and studying the effect various modelling choices in the derivation of the approximate posterior on its accuracy and computational efficiency. We now summarise the main conclusions of our investigation.
\begin{enumerate}
    \item {\bf Emulating log-likelihood vs emulating observations.} We can construct an emulator for the negative log-likelihood $\Phi$ or the parameter-to-observation map $\mathcal G_X$ in the likelihood \eqref{equ:likelihood}. 
    \begin{itemize}
        \item {\em Computational efficiency.} The log-likelihood $\Phi$ is always scalar-valued, independent of the number of observations $\dz$, which makes the computation of the approximate likelihood for a given values of the parameters $\btheta$ much cheaper than the approximate likelihood with emulated $\mathcal G_X$. The relative cost will depend on $\dz$. 
        \item {\em Accuracy.}  When only limited training data is provided, emulating $\mathcal G_X$ appears more reliable than emulating the log-likelihood, even with the baseline model. The major advantage of emulating $\mathcal G_X$ is that it allows us to include correlation between different observations, i.e. between the different entries of $\mathcal G_X$. This substantially increases the accuracy of the approximate posteriors, in particular if we use the PDE structure to define the correlations (see point 3 below).
    \end{itemize}
    \item {\bf Mean-based vs marginal posterior approximations.} We can use only the mean of the Gaussian process emulator to define the approximate posterior as in \eqref{eq:meanpost} and \eqref{eq:potenmean}, or we can make use of its full distribution to define the marginal approximate posteriors as in \eqref{eq:postinfl} and \eqref{eq:poteninfl}.
    \begin{itemize}
        \item {\em Computational efficiency.} The mean-based approximations are faster to sample from using MALA. This is due to simpler structure of the gradient required for the proposals.
        \item {\em Accuracy.} The marginal approximations correspond to a form of variance inflation in the approximate posterior (see Section \ref{ssec:gp_posterior}), representing our incomplete knowledge about the PDE solution. They thus combat over-confident predictions. They typically allocate larger mass to regions around the true parameter value than the mean-based approximations.
    \end{itemize}
    \item {\bf Spatial correlation and PDE-constrained priors.}  
    \begin{itemize}
        \item {\em Computational efficiency.} Introducing the spatially correlated model only affects the marginal approximation, and sampling from the marginal approximate posterior with the spatially correlated model is slightly slower than with baseline model. The PDE-constrained model significantly increases the computational times for both the mean-based and marginal approximations, by how much highly depends on the size of additional training data.
        \item {\em Accuracy.} Introducing spatial correlation improves the accuracy of the marginal approximation compared to the baseline model. The most accurate results are obtained with the PDE-constrained priors, which are problem specific and more informative. A benefit of the spatially correlated 
        model is that it does not rely on the underlying PDE being linear, and easily extends to non-linear settings. 
    \end{itemize}
\end{enumerate}

In summary, the marginal posterior approximations and spatially correlated/ PDE-constrained prior distributions provide mechanisms of increasing the accuracy of the inference and avoiding over-confident biased predictions, without the need to increase $N$. 
This is particularly useful in practical applications, where the number of model runs $N$ available to train the surrogate model may be very small due to constraints in time and/or cost. This does result in higher computational cost compared to mean-based approximations based on black-box priors, but may still be the preferable option if obtaining another training point is impossible or computationally very costly.

Variance inflation, as exhibited in the marginal posterior approximations considered in this work, is a known tool to improve Bayesian inference in complex models, see e.g. \cite{conrad2017statistical,calvetti2018iterative,fox2020randomized}. Conceptually, it is also related to including model discrepancy \cite{KH00,brynjarsdottir2014learning}. However, the approach to variance inflation presented in this work has several advantages. Firstly, the variance inflation being equal to the predictive variance of the emulator means that the amount of variance inflation included depends on the location $\btheta$ in the parameter space. We introduce more uncertainty in parts of the parameter space where we have less training points and the emulator is possibly less accurate. Secondly, the amount of variance inflation can be tuned in a principled way using standard techniques for hyper-parameter estimation in Gaussian process emulators. There is no need to choose a model for the variance inflation separately to choosing the emulator, since this is determined automatically as part of the emulator.

We did not discuss optimal experimental design in this work, i.e. how we should optimally choose the locations $\Theta$ of the training data. In practice this will also have a large influence on the accuracy of the approximate posteriors, especially for small $N$. In the context of inverse problems as considered here, one usually wants to place the training points in regions of parameter space where the (approximate) posterior places significant mass (see e.g. \cite{helin2023introduction} and the references therein). For a fair comparison between all scenarios, and to eliminate the interplay between optimal experimental design and other modelling choices, we have chosen the training points as a space-filling design in our experiments. We expect the same conclusions to hold with optimally placed points.

\section*{Acknowledgements}
The authors would like to thank the Isaac Newton Institute for Mathematical Sciences, Cambridge, for support and hospitality during the programme {\em Mathematical and statistical foundation of future data-driven engineering} where work on this paper was undertaken. This work was supported by EPSRC grants no EP/R014604/1 and  EP/V006177/1.

\bibliographystyle{siam}  
\bibliography{main}

\begin{thebibliography}{10}

\bibitem{BNT07}
{\sc I.~Babuska, F.~Nobile, and R.~Tempone}, {\em A stochastic collocation
  method for elliptic partial differential equations with random input data},
  SIAM J. Numerical Analysis, 45 (2007), pp.~1005--1034.

\bibitem{bauschke2011fixed}
{\sc H.~H. Bauschke, R.~S. Burachik, P.~L. Combettes, V.~Elser, D.~R. Luke, and
  H.~Wolkowicz}, {\em Fixed-point algorithms for inverse problems in science
  and engineering}, vol.~49, Springer Science \& Business Media, 2011.

\bibitem{BGJ11}
{\sc S.~Brooks, A.~Gelman, G.~Jones, and X.-L. Meng}, {\em Handbook of Markov
  Chain Monte Carlo}, CRC press, 2011.

\bibitem{brynjarsdottir2014learning}
{\sc J.~Brynjarsd{\'o}ttir and A.~O'Hagan}, {\em Learning about physical
  parameters: The importance of model discrepancy}, Inverse problems, 30
  (2014), p.~114007.

\bibitem{BWG08}
{\sc T.~Bui-Thanh, K.~Willcox, and O.~Ghattas}, {\em Model reduction for
  large-scale systems with high-dimensional parametric input space}, SIAM
  Journal on Scientific Computing, 30 (2008), pp.~3270--3288.

\bibitem{calvetti2018iterative}
{\sc D.~Calvetti, M.~Dunlop, E.~Somersalo, and A.~Stuart}, {\em Iterative
  updating of model error for bayesian inversion}, Inverse Problems, 34 (2018),
  p.~025008.

\bibitem{10.1063/1.4985359}
{\sc J.~Cockayne, C.~Oates, T.~Sullivan, and M.~Girolami}, {\em {Probabilistic
  numerical methods for PDE-constrained Bayesian inverse problems}}, AIP
  Conference Proceedings, 1853 (2017), p.~060001.

\bibitem{conrad2017statistical}
{\sc P.~R. Conrad, M.~Girolami, S.~S{\"a}rkk{\"a}, A.~Stuart, and
  K.~Zygalakis}, {\em Statistical analysis of differential equations:
  introducing probability measures on numerical solutions}, Statistics and
  Computing, 27 (2017), pp.~1065--1082.

\bibitem{C15}
{\sc P.~G. Constantine}, {\em Active Subspaces}, Society for Industrial and
  Applied Mathematics, Philadelphia, PA, 2015.

\bibitem{CDW14}
{\sc P.~G. Constantine, E.~Dow, and Q.~Wang}, {\em Active subspace methods in
  theory and practice: Applications to kriging surfaces}, SIAM Journal on
  Scientific Computing, 36 (2014), pp.~A1500--A1524.

\bibitem{fox2020randomized}
{\sc C.~Fox, T.~Cui, and M.~Neumayer}, {\em Randomized reduced forward models
  for efficient metropolis--hastings mcmc, with application to subsurface fluid
  flow and capacitance tomography}, GEM-International Journal on
  Geomathematics, 11 (2020), pp.~1--38.

\bibitem{giordano2020consistency}
{\sc M.~Giordano and R.~Nickl}, {\em Consistency of bayesian inference with
  gaussian process priors in an elliptic inverse problem}, Inverse Problems, 36
  (2020), p.~085001.

\bibitem{helin2023introduction}
{\sc T.~Helin, A.~M. Stuart, A.~L. Teckentrup, and K.~C. Zygalakis}, {\em
  {Introduction To Gaussian Process Regression In Bayesian Inverse Problems,
  With New Results On Experimental Design For Weighted Error Measures}}, arXiv
  preprint arXiv:2302.04518,  (2023).

\bibitem{HKC04}
{\sc D.~Higdon, M.~Kennedy, J.~C. Cavendish, J.~A. Cafeo, and R.~D. Ryne}, {\em
  Combining field data and computer simulations for calibration and
  prediction}, SIAM Journal on Scientific Computing, 26 (2004), pp.~448--466.

\bibitem{ks05}
{\sc J.~Kaipio and E.~Somersalo}, {\em {Statistical and Computational Inverse
  Problems}}, Springer, Dordrecht, 2005.

\bibitem{KH00}
{\sc M.~C. Kennedy and A.~O'Hagan}, {\em Bayesian calibration of computer
  models}, Journal of the Royal Statistical Society, Series B, Methodological,
  63 (2000), pp.~425--464.

\bibitem{lst18}
{\sc H.~C. Lie, T.~J. Sullivan, and A.~L. Teckentrup}, {\em {Random forward
  models and log-likelihoods in Bayesian inverse problems}}, SIAM/ASA Journal
  on Uncertainty Quantification, 6 (2018), pp.~1600--1629.

\bibitem{MX09}
{\sc Y.~Marzouk and D.~Xiu}, {\em A stochastic collocation approach to bayesian
  inference in inverse problems}, PRISM: NNSA Center for Prediction of
  Reliability, Integrity and Survivability of Microsystems, 6 (2009).

\bibitem{MNR07}
{\sc Y.~M. Marzouk, H.~N. Najm, and L.~A. Rahn}, {\em Stochastic spectral
  methods for efficient bayesian solution of inverse problems}, Journal of
  Computational Physics, 224 (2007), pp.~560--586.

\bibitem{matsumoto2023images}
{\sc T.~Matsumoto and T.~Sullivan}, {\em {Images of Gaussian and other
  stochastic processes under closed, densely-defined, unbounded linear
  operators}}, arXiv preprint arXiv:2305.03594,  (2023).

\bibitem{niederreiter1992random}
{\sc H.~Niederreiter}, {\em Random number generation and quasi-Monte Carlo
  methods}, SIAM, 1992.

\bibitem{H06}
{\sc A.~O'Hagan}, {\em Bayesian analysis of computer code outputs: A tutorial},
  Reliability Engineering \& System Safety, 91 (2006), pp.~1290--1300.

\bibitem{RRPK17}
{\sc M.~Raissi, P.~Perdikaris, and G.~E. Karniadakis}, {\em Machine learning of
  linear differential equations using gaussian processes}, Journal of
  Computational Physics, 348 (2017), pp.~683--693.

\bibitem{RW06}
{\sc C.~E. Rasmussen and C.~K.~I. Williams}, {\em Gaussian Processes for
  Machine Learning}, MIT Press, 2006.

\bibitem{rathgeber2016firedrake}
{\sc F.~Rathgeber, D.~A. Ham, L.~Mitchell, M.~Lange, F.~Luporini, A.~T. McRae,
  G.-T. Bercea, G.~R. Markall, and P.~H. Kelly}, {\em Firedrake: automating the
  finite element method by composing abstractions}, ACM Transactions on
  Mathematical Software (TOMS), 43 (2016), pp.~1--27.

\bibitem{RC04}
{\sc C.~Robert and G.~Casella}, {\em Monte {Carlo} statistical methods},
  Springer Verlag, 2004.

\bibitem{roberts1996exponential}
{\sc G.~O. Roberts and R.~L. Tweedie}, {\em Exponential convergence of langevin
  distributions and their discrete approximations}, Bernoulli,  (1996),
  pp.~341--363.

\bibitem{SWM89}
{\sc J.~Sacks, W.~J. Welch, T.~J. Mitchell, and H.~P. Wynn}, {\em {Design and
  Analysis of Computer Experiments}}, Statistical Science, 4 (1989), pp.~409 --
  423.

\bibitem{spitieris2023bayesian}
{\sc M.~Spitieris and I.~Steinsland}, {\em {Bayesian Calibration of Imperfect
  Computer Models using Physics-Informed Priors}}, Journal of Machine Learning
  Research, 24 (2023), pp.~1--39.

\bibitem{S99}
{\sc M.~L. Stein}, {\em Interpolation of spatial data}, Springer Series in
  Statistics, Springer-Verlag, New York, 1999.

\bibitem{ST16}
{\sc A.~Stuart and A.~Teckentrup}, {\em Posterior consistency for gaussian
  process approximations of bayesian posterior distributions}, Mathematics of
  Computation, 87 (2018).

\bibitem{S10}
{\sc A.~M. Stuart}, {\em Inverse problems: A bayesian perspective}, Acta
  Numerica, 19 (2010), p.~451–559.

\bibitem{sgfsj21}
{\sc L.~P. Swiler, M.~Gulian, A.~L. Frankel, C.~Safta, and J.~D. Jakeman}, {\em
  Constrained gaussian processes: A survey.}, tech. rep., Sandia National
  Lab.(SNL-NM), Albuquerque, NM (United States); Sandia~…, 2021.

\bibitem{t20}
{\sc A.~L. Teckentrup}, {\em Convergence of gaussian process regression with
  estimated hyper-parameters and applications in bayesian inverse problems},
  SIAM/ASA Journal on Uncertainty Quantification, 8 (2020), pp.~1310--1337.

\bibitem{XK03}
{\sc D.~Xiu and G.~E. Karniadakis}, {\em Modeling uncertainty in flow
  simulations via generalized polynomial chaos}, Journal of Computational
  Physics, 187 (2003), pp.~137--167.

\end{thebibliography}

\newpage
\appendix
\section{Derivation of the analytical formula of the marginal approximation}\label{app:infl}

In order to simplify the notation here, we let $\mathbf{m}_{\btheta} = \mathbf{m}_N^{\mathcal{G}_X}({\btheta})$, $K_{\btheta} = K_N(\btheta,\btheta)$ and $\Gamma_{\eta} = \sigma^{-2}_{\eta}I_{\dz}$. First, we assume $\mathcal{G}^N_{X}(\btheta) = \mathbf{m}_{\btheta} + \boldsymbol{\xi}$, where $\boldsymbol{\xi} \sim \mathcal{N}(0, K_{\btheta})$, so by the definition of expectation we have 
\begin{align*}
&\mathbb{E}\left(\exp\left(-\frac{1}{2}\|\mathcal{G}_{X}^N(\btheta)-\mathbf{y}\|_{\Gamma_eta}\right)\pi_0(\btheta)\right) \nonumber \\
&= \frac{1}{\sqrt{(2\pi)^{d_\mathbf{y}}\det{(K_{\btheta})}}}\int_{\mathbb{R}^{d_\mathbf{y}}}
\exp\left(-\frac{\|\mathbf{m}_{\btheta}+\boldsymbol{\xi}-\mathbf{y}\|^2_{\Gamma_{\eta}}}{2}\right)
\exp\left(-\frac{\|\boldsymbol{\xi}\|^2_{K_{\btheta}}}{2}\right)d\boldsymbol{\xi},
\end{align*}  
then rewrite and simplify the formula
\[= \frac{1}{\sqrt{(2\pi)^{d_\mathbf{y}}\det{(K_{\btheta})}}}\int_{\mathbb{R}^{d_\mathbf{y}}}
\exp\left(-\frac{1}{2}\left(\|\boldsymbol{\xi}-(\mathbf{y}-\mathbf{m}_{\btheta})\|^2_{\Gamma_{\eta}}+\|\boldsymbol{\xi}\|^2_{K_{\btheta}}\right)\right)d\bxi\]
we let $\Bar{\mathbf{y}} = \mathbf{y}-\mathbf{m}_{\btheta}$, then
\begin{align*}
&= \frac{1}{\sqrt{(2\pi)^{d_\mathbf{y}}\det{(K_{\btheta})}}}\int_{\mathbb{R}^{d_\mathbf{y}}}
\exp\left(-\frac{1}{2}\left(\|\boldsymbol{\xi}-\Bar{\mathbf{y}}\|^2_{\Gamma_{\eta}}+\|\boldsymbol{\xi}\|^2_{K_{\btheta}}\right)\right)d\bxi    \\
&= \frac{1}{\sqrt{(2\pi)^{d_\mathbf{y}}\det{(K_{\btheta})}}}\int_{\mathbb{R}^{d_\mathbf{y}}}
\exp\left(-\frac{1}{2}\left(
(\boldsymbol{\xi}-\Bar{\mathbf{y}})^T\Gamma_{\eta}^{-1}(\boldsymbol{\xi}-\Bar{\mathbf{y}}) + \bxi^TK_{\btheta}^{-1}\bxi
\right)\right)d\bxi\\
&= \frac{1}{\sqrt{(2\pi)^{d_\mathbf{y}}\det{(K_{\btheta})}}}\int_{\mathbb{R}^{d_\mathbf{y}}}
\exp\left(-\frac{1}{2}\left(
\bxi^T(\Gamma_{\eta}^{-1}+K_{\btheta}^{-1})\bxi - 2\Bar{\mathbf{y}}^T\Gamma_{\eta}^{-1}\bxi + \Bar{\mathbf{y}}^T\Gamma_{\eta}^{-1}\Bar{\mathbf{y}}
\right)\right)d\bxi
\end{align*}
Since $\Gamma_{\eta}$ and $K_{\btheta}$ are symmetric matrices, we have 
\begin{align*}
    \Bar{\mathbf{y}}^T\Gamma_{\eta}^{-1}\bxi 
    &= \Bar{\mathbf{y}}^T ((K_{\btheta}+\Gamma_{\eta})^{-1}K_{\btheta})(K_{\btheta}^{-1}(K_{\btheta}+\Gamma_{\eta}))\Gamma_{\eta}^{-1}\bxi\\
    &= (K_{\btheta}(K_{\btheta}+\Gamma_{\eta})^{-1}\Bar{\mathbf{y}})^TK_{\btheta}^{-1}(K_{\btheta}+\Gamma_{\eta})\Gamma_{\eta}^{-1}\bxi\\
    &= \tilde{\mathbf{y}}^TC^{-1}\bxi,
\end{align*}
where $C = K_{\btheta}(K_{\btheta}+\Gamma_{\eta})^{-1}\Gamma_{\eta}$ and $\tilde{\mathbf{y}} = C\Gamma_{\eta}^{-1}\Bar{\mathbf{y}}$. Substituting it into the formula above, we have 
\begin{align*}
&= \frac{1}{\sqrt{(2\pi)^{d_\mathbf{y}}\det{(K_{\btheta})}}}\int_{\mathbb{R}^{d_\mathbf{y}}}
\exp\left(-\frac{1}{2}\left(
\bxi^TC^{-1}\bxi - 2\tilde{\mathbf{y}}^TC^{-1}\bxi + \Bar{\mathbf{y}}^T\Gamma_{\eta}^{-1}\Bar{\mathbf{y}}
\right)\right)d\bxi
\end{align*}
Then we can complete the square
\begin{align*}
&= \frac{1}{\sqrt{(2\pi)^{d_\mathbf{y}}\det{(K_{\btheta})}}}\int_{\mathbb{R}^{d_\mathbf{y}}}
\exp\left(-\frac{1}{2}\left(\|\bxi - \tilde{\mathbf{y}}\|^2_{C} - \tilde{\mathbf{y}}^TC^{-1}\tilde{\mathbf{y}} + \Bar{\mathbf{y}}^T\Gamma_{\eta}^{-1}\Bar{\mathbf{y}}
\right)\right)d\bxi \\
&= \frac{1}{\sqrt{(2\pi)^{d_\mathbf{y}}\det{(K_{\btheta})}}}\int_{\mathbb{R}^{d_\mathbf{y}}}
\exp\left(-\frac{1}{2}\left(\|\bxi - \tilde{\mathbf{y}}\|^2_{C} - (C\Gamma_{\eta}^{-1}\Bar{\mathbf{y}})^TC^{-1}(C\Gamma_{\eta}^{-1}\Bar{\mathbf{y}}) + \Bar{\mathbf{y}}^T\Gamma_{\eta}^{-1}\Bar{\mathbf{y}}
\right)\right)d\bxi \\
&= \frac{1}{\sqrt{(2\pi)^{d_\mathbf{y}}\det{(K_{\btheta})}}}\int_{\mathbb{R}^{d_\mathbf{y}}}
\exp\left(-\frac{1}{2}\left(\|\bxi - \tilde{\mathbf{y}}\|^2_{C} - \Bar{\mathbf{y}}^T\Gamma_{\eta}^{-1}K_{\btheta}(K_{\btheta}+\Gamma_{\eta})^{-1}\Bar{\mathbf{y}}) + \Bar{\mathbf{y}}^T\Gamma_{\eta}^{-1}\Bar{\mathbf{y}}
\right)\right)d\bxi \\
&= \frac{1}{\sqrt{(2\pi)^{d_\mathbf{y}}\det{(K_{\btheta})}}}\int_{\mathbb{R}^{d_\mathbf{y}}}
\exp\left(-\frac{1}{2}\left(\|\bxi - \tilde{\mathbf{y}}\|^2_{C} - \Bar{\mathbf{y}}^T(\Gamma_{\eta}^{-1}K_{\btheta}(K_{\btheta}+\Gamma_{\eta})^{-1} - \Gamma_{\eta}^{-1})\Bar{\mathbf{y}})
\right)\right)d\bxi \\
&= \frac{1}{\sqrt{(2\pi)^{d_\mathbf{y}}\det{(K_{\btheta})}}}\int_{\mathbb{R}^{d_\mathbf{y}}}
\exp\left(-\frac{1}{2}\left(\|\bxi - \tilde{\mathbf{y}}\|^2_{C} + \Bar{\mathbf{y}}^T(K_{\btheta}+\Gamma_{\eta})^{-1}\Bar{\mathbf{y}}
\right)\right)d\bxi \\
&= \frac{1}{\sqrt{(2\pi)^{d_\mathbf{y}}\det{(K_{\btheta})}}}\exp{\left(-\frac{1}{2}\|\Bar{\mathbf{y}}\|^2_{(K_{\btheta}+\Gamma_{\eta})}\right)}\int_{\mathbb{R}^{d_\mathbf{y}}}
\exp\left(-\frac{1}{2}\left(\|\bxi - \tilde{\mathbf{y}}\|^2_{C}
\right)\right)d\bxi \\
&= \frac{\sqrt{\det{(C)}}}{\sqrt{\det{(K_{\btheta})}}}\exp{\left(-\frac{1}{2}\|\Bar{\mathbf{y}}\|^2_{(K_{\btheta}+\Gamma_{\eta})}\right)}\int_{\mathbb{R}^{d_\mathbf{y}}}\frac{1}{\sqrt{(2\pi)^{d_\mathbf{y}}\det{(C)}}}
\exp\left(-\frac{1}{2}\left(\|\bxi - \tilde{\mathbf{y}}\|^2_{C}
\right)\right)d\bxi \\
&\propto  \frac{1}{\sqrt{(2\pi)^{d_\mathbf{y}}\det(K_{\btheta}+\Gamma_{\eta})}}\exp{\left(-\frac{1}{2}\|\mathbf{y}-\mathbf{m}_{\btheta}\|^2_{(K_{\btheta}+\Gamma_{\eta})}\right)} \\
\end{align*}
Hence, we obtain the explicit form of the marginal approximation.

\section{Derivation of the gradient of the approximate log-posteriors}
\begin{proof}
\begin{align*}
\nabla \log \pi^{N,\mathcal{G}_{X}}_{\mathrm{mean}}(\btheta|\mathbf{y}) &= \nabla \log \left( \exp\left(-\frac{1}{2\sigma^{2}_{\eta}}\|\mathbf{m}^{\mathcal{G}_X}_{N}(\btheta)-\mathbf{y}\|^2\right)\right)\\
 &= -\frac{1}{2\sigma^{2}_{\eta}}\nabla \left(\|\mathbf{m}^{\mathcal{G}_X}_{N}(\btheta)-\mathbf{y}\|^2\right)\\
 &= -\frac{1}{\sigma^{2}_{\eta}} \left(\nabla\mathbf{m}^{\mathcal{G}_X}_{N}(\btheta)\right)^{T}\left(\mathbf{m}^{\mathcal{G}_X}_{N}(\btheta)-\mathbf{y}\right)\\
 &= -\frac{1}{\sigma^{2}_{\eta}} \left(\nabla K(\btheta,\Theta)K(\Theta,\Theta)^{-1}\mathbf{y}\right)^{T}\left(\mathbf{m}^{\mathcal{G}_X}_{N}(\btheta)-\mathbf{y}\right)\\
\end{align*}

\begin{align*}
&\nabla \log \pi^{N,\mathcal{G}_{X}}_{\mathrm{marginal}}(\btheta|\mathbf{y}) \\
&= \nabla \log \left(\frac{\exp\left(-\frac{1}{2}\|\mathbf{m}^{\mathcal{G}_X}_{N}(\btheta)-\mathbf{y}\|^2_{(K_{N}(\btheta,\btheta)+\Gamma_{\eta})}\right)}{\sqrt{(2\pi)^{d_\mathbf{y}}\det\left(K_{N}(\btheta,\btheta)+\Gamma_{\eta}\right)}}\right)\\
&= -\frac{1}{2}\nabla \left(\|\mathbf{m}^{\mathcal{G}_X}_{N}(\btheta)-\mathbf{y}\|^2_{(K_{N}(\btheta,\btheta)+\Gamma_{\eta})}\right) - \frac{1}{2} \nabla \log \left( (2\pi)^n\det\left(K_{N}(\btheta,\btheta)+\Gamma_{\eta}\right)\right)\\
&= -(\nabla K(\btheta,\Theta)K(\Theta,\Theta)^{-1}\mathbf{y})^T(K_N(\btheta,\btheta)+\Gamma_{\eta})^{-1}(\mathbf{m}_N^{\mathcal{G}_{X}}(\btheta)-\mathbf{y}) \\
&\qquad - \frac{1}{2}(\mathbf{m}_N^{\mathcal{G}_{X}}(\btheta)-\mathbf{y})^T\nabla\left((K_N(\btheta,\btheta)+\Gamma_{\eta})^{-1}\right)(\mathbf{m}_N^{\mathcal{G}_{X}}(\btheta)-\mathbf{y})\\
&\qquad - \frac{1}{2}\left( {\Tr \left((K_N(\btheta,\btheta)+\Gamma_{\eta} )^{-1}\right)\nabla(K_N(\btheta,\btheta))}\right),
\end{align*}
where \[\nabla\left((K_N(\btheta,\btheta)+\Gamma_{\eta})^{-1}\right) =  -(K_N(\btheta,\btheta)+\Gamma_{\eta})^{-1}\nabla\left(K_N(\btheta,\btheta)\right) (K_N(\btheta,\btheta)+\Gamma_{\eta})^{-1},
\]
and \[\nabla K_N(\btheta,\btheta) = 2\nabla K(\btheta,\Theta)K(\Theta,\Theta)^{-1}K(\Theta,\btheta)\]
\end{proof}

\end{document}